\documentclass[12pt]{article}
\usepackage{amsmath}
\usepackage{times}
\usepackage{graphicx}
\usepackage{color}
\usepackage{multirow}
\usepackage[authoryear]{natbib}
\usepackage{rotating}
\usepackage{bbm}
\usepackage{latexsym}

\usepackage[titletoc,title]{appendix}
\usepackage{Definitions}
\usepackage{amsmath}
\usepackage{amssymb}
\usepackage{hyperref}
\usepackage{caption}
\usepackage{subcaption}
\usepackage{graphicx}

\hypersetup{
    colorlinks=true,
    linkcolor=black,
    citecolor=black,
    filecolor=black,
    urlcolor=black
}

\newcommand{\captionfonts}{\normalsize}

\makeatletter  
\long\def\@makecaption#1#2{%
  \vskip\abovecaptionskip
  \sbox\@tempboxa{{\captionfonts #1: #2}}%
  \ifdim \wd\@tempboxa >\hsize
    {\captionfonts #1: #2\par}
  \else
    \hbox to\hsize{\hfil\box\@tempboxa\hfil}%
  \fi
  \vskip\belowcaptionskip}
\makeatother   

\renewcommand{\thefootnote}{\normalsize \arabic{footnote}}      

\usepackage{algorithm}
\usepackage{algorithmic}
\usepackage{booktabs}
\usepackage{mathtools}
\usepackage[titletoc,title]{appendix}
\usepackage[normalem]{ulem}
\usepackage{bm}
\usepackage{comment}
\usepackage{slashbox}
\usepackage[normalem]{ulem}
\usepackage{color}
\usepackage{todonotes}
\usepackage{Definitions}
\usepackage{amsmath}
\usepackage{amssymb}

\newcommand{\tsigma}{\tilde \sigma}
\newcommand{\hLcal}{\hat \Lcal}
\newcommand{\hL}{\hat L}
\newcommand{\hf}{\hat f}
\newcommand{\htheta}{\hat \theta}
\newcommand{\bg}{\bar g}
\newcommand{\hg}{\hat g}
\newcommand{\baf}{\bar f}
\newcommand{\bW}{\bar W}
\newcommand{\uc}{\textit{\underbar c}}
\newcommand{\bc}{\bar c}
\newcommand{\bepsilon}{\bar \epsilon}
\newcommand{\tGcal}{\tilde \Gcal}
\newcommand{\tightoverset}[2]{%
  \mathop{#2}\limits^{\vbox to -.5ex{\kern-1.45ex\hbox{{\scriptsize$#1$}}\vss}}}

\newenvironment{myquote}%
  {\list{}{\leftmargin=0.25in\rightmargin=0in}\item[]}%
  {\endlist}

\newcommand{\rev}[1]{#1}

\usepackage{lipsum}

\newcommand\blfootnote[1]{%
  \begingroup
  \renewcommand\thefootnote{}\footnote{#1}%
  \addtocounter{footnote}{-1}%
  \endgroup
}

\begin{document}

\ \\

{\bf \large \noindent Understanding Dynamics of Nonlinear Representation  Learning and Its Application}

\vspace{10pt}

\ \\
{\bf  Kenji Kawaguchi} \\ 
kkawaguchi@fas.harvard.edu \\
Harvard University\\
\ \\
{\bf  Linjun Zhang} \\ 
linjun.zhang@rutgers.edu  \\
Rutgers University \\
\ \\
{\bf   Zhun Deng} \\ 
zhundeng@g.harvard.edu \\
Harvard University\\
%

\thispagestyle{empty}
\markboth{}{NC instructions}
\blfootnote{\hspace{-18pt}\noindent Neural computation, volume 34, pages 991-1018 (2022) \\  \url{https://doi.org/10.1162/neco_a_01483}}

\ \vspace{-20mm}\\
%
\begin{center} {\bf Abstract} \end{center}
Representations of the world environment play a crucial role in artificial intelligence. It is often inefficient to conduct reasoning and inference directly in the space of raw sensory representations, such as pixel values of images. Representation learning allows us to automatically discover suitable representations from raw sensory data. For example, given raw sensory data, a deep neural network learns nonlinear representations at its hidden layers, which are subsequently used for classification (or regression) at its output layer. This happens implicitly during training through minimizing a supervised or unsupervised loss in  common practical regimes of deep learning, unlike the neural tangent kernel (NTK) regime. In this paper, we study the dynamics of such implicit nonlinear representation learning, which is beyond the NTK\ regime. We identify a pair of a new assumption and a novel condition, called the \textit{common model structure assumption} and the \textit{data-architecture alignment condition}. Under the common model structure assumption, the data-architecture alignment condition is shown to be sufficient for the global convergence and necessary for the global optimality. \rev{Moreover, our theory explains  how and when increasing the network size does and does not improve the training behaviors in the practical regime}. Our results provide practical guidance for designing a model structure: e.g., the common model structure assumption can be used as a justification for using a particular model structure instead of others. As an application, we then derive a new training framework, which satisfies the data-architecture alignment condition without assuming it by automatically modifying any given training algorithm dependently on each data and architecture. Given a standard training algorithm, the framework running its modified version is empirically shown to maintain competitive (practical) test performances while providing global convergence guarantees for deep residual neural networks with convolutions, skip connections, and batch normalization with standard benchmark datasets, including MNIST, CIFAR-10, CIFAR-100, Semeion, KMNIST and SVHN.

\section{Introduction}
\rev{LeCun, Bengio, and Hinton  \citep{lecun2015deep} described deep learning  as  one of hierarchical nonlinear representation learning approaches:\vspace{-8pt}
\begin{myquote}
Deep-learning methods are representation-learning methods with multiple levels of representation, obtained by composing simple but non-linear modules that each transform the representation at one level (starting with the raw input) into a representation at a higher, slightly more abstract level. (p. 436)
\end{myquote}}

\rev{In applications such as computer vision and natural language processing, the success of deep learning can be attributed to  its ability to learn hierarchical nonlinear representations  by automatically changing nonlinear features and  kernels during training based on the given data. This is in contrast to the neural tangent kernel (NTK) regime, extremely wide neural networks, and  classical machine-learning methods, where representations or equivalently  nonlinear features and  kernels  are (approximately) fixed during training.}

Deep learning in practical regimes, which has  the ability to learn nonlinear representation  \citep{bengio2013representation}, has had a profound impact in many areas, including  object recognition in computer vision \citep{rifai2011manifold,hinton2006fast,bengio2007greedy,ciregan2012multi,krizhevsky2012imagenet}, style transfer \citep{gatys2016image,luan2017deep}, image super-resolution \citep{dong2014learning}, speech recognition \citep{dahl2010phone,deng2010binary,seide2011conversational,mohamed2011acoustic,dahl2011context,hinton2012deep}, machine translation \citep{schwenk2012large,le2012structured}, paraphrase detection \citep{socher2011dynamic}, word sense disambiguation \citep{bordes2012joint},  and sentiment analysis \citep{glorot2011domain,socher2011semi}. However, we do not yet know what is the precise condition that  makes  deep learning   tractable  in the practical  regime of representation learning. 

To initiate  a study towards such a condition, we consider the following problem setup that covers deep learning in the practical regime and other  nonlinear  representation learning methods in general. We are given a training dataset $((x_i,y_i))_{i=1}^n$ of  $n$ samples where $x_i \in \Xcal \subseteq  \RR^{m_x}$ and $y_i \in\Ycal\subseteq  \RR^{m_y}$ are the $i$-th input and the $i$-th target respectively. We would like to learn a predictor (or hypothesis) from a parametric family $\Hcal=\{f(\cdot,\theta): \RR^{m_x}  \rightarrow \RR^{ m_y} \mid \theta \in \RR^d\}$  by minimizing the  objective $\Lcal$: 
\begin{align}
\Lcal(\theta)= \frac{1}{n} \sum_{i=1}^n \ell(f(x_{i},\theta),y_{i}), 
\end{align} 
where $\ell: \RR^{m_{y}} \times\Ycal \rightarrow \RR_{\ge 0}$ is the loss function  that measures the discrepancy between the prediction $f(x_{i},\theta)$ and the target $y_{i}$ for each sample. In this paper, it is allowed to let $y_{i}=x_i$ (for $i=1,\dots,n$) to include the setting of unsupervised learning. The function $f$ is also allowed to represent a wide range of machine learning models. \rev{For a (deep) neural network, $f(x_{i},\theta)$ represents the \textit{pre-activation} output of the last layer of the network, and the parameter vector $\theta \in \RR^d$ contains all the trainable parameters, including  weights and bias terms of all layers.} 

For example, one of the simplest models is the linear model in the form of  $f(x,\theta)=\phi(x)\T \theta$ where $\phi:\Xcal \rightarrow\RR^d$ is a fixed function and $\phi(x) $ is a nonlinear \textit{representation} of input data $x$. This is a classical machine learning model where much of the effort goes into the design of the hand-crafted feature map $\phi$ via feature engineering \citep{turner1999conceptual,zheng2018feature}. In this linear model, we do not learn the representation $\phi(x) $ because the feature map $\phi$ is fixed without dependence on  the model parameter  $\theta$ that is optimized with  the dataset  $((x_i,y_i))_{i=1}^n$.

 Similarly to many definitions in  mathematics, where an intuitive notion in a special case  is formalized to a definition for a more general case, we now abstract and generalize this intuitive notion  of the \textit{representation} $\phi(x)$ of the linear model to that of all differentiable models as follows:       
\begin{definition} \label{def:1}
Given any  $x\in\Xcal $ and  differentiable function $f $, we define $\frac{\partial f(x,\theta)}{\partial\theta}$ to be  the \textit{gradient representation of the data $x$ under the model $f$ at $\theta$}. 
\end{definition}
This definition  recovers  
the  standard representation $\phi(x) $ in the linear model as $\frac{\partial f(x,\theta)}{\partial\theta}=\frac{\partial\phi(x)\T \theta}{\partial\theta}=\phi(x)$ and is applicable to all differentiable nonlinear models in representation learning. Moreover, this definition captures the key challenge of understanding the dynamics of nonlinear representation learning well, as illustrated below. Using the notation of $\frac{d\theta^{t}}{d t} =\Delta^{t}$, the dynamics of the model $f(x,\theta^{t} )$ over the time $t$ can be written by  
\begin{align}
\quad \frac{d}{d t} f(x,\theta^{t} )=\frac{\partial f(x,\theta^{t})}{\partial\theta^{t}} \frac{d\theta^{t}}{d t}=\frac{\partial f(x,\theta^{t})}{\partial\theta^{t}} \Delta^{t}.
\end{align}
Here, we  can see that the dynamics are \textit{linear} in $\Delta^{t}$ if  there is no  gradient representation learning  as $\frac{\partial f(x,\theta^{t})}{\partial\theta^{t}}\approx\frac{\partial f(x,\theta^{0})}{\partial\theta^{0}}$, which  is indeed the case for  the neural tangent kernel (NTK) regime and  extremely wide neural networks. However,  with   representation learning in the practical  regime, the gradient representation  $\frac{\partial f(x,\theta^{t})}{\partial\theta^{t}}$   changes depending on $t$ (and $\Delta^t$), resulting in the dynamics that are \textit{nonlinear} in $\Delta^{t}$. Therefore,  the definition of the gradient representation
can distinguish fundamentally different dynamics in machine learning.

In this paper, we initiate the  study of the dynamics of learning   gradient representation that are nonlinear in $\Delta^{t}$.  That is, we focus on the regime where  the gradient representation $\frac{\partial f(x,\theta^{t})}{\partial\theta^{t}}$ at the end of training time $t$ differs  greatly from the initial representation $\frac{\partial f(x,\theta^{0})}{\partial\theta^{0}}$, unlike  the NTK regime where  we have  $\frac{\partial f(x,\theta^{t})}{\partial\theta^{t}}\approx\frac{\partial f(x,\theta^{0})}{\partial\theta^{0}}$. This practical regime of representation learning was studied in the past   for the case where the function $\phi(x) \mapsto f(x,\theta)$ is affine for some fixed feature map $\phi$ \citep{saxe2013exact, kawaguchi2016deep,  laurent2018deep,bartlett2019gradient,zou2020global,kawaguchi2021on,Xu2021gnnopt}. Unlike any previous studies, we focus on the problem setting where the function $\phi(x) \mapsto f(x,\theta)$ is  nonlinear and non-affine, with the effect of  nonlinear (gradient) representation learning having  $\frac{\partial f(x,\theta^{t})}{\partial\theta^{t}} \not\approx \frac{\partial f(x,\theta^{0})}{\partial\theta^{0}}$. The results of this paper avoid the curse of dimensionality by studying the global convergence of the gradient-based dynamics instead of the dynamics of global optimization \citep{kawaguchi2016global} and Bayesian optimization \citep{kawaguchiNIPS2015}. Importantly, we do not require any wide layer or large input dimension throughout this paper. Our main contributions are summarized as follows: 
\begin{enumerate}
\item 
In Section \ref{sec:13}, we  identify a pair of a novel assumption and a new condition, called the \textit{common model structure assumption} and the \textit{data-architecture alignment condition}. Under  the common model structure assumption, the data-architecture alignment condition  is shown to be a necessary condition for the globally optimal model and a sufficient condition for  the global convergence. The condition is dependent on both data and architecture. Moreover, we empirically verify and deepen this new understanding. When we apply representation learning in practice, we often have overwhelming options regarding which model structure to be used. Our results provide a practical guidance for choosing or designing model structure via the common model structure assumption, which is indeed satisfied by many representation learning models used in practice.

\item  In Section \ref{sec:14},
we discard the assumption of the data-architecture alignment condition. Instead, we  derive a novel training framework, called the Exploration-Exploitation Wrapper (EE Wrapper), which  satisfies the data-architecture alignment condition time-independently a priori. The EE Wrapper is then proved to have global convergence guarantees under the \textit{safe-exploration condition}. The safe-exploration condition is what allows us to explore various gradient representations safely  without getting stuck in the states where we cannot provably satisfy the data-architecture alignment condition. The safe-exploration  condition is shown to hold true for ResNet-18 with standard benchmark datasets, including MNIST, CIFAR-10, CIFAR-100, Semeion, KMNIST and SVHN time-independently.      

\item 
 In Subsection \ref{sec:9}, the Exploration-Exploitation Wrapper is shown to not  degrade practical performances of ResNet-18 for  the standard datasets, MNIST, CIFAR-10, CIFAR-100, Semeion, KMNIST and SVHN. To our knowledge, the present paper  provides the first practical algorithm with  global convergence guarantees  without degrading practical performances of ResNet-18 on these standard  datasets, using convolutions, skip connections, and batch normalization without any extremely wide  layer of the width larger  than the number of data points. To the best of our knowledge, we are not aware of any similar  algorithms with global convergence guarantees in the regime of learning nonlinear representations without degrading practical performances.         

\end{enumerate}

\rev{It is   empirically known  that increasing the network size tends to improve the training behaviors. Indeed, the size of networks correlates well with the training error  in many cases
 in our experiments (e.g., Figure \ref{fig:1} b). However, the size and the training error do not correlate well   in some experiments (e.g., Figure \ref{fig:1} c). Our new theoretical results explain that the training behaviors  correlate more directly with the data-architecture alignment condition instead. The seeming correlation with the network size is caused by another correlation between the network size and the data-architecture alignment condition. This is explained  in more details in Section \ref{sec:new:3}.}

\section{Understanding Dynamics via Common Model Structure and Data-Architecture Alignment  } \label{sec:13}

In this section, we identify the common model structure assumption  and study the data-architecture alignment condition for  the global convergence in nonlinear representation learning. We begin by presenting an overview of our results   in Subsection \ref{sec:1}, deepen our   understandings via experiments in   Subsection \ref{sec:3}, discuss implications of our results in Subsection \ref{sec:new:3}, and establish mathematical theories  in Subsection \ref{sec:2}.

\subsection{Overview} \label{sec:1}

We introduce the common model structure assumption in Subsection \ref{sec:new:new:1} and define the data-architecture alignment condition in Subsection \ref{sec:new:new:2}. Using the assumption and the condition, we present the global convergence result in Subsection \ref{sec:new:new:3}. 
 
\subsubsection{Common Model Structure Assumption} \label{sec:new:new:1}

Through examinations of representation learning models used in applications, we identified and formalized one of their common properties as follows:   
   
\begin{assumption} \label{assump:7}
\emph{(Common Model Structure Assumption)} There exists a subset $S\subseteq \{1,2,\dots,d\}$ such that $f(x_i,\theta)=\sum_{k=1}^d \one\{k \in S\}\theta_k(\frac{\partial f(x_{i},\theta)}{\partial\theta_{k}})$ for any $i\in \{1,\dots,n\}$ and $\theta \in \RR^d$. 
\end{assumption}        

Assumption \ref{assump:7} is satisfied by common machine learning models, such as kernel models and multilayer neural networks, with or without convolutions, batch normalization, pooling, and skip connections. For example, consider a multilayer neural network of the form $f(x,\theta)=W h(x,u)+b$, where $h(x,u)$ is an output of its last hidden layer and the parameter vector $\theta$ consists of the parameters $(W,b)$ at the last layer and the parameters $u$ in all other layers as $\theta=\vect([W,b,u])$. \rev{Here, for any   matrix $M\in \RR^{m \times \bar m}$, we let $\vect(M)\in \RR^{ m \bar m}$ be  the standard vectorization of the matrix $M$ by stacking columns.} Then, Assumption \ref{assump:7} holds because $f(x,\theta)=\sum_{k=1}^d \one\{k \in S\}\theta_k(\frac{\partial f(x_{i},\theta)}{\partial\theta_{k}})$ where $S$ is defined by $\{\theta_k:k\in S\}=\{\vect([W,b])_{k} : k =1,2,\dots,\xi\}$  with $\vect([W,b])\in \RR^{\xi}$. Since $h$ is arbitrary in this example, the common model structure assumption holds, for example, for any multilayer neural networks with a fully-connected last layer. In general, because the nonlinearity at the output layer can be treated as a part of the loss function $\ell$ while preserving convexity of $q \mapsto \ell(q,y)$ (e.g., cross-entropy loss with softmax), this assumption is satisfied by many machine learning models, including ResNet-18 and all models used in the experiments in this paper (as well as all linear models). Moreover, Assumption \ref{assump:7} is automatically satisfied in the next section by using the EE Wrapper.

\subsubsection{Data-Architecture Alignment Condition} \label{sec:new:new:2}

Given a target matrix $Y=(y_{1},y_{2},\dots,y_{n})\T \in \RR^{n \times m_y}$ and a loss function $\ell$, we define the modified target matrix 
  $ Y_{\ell}= ( y_{1}^{\ell}, y_{2}^{\ell},\dots,  y_{n}^{\ell})\T\in \RR^{n \times m_y}$    by $Y_{\ell}=Y$  for the squared loss $\ell$ , and  by $(Y_{\ell})_{ij}=2Y_{ij}-1$  for the  (binary and multi-class) cross-entropy losses $\ell$ with $Y_{ij} \in \{0,1\}$. Given input matrix $X=(x_1,x_2,\dots,x_n)\T \in \RR^{n \times m_x}$, the  output matrix $f_{X}(\theta)\in \RR^{n \times m_y}$ is defined  by $f_{X}(\theta)_{ij}=f(x_{i},\theta)_{j} \in \RR$. For any   matrix $M\in \RR^{m \times \bar m}$, we let  $\Col(M) \subseteq \RR^{m}$ be its column space. With these notations, we are now ready  to introduce the data-architecture alignment condition: 
\begin{definition} \label{def:2}
(Data-Architecture Alignment Condition) Given any dataset $(X,Y)$, differentiable function $f$, and loss function $\ell$, the \textit{data-architecture alignment condition} is said to be satisfied at $\theta$ if $\vect(Y_{\ell}) \in \Col(\frac{\partial \vect(f_{X}(\theta))}{\partial \theta})$.
\end{definition}
The data-architecture alignment condition depends on both data (through  the target  $Y$ and the input  $X$) and architecture (through the model $f$). 
It is  satisfied only when the data and architecture align well to each other. For example, in the case of linear model $f(x,\theta)=\phi(x)\T \theta \in \RR$, the condition can be written as $\vect(Y_{\ell}) \in \Col(\Phi(X))$ where $\Phi(X)\in \RR^{n \times d}$ and  $\Phi(X)_{ij}=\phi(x_i)_j$. \rev{In Definition \ref{def:2}, $f_{X}(\theta)$ is  a matrix of the \textit{pre-activation} outputs of the last layer. Thus, in the case of classification tasks with a nonlinear activation at the output layer, $f_{X}(\theta)$ and $Y$ are not in the same space, which is the reason why we use $Y_\ell$ here instead of $Y$.}          

Importantly, the data-architecture alignment condition does \textit{not} make any requirements on the the rank of the Jacobian matrix $\frac{\partial \vect(f_{X}(\theta))}{\partial \theta}\in \RR^{n m_y \times d}$: i.e., the rank of $\frac{\partial \vect(f_{X}(\theta))}{\partial \theta}$ is allowed to  be smaller than $n m_y$ and $d$. Thus, for example, the data-architecture alignment condition can be satisfied depending on the given data and architecture even if the minimum eigenvalue of the matrix  $\frac{\partial \vect(f_{X}(\theta))}{\partial \theta} (\frac{\partial \vect(f_{X}(\theta))}{\partial\theta})\T$ is zero, in both cases of over-parameterization (e.g., $d\gg n$) and under-parameterization (e.g., $d \ll n$). This is further illustrated in Subsection \ref{sec:3} and discussed in Subsection \ref{sec:new:3}. We  note that we further discard the assumption of the data-architecture alignment condition in Section \ref{sec:14} as it is automatically satisfied by using the EE Wrapper.

\subsubsection{Global Convergence } \label{sec:new:new:3}

Under the common model structure assumption, the data-architecture alignment condition is shown to be what lets us avoid the failure of  the global convergence and sub-optimal local minima. More concretely, we prove a global convergence guarantee under the data-architecture alignment condition as well as the necessity of the condition for the global optimality:      
\begin{theorem} \label{thm:new:1}
\emph{(Informal Version)} Let Assumption \ref{assump:7} hold. Then, the following two statements hold for  gradient-based dynamics:
 \begin{enumerate}[label=(\roman*)]
\item \vspace{-3pt}
 The global optimality gap bound decreases per iteration towards zero at the rate of $O(1/\sqrt{|\Tcal|})$ for any $\Tcal$ such that the data-architecture alignment condition is satisfied at $\theta^{t}$ for $t \in \Tcal$.    
\item \vspace{-3pt}
 For any $\theta \in \RR^d$, the data-architecture alignment condition  at $\theta$ is necessary to have the globally optimal model  $f_X(\theta) = \eta Y_{\ell}$ at $\theta$ for any $\eta \in \RR$.
\end{enumerate}
\end{theorem} 
 Theorem \ref{thm:new:1} (i) guarantees the global convergence \textit{without} the need to satisfy the data-architecture alignment condition at every iteration or at the limit point. Instead, it shows that the bound on the global optimality gap decreases towards zero per iteration whenever the data-architecture alignment condition holds. \rev{Theorem \ref{thm:new:1} (ii)  shows that the data-architecture alignment condition is  necessary for the global optimality. Intuitively, this is  because the expressivity of a model class satisfying the common model structure assumption is restricted such that it is required to align the architecture to the data in order to contain the globally optimal model  $f_X(\theta) = \eta Y_{\ell}$ (for any $\eta \in \RR$).  } 
 
 To better understand the statement of Theorem \ref{thm:new:1} (i), consider a counter example with a dataset consisting of the single point $(x,y)=(1,0)$, the model $f(x,\theta) = \theta^4 - 10 \theta^2 + 6 \theta + 100$, and the squared loss $\ell(q,y)=(q-y)^2$. In this example, we have $\Lcal(\theta)=f(x,\theta)^{2}$, which has multiple suboptimal local minima of different values. Then, via gradient descent,  the model converges to the closest  local minimum and, in particular, does not necessarily converge to a global minimum. Indeed, this example violates the common model structure assumption (Assumption \ref{assump:7}) (although it satisfies the data-architecture alignment condition), showing the importance of the common model structure assumption along with the data-architecture alignment. This also illustrates  the non-triviality of Theorem \ref{thm:new:1} (i) in that the data-architecture alignment is not sufficient, and we needed to understand  what types of model structures are commonly used in practice and formalize the understanding as the common model structure assumption.

To further understand the importance of the common model structure assumption in Theorem \ref{thm:new:1}, let us now consider the case where we do not require the assumption. Indeed, we can guarantee the global convergence without the common model structure assumption if we  ensure that the minimum eigenvalue of the matrix $\frac{\partial \vect(f_{X}(\theta))}{\partial \theta} (\frac{\partial \vect(f_{X}(\theta))}{\partial\theta})\T$ is nonzero. This can be proved by the  following derivation. Let $\theta$ be an arbitrary stationary point of $\Lcal$. Then, we have $0 = \frac{\partial \Lcal(\theta)}{\partial \theta} = \frac{1}{n} \sum_{i=1}^n (\frac{\partial \ell(q,y_i)}{\partial q}\vert_{q=f(x_i,\theta)}) \frac{\partial f(x_i,\theta)}{\partial \theta}$, which implies that
\begin{align}
\frac{\partial \vect(f_{X}(\theta))}{\partial \theta}v=0,
\end{align}  
where $v=\vect((\frac{\partial \ell(q,u_{1})}{\partial q} \vert_{q=f(x_1,\theta)} )\T,\dots,(\frac{\partial \ell(q,u_{N})}{\partial q} \vert_{q=f(x_n,\theta)} )\T )\in \RR^{nm_{y}}$. Therefore, if the minimum eigenvalue of the matrix $\frac{\partial \vect(f_{X}(\theta))}{\partial \theta} (\frac{\partial \vect(f_{X}(\theta))}{\partial\theta})\T$ is nonzero, then we have $v=0$: i.e., $\frac{\partial \ell(q,u_{i})}{\partial q} \vert_{q=f(x_i,\theta)} \allowbreak =0$ for all $i\in\{1,2,\dots,n\}$. Using the convexity of  the map $q\mapsto \ell(q,y)$ (which is satisfied by the squared loss and cross-entropy loss), this implies that for any $ q_1,q_2,\dots,q_n \in \RR^{m_{y}}$,
\begin{align}
&\Lcal(\theta)= \frac{1}{n} \sum_{i=1}^n \ell(f(x_{i},\theta),y_{i})
\\ &\le\frac{1}{n} \sum_{i=1}^n \left(\ell(q_{i},y_{i})- \left(\frac{\partial \ell(q,u_{i})}{\partial q} \Big\vert_{q=f(x_i,\theta)}\right) (q_i-f(x_{i},\theta)) \right)
\\ & \le\frac{1}{n} \sum_{i=1}^n \ell(q_{i},y_{i}).
\end{align}  
Since  $f(x_1),f(x_2),\dots,f(x_n) \in \RR^{m_{y}}$,
this implies that any stationary point $\theta$ is a global minimum if the minimum eigenvalue of the matrix $\frac{\partial \vect(f_{X}(\theta))}{\partial \theta} (\frac{\partial \vect(f_{X}(\theta))}{\partial\theta})\T$ is nonzero, without the common model structure assumption (Assumption \ref{assump:7}). Indeed, in the above  example with the model $f(x,\theta) = \theta^4 - 10 \theta^2 + 6 \theta + 100$, the common model structure assumption is violated but we still have the global convergence if the minimum eigenvalue is nonzero: e.g., $f(x,\theta)=y=0$ at any stationary point $\theta$ such that the minimum eigenvalue of the matrix $\frac{\partial \vect(f_{X}(\theta))}{\partial \theta} (\frac{\partial \vect(f_{X}(\theta))}{\partial\theta})\T$ is nonzero. In contrast, Theorem \ref{thm:new:1} allows the global convergence even when  the minimum eigenvalue of the matrix $\frac{\partial \vect(f_{X}(\theta))}{\partial \theta} (\frac{\partial \vect(f_{X}(\theta))}{\partial\theta})\T$ is zero, by utilizing the common model structure assumption.

 The formal version of Theorem \ref{thm:new:1} is presented in Subsection \ref{sec:2} and is proved in  Appendix \ref{sec:app:1} in the Supplementary Information. Before proving the  statement, we first examine the meaning and implications of our results through illustrative examples in the two next subsections. ~

 \subsection{Illustrative Examples in Experiments } \label{sec:3}

Theorem \ref{thm:new:1}  suggests that    data-architecture alignment condition   $\vect(Y_{\ell}) \in \Col(\frac{\partial \vect(f_{X}(\theta^{t}))}{\partial \theta^{t}})$  has the ability to distinguish the success and failure cases, even when the minimum eigenvalue of the matrix  $\frac{\partial \vect(f_{X}(\theta^{t}))}{\partial \theta^{t}} \allowbreak (\frac{\partial \vect(f_{X}(\theta^{t}))}{\partial\theta^{t}})\T$ is zero for all $t\ge 0$. In this subsection, we conduct experiments to further verify and deepen this theoretical understanding.  

\begin{figure*}[b!] 
\center
\begin{subfigure}[b]{0.45\textwidth}
  \includegraphics[width=\textwidth, height=0.9\textwidth]{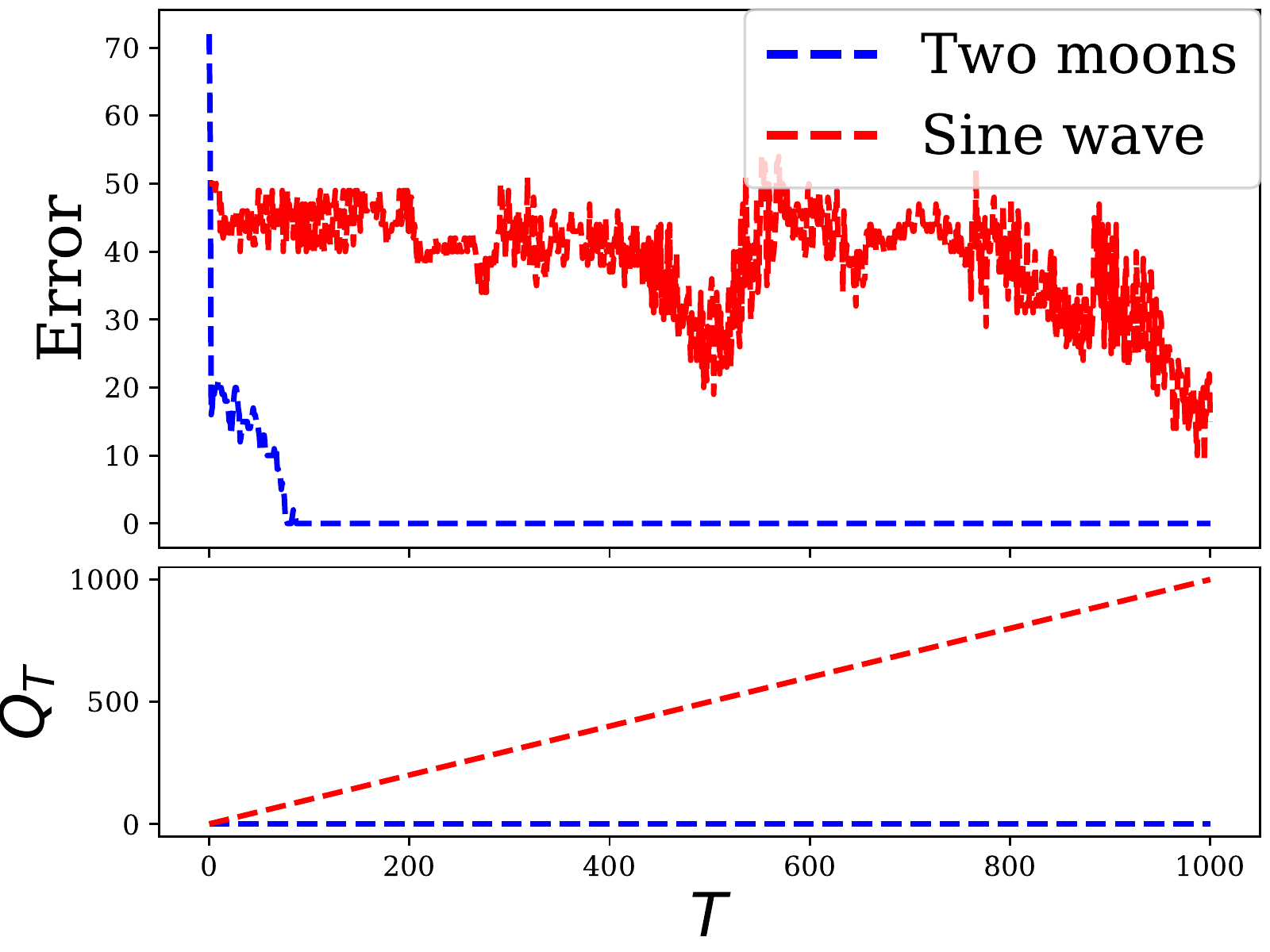}
   \caption{Fully-connected network}
\end{subfigure} \\
\begin{subfigure}[b]{0.45\textwidth}
  \includegraphics[width=\textwidth, height=0.9\textwidth]{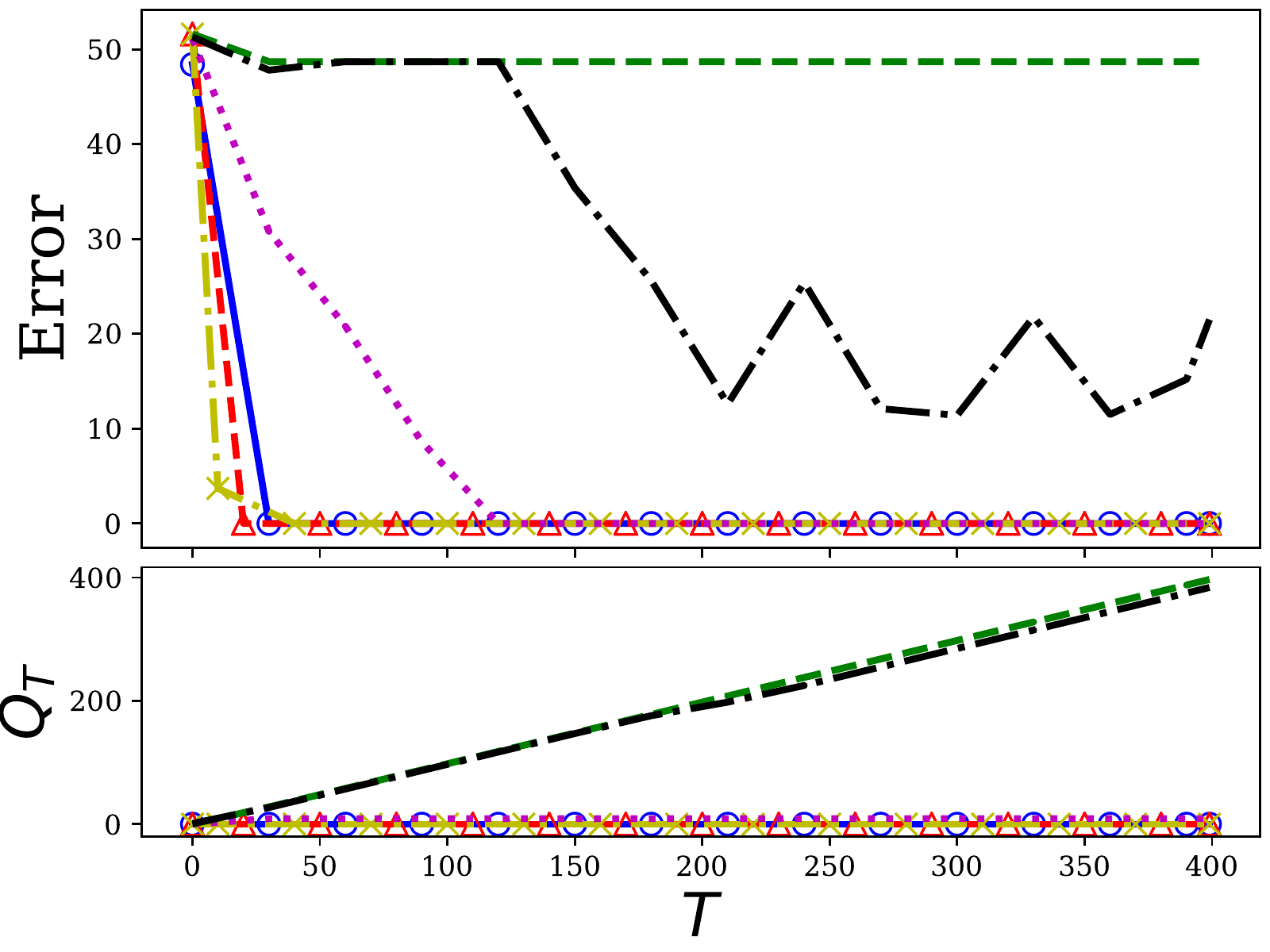}
  \caption{Convolutional network: seed \#1  } 
\end{subfigure}
\begin{subfigure}[b]{0.45\textwidth}
  \includegraphics[width=\textwidth, height=0.9\textwidth]{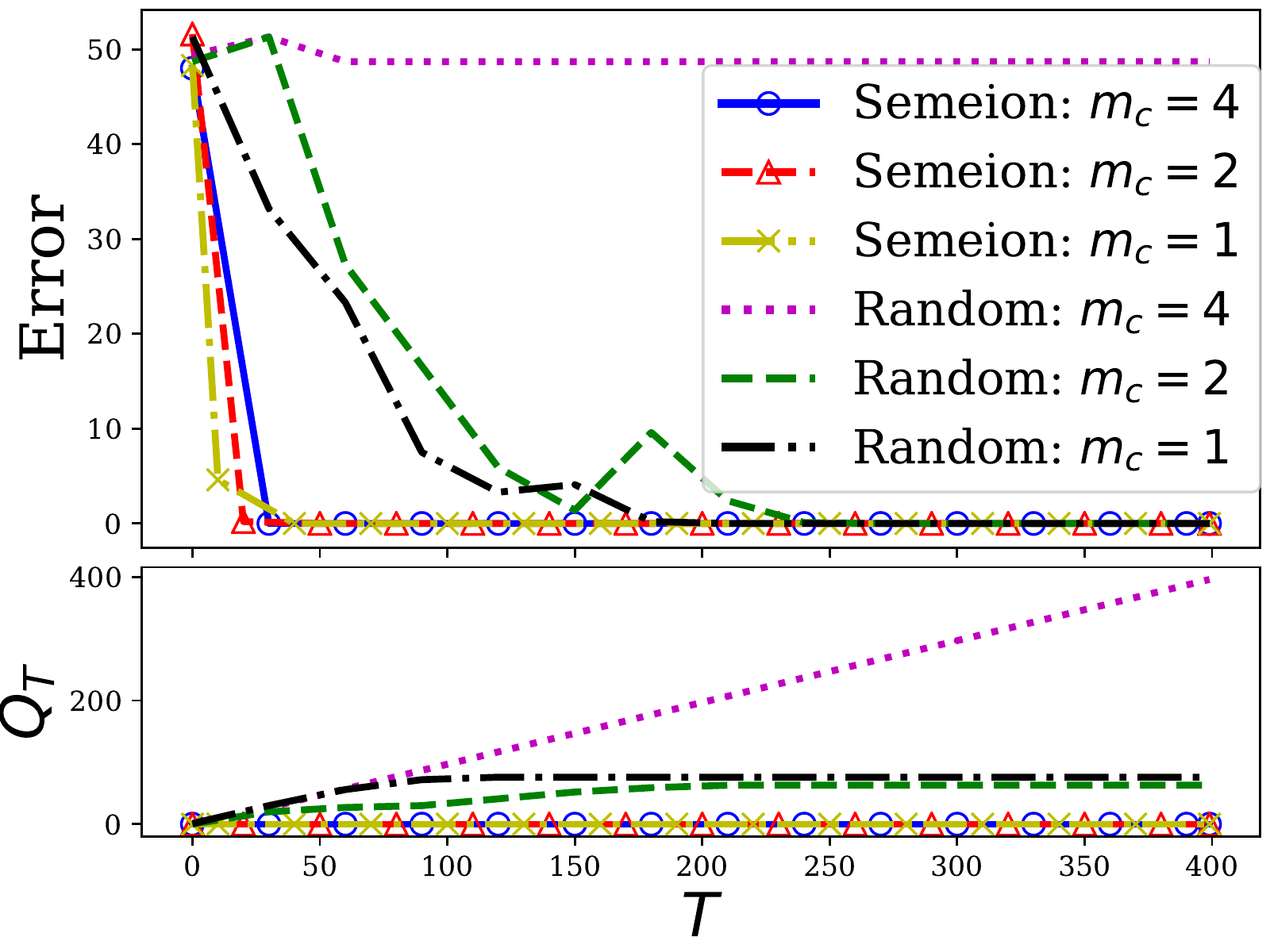}
   \caption{Convolutional network: seed \#2}  
\end{subfigure}   
\caption{Training error and the value of $Q_T$ over time steps $T$. The legend  of  (b) is shown in (c). The value of $Q_T$ measures the number of  $Y_{\ell}$ not satisfying the condition of $\vect(Y_{\ell}) \in \Col(\frac{\partial \vect(f_{X}(\theta^{t}))}{\partial \theta^{t}})$. This figure validates our theoretical understanding that the bound on the global optimality gap decreases at any iteration when  $Q_T$ does not increase\ at the iteration:\ i.e., when the $Q_T$ value is flat. The minimum eigenvalue of the matrix  $\mathbf{M}(\theta^t)=\frac{\partial \vect(f_{X}(\theta))}{\partial \theta} (\frac{\partial \vect(f_{X}(\theta))}{\partial\theta})\T$ is zero at all iterations in Figure \ref{fig:1} (b) and (c) for all cases with $m_c=1$.} 
\label{fig:1}
\end{figure*}

\begin{table*}[t!]
\centering \renewcommand{\arraystretch}{0.8} \fontsize{11.pt}{11.pt}\selectfont
\caption{The change of the gradient representation during training, $\|\mathbf{M}(\theta^{\hat  T})-\mathbf{M}(\theta^{0})\|_{F}^2$, where $\mathbf{M}(\theta):=\frac{\partial \vect(f_{X}(\theta))}{\partial \theta} (\frac{\partial \vect(f_{X}(\theta))}{\partial\theta})\T$ and $\hat T$  is the last time step.  } \label{tbl:3}  \vspace{-0pt}
\begin{subtable}[b]{1.0\textwidth}
\centering 
\caption{Fully-connected network} 
\begin{tabular}{lcc}
\toprule
Dataset &      \\
\midrule
 Two moons    & $2.09 \hspace{-2pt} \times \hspace{-2pt} 10^{11}$  \\
\midrule
 Sine wave  & $3.95 \hspace{-2pt} \times \hspace{-2pt} 10^{9}$   \\
\bottomrule
\end{tabular} 
\end{subtable}
\\ \vspace{10pt}
\begin{subtable}[b]{1.0\textwidth}
\centering 
\caption{Convolutional network} 
\begin{tabular}{lcccccc}
\toprule
Dataset &  \multicolumn{2}{c}{$m_c=4$} & \multicolumn{2}{c}{$m_c=2$} & \multicolumn{2}{c}{$m_c=1$}   \\  \cline{2-7} 
 & seed\#1 & seed\#2  & seed\#1 & seed\#2  & seed\#1 & seed\#2   \\
\midrule
 Semeion    & $8.09 \hspace{-2pt} \times \hspace{-2pt}  10^{12}$     & $5.19 \hspace{-2pt}\times\hspace{-2pt} 10^{12}$  & $9.82 \hspace{-2pt}\times\hspace{-2pt} 10^{12}$     & $3.97 \hspace{-2pt}\times\hspace{-2pt} 10^{12}$   & $2.97 \hspace{-2pt}\times\hspace{-2pt} 10^{12}$     & $5.41\hspace{-2pt}\times\hspace{-2pt} 10^{12}$\\
\midrule
 Random   & $3.73 \hspace{-2pt}\times\hspace{-2pt} 10^{12} $ & $1.64 \hspace{-2pt} \times\hspace{-2pt}  10^{12}$   & $3.43 \hspace{-2pt} \times\hspace{-2pt}  10^{7}$     & $4.86 \hspace{-2pt} \times\hspace{-2pt}  10^{12}$   & $1.40 \hspace{-2pt}\times\hspace{-2pt} 10^{7}$     & $8.57 \hspace{-2pt}\times\hspace{-2pt} 10^{11}$ \\
\bottomrule
\end{tabular} 
\end{subtable}
\end{table*}  

We employ a fully-connected   network having four layers with 300 neurons per hidden layer, and a convolutional network, LeNet \citep{lecun1998gradient}, with five layers. For the fully-connected   network, we use the  two-moons dataset \citep{scikitlearn} and a sine wave dataset. To create the  sine wave dataset, we  randomly generated the input $x_i $  from the uniform distribution on  the interval $[-1,1]$ and set $y_i=\one\{\sin(20x_i)<0\}\in \RR$ for all $i \in [n]$ with $n=100$. For the convolutional   network, we use the Semeion dataset \citep{srl1994semeion} and a random dataset. The random dataset was created  by randomly generating each pixel of the input image  $x_i \in \RR^{16\times 16 \times 1}$  from the standard normal distribution  and by  sampling $y_i$ uniformly from $\{0,1\}$ for all $i \in [n]$ with $n=1000$.  We set the  activation functions of all layers  to be softplus $\bar \sigma(z)=\ln(1+\exp(\varsigma z))/\varsigma$ with $\varsigma=100$, which  approximately behaves as the ReLU activation   as shown in Appendix \ref{sec:app:3} in the Supplementary Information. See Appendix \ref{sec:app:2} in the Supplementary Information for more details of the experimental settings.

The results of the experiments are presented in Figure \ref{fig:1}.
In each subfigure, the training errors and the values of $Q_T$  are plotted  over time $T$. Here, $Q_T$ counts the number of $Y_{\ell}$ \textit{not} satisfying the condition $\vect(Y_{\ell}) \in \Col(\frac{\partial \vect(f_{X}(\theta^{t}))}{\partial \theta^{t}})$ during $t \in \{0,1,\dots,T\}$ and is defined by     
\begin{align}
 Q_T=\sum_{t=0}^T \one\left\{\vect(Y_{\ell}) \notin \Col\left(\frac{\partial \vect(f_{X}(\theta^{t}))}{\partial \theta^{t}}\right)\right\}.
\end{align}
Figure \ref{fig:1} (a) shows the results for  the fully-connected   network. For the two-moons dataset, the network  achieved the zero training error with $Q_T=0$ for all $T$ (i.e., $\vect(Y_{\ell}) \in \Col(\frac{\partial \vect(f_{X}(\theta^{T}))}{\partial \theta^{T}})$ for all $T$). For  the sine wave dataset, it obtained high training errors with $Q_T=T$ for all $T$ (i.e., $\vect(Y_{\ell}) \notin \Col(\frac{\partial \vect(f_{X}(\theta^{T}))}{\partial \theta^{T}})$ for all $T$). This is consistent with our theory. Our theory  explains that what makes a dataset easy to be fitted or not    is  whether the condition  $\vect(Y_{\ell}) \in \Col(\frac{\partial \vect(f_{X}(\theta^{t}))}{\partial \theta^{t}})$ is satisfied or not.

Figure \ref{fig:1} (b) and (c) shows the results for the  convolutional   networks with two  random initial points using two different random seeds. In the subfigures, we report the training behaviors with different network sizes $m_c=1,2$ and $4$: i.e., the number of convolutional filters per  convolutional layer is  $8 \times m_c$ and the number of neurons per fully-connected hidden layer is $128 \times m_c$. 
As can be seen, with the Semeion dataset, the networks of all sizes   achieved the zero  error with $Q_T=0$ for all $T$. With the random dataset,  the deep networks yielded the zero training error whenever $Q_T$ is not linearly increasing over the time, or equivalently  whenever the condition of $\vect(Y_{\ell}) \in \Col(\frac{\partial \vect(f_{X}(\theta^{T}))}{\partial \theta^{T}})$  holds sufficiently many steps $T$.    This is  consistent with our theory.

Finally, we  also confirmed that gradient representation $\frac{\partial f(x,\theta^{t})}{\partial\theta^{t}}$ changed significantly from the initial one $\frac{\partial f(x,\theta^{0})}{\partial\theta^{0}}$   in our experiments. That is,  the values of $\|\mathbf{M}(\theta^{T})-\mathbf{M}(\theta^{0})\|_{F}^2$ were significantly large and tended to increase as $T$ increases, where   the matrix $\mathbf{M}(\theta) \in \RR^{nm_y\times nm_y}$  is defined by
$
 \mathbf{M}(\theta)=\frac{\partial \vect(f_{X}(\theta))}{\partial \theta} (\frac{\partial \vect(f_{X}(\theta))}{\partial\theta})\T. 
$
Table \ref{tbl:3} summarizes the values of  $\|\mathbf{M}(\theta^{T})-\mathbf{M}(\theta^{0})\|_{F}^2$ at the end of the training.

\subsection{Implications} \label{sec:new:3}
In Section \ref{sec:new:new:3}, we showed that  an uncommon model structure  $f(x,\theta) = \theta^4 - 10 \theta^2 + 6 \theta + 100$ does not satisfy Assumption \ref{assump:7} and Assumption \ref{assump:7} is not required for global convergence if  the minimum eigenvalue is  nonzero. However, in practice, we typically use machine learning models that satisfy Assumption \ref{assump:7} instead of the model   $f(x,\theta) = \theta^4 - 10 \theta^2 + 6 \theta + 100$, and the minimum eigenvalue is  zero in many cases. In this context,  Theorem \ref{thm:new:1} provides the justification for common practice in nonlinear representation learning. Furthermore,  Theorem \ref{thm:new:1} (i) contributes to the  literature by identifying the common model structure assumption (Assumption \ref{assump:7}) and the data-architecture alignment condition (Definition \ref{def:2}) as the novel and practical conditions to ensure the global convergence even when the minimum eigenvalue  becomes  zero. Moreover, Theorem \ref{thm:new:1} (ii) shows that this condition is not  arbitrary in the sense that it is also necessary to obtain the globally optimal models.  Furthermore, the data-architecture alignment condition is strictly more general than the condition of the minimum eigenvalue being nonzero, in the sense that the latter implies the former but not vice versa.

Our new theoretical understanding based on the data-architecture alignment condition can explain and deepen the previously known empirical observation that increasing the network size tends to improve the training behaviors. Indeed, the size of networks seems to correlate well with the training error  to a certain degree in Figure \ref{fig:1} (b). However, the size and the training error do not correlate well   in  Figure \ref{fig:1} (c). Our new theoretical understanding explains that the training behaviors  correlate more directly with the data-architecture alignment condition of $\vect(Y_{\ell}) \in \Col(\frac{\partial \vect(f_{X}(\theta^{t}))}{\partial \theta^{t}})$ instead. The seeming correlation with the network size is indirect and caused by another correlation between the network size and the condition of $\vect(Y_{\ell}) \in \Col(\frac{\partial \vect(f_{X}(\theta^{t}))}{\partial \theta^{t}})$. That is, the condition of $\vect(Y_{\ell}) \in \Col(\frac{\partial \vect(f_{X}(\theta^{t}))}{\partial \theta^{t}})$ more likely tends to  hold  when the network size is larger because the matrix $\frac{\partial \vect(f_{X}(\theta^{t}))}{\partial \theta^{t}}$ is of size $n m_y \times d$ where $d$ is the number of parameters: i.e., by increasing $d$, we can increase the column space $\Col(\frac{\partial \vect(f_{X}(\theta^{t}))}{\partial \theta^{t}})$ to increase the chance of satisfying the condition of $\vect(Y_{\ell}) \in \Col(\frac{\partial \vect(f_{X}(\theta^{t}))}{\partial \theta^{t}})$.  

Note that the minimum eigenvalue of the matrix  $\mathbf{M}(\theta^t)=\frac{\partial \vect(f_{X}(\theta))}{\partial \theta} (\frac{\partial \vect(f_{X}(\theta))}{\partial\theta})\T$ is zero at all iterations in Figure \ref{fig:1} (b) and (c) for all cases of $m_{c}=1$. Thus, Figure \ref{fig:1} (b) and (c) also illustrates the fact that while having the zero  minimum eigenvalue  of the matrix  $\mathbf{M}(\theta^t)$, the dynamics can achieve the global convergence under the data-architecture alignment condition.  Moreover, because the multilayer neural network in the lazy training regime \citep{kawaguchi2021recipe} achieve zero training errors for \textit{all} datasets, Figure \ref{fig:1} additionally illustrates the fact that our theoretical and empirical results apply to the models outside of the lazy training regime and can distinguishing `good' datasets from 'bad' datasets given a learning algorithm.

In sum, our new theoretical understanding has the ability to explain and distinguish the successful case and failure case based on  the data-architecture alignment condition for the common machine learning models. Because the data-architecture alignment condition is dependent on data and architecture, Theorem \ref{thm:new:1} along with our experimental results show why and when the global convergence in nonlinear representation learning is achieved based on the relationship between the data $(X,Y)$ and architecture $f$. This new understanding is used in Section \ref{sec:14} to derive a practical algorithm, and is expected to be a basis for many future algorithms.

\subsection{Details and Formalization of Theorem \ref{thm:new:1}} \label{sec:2}
This subsection presents the precise mathematical statements that formalize the informal description of Theorem \ref{thm:new:1}. In the following subsections, we devide the formal version of Theorem \ref{thm:new:1} into Theorem \ref{thm:3}, Theorem \ref{thm:p:1}, and Proposition \ref{prop:1}.

\subsubsection{Preliminaries} 

Let $(\theta^{t})_{t=0}^\infty$ be  the sequence defined by 
$
\theta^{t+1} = \theta^{t}-\alpha^t \bg^{t}
$ with an initial parameter vector $\theta^{0}$, a learning rate $\alpha^t$, and an update vector $\bg^t$. The analysis in this section relies on   the following assumption on the update vector $\bg^{t}$:   

\begin{assumption} \label{assump:5}
There exist $\bc,\uc>0$ such that 
$
\uc \|\nabla\Lcal(\theta_{}^{t})\|^2  \allowbreak \le\nabla\Lcal(\theta_{}^{t})\T \bg^t 
$
and
$
 \|\bg^t\|^2 \le \bc \|\nabla\Lcal(\theta_{}^{t})\|^2
$ for any $t \ge 0$.
\end{assumption}

Assumption \ref{assump:5} is satisfied by using  
$
\bg^t =  D^t\nabla\Lcal(\theta^{t}), 
$
where $D^{t}$ is any positive definite symmetric  matrix with eigenvalues in the interval  $[\uc, \sqrt{\bc}]$. If we set $D^{t}=I$,  we have gradient descent and Assumption \ref{assump:5} is satisfied with $\uc=\bc =1$.  This section also uses the standard assumption of  differentiability and Lipschitz continuity:  

\begin{assumption} \label{assump:6}

 For every $i \in [n]$, the function $\ell_i:q\mapsto\ell(q,y_{i}) $  is differentiable and convex, the map    $f_{i}:\theta \mapsto f(x_i, \theta)$ is differentiable, and  $\|\nabla\Lcal(\theta)-\nabla\Lcal(\theta')\|\le L \|\theta-\theta'\|$ for all $\theta,\theta'$ in the domain of $\Lcal$ for some $L \ge 0$. 
\end{assumption}

The assumptions on the loss function in Assumption \ref{assump:6} are satisfied by using standard loss functions, including the squared loss, logistic loss, and  cross entropy loss. Although the objective function $\Lcal$  is non-convex and non-invex, the  function $q  \mapsto\ell(q, y_{i})$ is typically convex.

For any matrix $Y^*= (y_{1}^{*}, y_{2}^*,\dots, y_{n}^*)\T \in \RR^{n \times m_y}$, we  define 
\begin{align}
\Lcal^*(Y^*) =\frac{1}{n}\sum_{i=1}^n  \ell\left(y_{i}^*, y_{i}\right).
\end{align}
For example, for the squared loss $\ell$, the value of  $\Lcal^*(Y_{\ell})$  is at most the global minimum value of $\Lcal$ as  
\begin{align} \label{eq:p:3}
\Lcal^*(Y_{\ell}) \le  \Lcal(\theta), \ \forall \theta\in\RR^d,   
\end{align}
since $\Lcal^*(Y_{\ell}) =\frac{1}{n}\sum_{i=1}^n  \|y_{i}-y_{i}\|^2_2=0\le \Lcal(\theta)\ \forall \theta \in \RR^d$. This paper also uses the notation of    $[k]=\{1,2,\dots,k\}$ for any  $k \in \NN^+$ and $\|\cdot\|=\|\cdot\|_2$  (Euclidean norm). Finally, we note that for any $\eta \in\RR$, the condition of $\vect(Y_{\ell}) \in \Col(\frac{\partial \vect(f_{X}(\theta))}{\partial \theta})$ is necessary to learn a near global optimal model  $f_X(\theta) = \eta Y_{\ell}$: 

\begin{proposition} \label{prop:1}
Suppose Assumption \ref{assump:7} holds. If $\vect(Y_{\ell}) \notin \Col(\frac{\partial \vect(f_{X}(\theta))}{\partial \theta})$, then $f_X(\theta) \neq \eta Y_{\ell}$ for any $\eta\in\RR$.
\end{proposition}
\begin{proof}
All  proofs of this paper are   presented  in Appendix \ref{sec:app:1} in the Supplementary Information.
\end{proof}

\subsubsection{Global Optimality at the Limit Point} \label{sec:4}

The following theorem shows that  every limit point $\htheta$ of the sequence $(\theta^t)_{t}$ achieves a    loss value $\Lcal(\htheta)$ no worse than $ \inf_{\eta \in \RR}\Lcal^*(\eta Y^{*})$ for any $Y^*$ such that   $\vect(Y^{*}) \in \Col(\frac{\partial \vect(f_{X}(\theta^{t}))}{\partial \theta^{t}})$ for all $t \in[\tau,\infty)$ with some $\tau\ge 0$: 

\begin{theorem} \label{thm:3}
Suppose Assumptions \ref{assump:7}--\ref{assump:6} hold. Assume that the learning rate sequence $(\alpha^t)_{t}$ satisfies either (i) $\epsilon \le \alpha^t \le \frac{\uc (2-\epsilon)}{L\bc}$ for some $\epsilon>0$, 
or (ii) $\lim_{t \rightarrow \infty}\alpha^t =0$ and $\sum_{t=0}^\infty \alpha^t = \infty$. Then, for any $Y^* \in\RR^{n \times m_y}$, if there exists  $\tau\ge 0$ such that   $\vect(Y^{*}) \in \Col(\frac{\partial \vect(f_{X}(\theta^{t}))}{\partial \theta^{t}})$ for all $t \in[\tau,\infty)$,  every limit point $\htheta$ of the sequence $(\theta^t)_{t}$ satisfies
\begin{align}
\Lcal(\htheta)\le \Lcal^*(\eta Y^{*}), \ \ \ \forall \eta \in \RR.
\end{align}
\end{theorem}

For example, for the squared loss $\ell(q,y)=\|q-y\|^2$, Theorem \ref{thm:3} implies that  every limit point $\htheta$ of the sequence $(\theta^t)_{t}$ is a    global minimum as 
\begin{align}
\Lcal(\htheta)\le \Lcal(\theta), \ \ \forall \theta \in \RR^d,
\end{align} 
if   $\vect(Y_{\ell}) \in \Col(\frac{\partial \vect(f_{X}(\theta^{t}))}{\partial \theta^{t}})$ for   $t \in[\tau,\infty)$ with some $\tau\ge 0$. This is because  $\Lcal(\htheta)\le \Lcal^*(Y_{\ell}) \le \Lcal(\theta) \ \forall \theta \in \RR^d$ from  Theorem \ref{thm:3} and equation \eqref{eq:p:3}.

In practice, one can easily satisfy  all the assumptions in Theorem \ref{thm:3}  except for the condition that $\vect(Y^{*}) \in \Col(\frac{\partial \vect(f_{X}(\theta^{t}))}{\partial \theta^{t}})$ for all $t \in[\tau,\infty)$. Accordingly, we will now weaken this condition by analyzing optimality at each iteration so that the condition is verifiable in experiments.

\subsubsection{Global Optimality Gap at Each Iteration} \label{sec:5}

The following theorem states that under standard  settings, the sequence $(\theta^t)_{t\in\Tcal}$ converges to a    loss value no worse than $ \inf_{\eta \in \RR}\Lcal^*(\eta Y^{*})$ at the rate of $O(1/\sqrt{|\Tcal|} )$  for any  $\Tcal$ and $Y^*$ such that $\vect(Y^{*}) \in \Col(\frac{\partial \vect(f_{X}(\theta^{t}))}{\partial \theta^{t}})$ for $t \in\Tcal$: 

\begin{theorem} \label{thm:p:1}
Suppose Assumptions \ref{assump:7} and \ref{assump:6} hold. Let $(\alpha^t, \bg^t)=(\frac{2\alpha}{L},\nabla \Lcal(\theta^{t}))$ with an arbitrary $\alpha \in(0,1) $.  Then, for any  $\Tcal\subseteq \NN_0$ and $Y^* \in\RR^{n \times m_y}$ such that $\vect(Y^{*}) \in \Col(\frac{\partial \vect(f_{X}(\theta^{t}))}{\partial \theta^{t}})$ for all $t \in\Tcal$,
it holds that  
\begin{align}
\min_{t\in \Tcal}\Lcal(\theta^{t}) \le\Lcal^*(\eta Y^{*}) + \frac{1}{\sqrt{|\Tcal|}}    \sqrt{\frac{L\zeta_{\eta}\Lcal(\theta^{t_{0}})}{2\alpha(1-\alpha)}},
\end{align} 
for  any $\eta\in \RR$, where  $t_0 = \min\{t:t \in \Tcal\}$, $\zeta_{\eta}:= 4 \max_{t \in\Tcal}\max(\|\nu(\theta^{t})\|^{2} , \|\hat \beta(\theta^{t},\eta)\|^2 )$,  $\hat \beta(\theta,\eta):=\eta ((\frac{\partial \vect(f_{X}(\theta))}{\partial \theta})\T \frac{\partial \vect(f_{X}(\theta))}{\partial \theta})^{\dagger}  (\frac{\partial \vect(f_{X}(\theta))}{\partial \theta})\T \vect(Y^{*})$, and   $\nu(\theta)_k=\theta_k$ for all $k\in S$ and  $\nu(\theta)_k=0$ for all  $k \notin S$ with the set $S$ being defined in Assumption \ref{assump:7}.
\end{theorem}

For the squared loss $\ell$, Theorem \ref{thm:p:1} implies the following for any $T\ge1$:  for any $\Tcal\subseteq [T]_{}$ such that   $\vect(Y_{\ell}) \in \Col(\frac{\partial \vect(f_{X}(\theta^{t}))}{\partial \theta^{t}})$ for all $t \in\Tcal$, we have
\begin{align} \label{eq:19}
\min_{t \in [T]_{}}\Lcal(\theta^{t}) \le\inf_{\theta \in \RR^d} \Lcal(\theta)+O(1/\sqrt{|\Tcal|} ) .
\end{align}
This is because $\Lcal^*(Y_{\ell}) \le \inf_{\theta \in \RR^d} \Lcal(\theta)$ from equations \eqref{eq:p:3}  and $\min_{t \in [T]_{}}\Lcal(\theta^{t}) \le\min_{t \in\Tcal} \allowbreak \Lcal(\theta^{t})$ from $\Tcal\subseteq [T]_{}$.

Similarly, for the binary and multi-class cross-entropy losses,  Theorem \ref{thm:p:1} implies the following  for any $T\ge1$:  for any $\Tcal\subseteq [T]_{}$ such that   $\vect(Y_{\ell}) \in \Col(\frac{\partial \vect(f_{X}(\theta^{t}))}{\partial \theta^{t}})$ for all $t \in\Tcal$, we have
that for any  $\eta \in \RR$,
\begin{align} \label{eq:p:5}
\min_{t \in [T]_{}}\Lcal(\theta^{t}) \le \Lcal^*(\eta Y_\ell) +O(\eta^{2} /\sqrt{|\Tcal|} ).
\end{align}
Given any desired   $\epsilon>0,$ since $\Lcal^*(\eta Y_\ell) \rightarrow 0$ as $\eta \rightarrow \infty$, setting $\eta$ to be sufficiently large obtains  the desired  $\epsilon$ value  as  $\min_{t \in \Tcal}\Lcal(\theta^{t})\le \epsilon $  in equation \eqref{eq:p:5} as $\sqrt{|\Tcal|} \rightarrow \infty$.

\section{Application to the Design of Training Framework} \label{sec:14}

The  results in the previous section show  that the bound on the global optimality gap decreases per iteration whenever  the data-architecture alignment condition  holds. Using this theoretical understanding, in this section, we propose a  new training framework with  prior guarantees while learning hierarchical nonlinear representations \textit{without} assuming the data-architecture alignment condition. As a result,  we made significant improvements over  the most closely related  study on global convergence guarantees \citep{kawaguchi2021recipe}. In particular, whereas the related study requires  a wide layer with a width larger than $n$, our results  reduce the requirement to  a layer with a width larger than $\sqrt{n}$. For example,  the MNIST dataset has $n=60000$ and hence previous studies require $ 60000$ neurons at a layer, whereas we only require  $\sqrt{60000} \approx 245$ neurons  at a layer. Our requirement  is consistent and satisfied by the models used in practice that typically have from 256 to 1024 neurons for some layers.     

We begin in Subsection \ref{sec:6} with  additional notations,  and then present the training framework in Subsection \ref{sec:7} and convergence analysis  in Subsection \ref{sec:8}. We conclude in Subsection \ref{sec:9} by providing empirical evidence to support our theory.

\subsection{Additional Notations} \label{sec:6}

We denote by $\theta_{(l)}\in \RR^{d_l}$ the vector of all the trainable parameters at the $l$-th layer \rev{for $l=1,\dots, H$ where $H$ is the depth or the number of trainable layers (i.e., one plus the number of hidden layers). That is, the $H$-th layer is the last layer containing the trainable parameter vector  $\theta_{(H)}$ at the last layer}. For any pair $(l,l')$ such that $1\le l \le l' \le H$, we define   $\theta_{(l:l')}=[\theta_{(l)}\T,\dots,\theta_{(l')}\T]\T  \in \RR^{d_{l:l'}}$: for example,  $\theta_{(1:H)}=\theta$ and $\theta_{(l:l')}=\theta_{(l)}$ if $l=l'$.
We  consider a family of training algorithms that update the parameter vector $\theta$ as follows: for each $l=1,\dots,H$,   
\begin{align}
\theta_{(l)}^{t+1} =\theta_{(l)}^{t} - g_{(l)}^t, \qquad g^t_{(l)} \sim \Gcal_{(l)}(\theta^{t}, t)
\end{align} 
where the function $\Gcal_{(l)}$ outputs a distribution over the  vector $g^t_{(l)}$ and differs for different training algorithms. For example, for  (mini-batch) stochastic gradient descent (SGD), $g^t_{(l)}$ represents a product of a learning rate and a stochastic gradient with respect to $\theta_{(l)}$ at the time $t$. We define $\Gcal=(\Gcal_{(1)},\dots,\Gcal_{(H)})$ to represent a training algorithm. 

For an arbitrary matrix $M \in \RR^{m \times m'}$, we let  $M_{*j}$ be its $j$-th column vector in $\RR^m$,  $M_{i*}$ be   its $i$-th row vector in $\RR^{m'}$, and $\rank(M)$ be its matrix rank. We define  $M \circ M'$  to be the Hadamard product of any matrices  $M$ and $M'$. For any vector $v\in\RR^m$, we let  $\diag(v)\in \RR^{m\times m}$  be the diagonal matrix with $\diag(v)_{ii}=v_{i}$ for $i \in [m]$. We denote by $I_{m}$ the $m \times m$ identity matrix.

\subsection{Exploration-Exploitation Wrapper} \label{sec:7}

In this section, we introduce the Exploration-Exploitation (EE) wrapper    $\Acal$. The EE wrapper $\Acal$    is not a stand-alone training algorithm. Instead, the EE wrapper $\Acal$  takes any  training algorithm $\Gcal$ as its input, and runs the   algorithm $\Gcal$ in a particular way to guarantee  global convergence. We note that the exploitation phase in the EE wrapper does \textit{not} optimize the last layer, and instead it optimizes hidden layers, whereas the exploration phase optimizes all layers.  The EE wrapper allows us to  learn the representation $\frac{\partial f(x,\theta^{t})}{\partial\theta^{t}}$ that differs  significantly from the initial representation $\frac{\partial f(x,\theta^{0})}{\partial\theta^{0}}$ without making assumptions on the minimum eigenvalue of the matrix $\frac{\partial \vect(f_{X}(\theta))}{\partial \theta} (\frac{\partial \vect(f_{X}(\theta))}{\partial\theta})\T$ by leveraging the data-architecture alignment condition. \rev{The data-architecture alignment condition is  ensured by the safe-exploration condition (defined  in Section \ref{sec:new:new:new:1}), which is time-independent and holds in practical common architectures (as demonstrated in Section \ref{sec:9}).}

\subsubsection{Main Mechanisms}

Algorithm \ref{algorithm:EEW} outlines the EE wrapper $\Acal$. During the exploration phase in lines  3--7 of Algorithm \ref{algorithm:EEW}, the EE wrapper $\Acal$  freely explores hierarchical nonlinear representations to be learned without any restrictions.  Then, during the exploitation phase in lines  8--12, it starts exploiting the current knowledge   to ensure $\vect(Y_{\ell}) \in \Col(\frac{\partial \vect(f_{X}(\theta^{t}))}{\partial \theta^{t}})$ for all $t$ to guarantee global convergence. The value of $\tau$ is the hyper-parameter that controls the time when it transitions from the exploration phase to the exploitation phase.

In the exploitation phase,  the wrapper $\Acal$ only optimizes the parameter vector $\theta_{(H-1)}^{t}$ at the $(H-1)$-th \textit{hidden} layer, instead of  the parameter vector $\theta_{(H)}^{t}$ at the last layer or the $H$-th layer. Despite this,  the EE wrapper $\Acal$  is proved to converge to global minima  of \textit{all} layers in $\RR^d$. The exploitation phase still allows us to significantly change the representations  as $\mathbf{M}(\theta^{t})\not\approx\mathbf{M}(\theta^\tau)$ for $t > \tau$. This is  because we optimize the hidden layers
instead of the last layer without any significant over-parameterizations. 

The exploitation phase uses an arbitrary optimizer $\tGcal$ with the update vector   $\tg^{t}_{}\sim \tGcal_{(H-1)}(\theta^{t}, t)$ with $\tg^t =\alpha^t \hg^{t}\in \RR^{d_{H-1}}$. During the two phases, we can use the same optimizers (e.g., SGD\ for both $\Gcal$ and $\tGcal$) or  different optimizers (e.g.,   SGD for $\Gcal$ and L-BFGS for $\tGcal$).

\subsubsection{Model Modification} \label{sec:10}

\begin{algorithm}[t!] 
\caption{$\Acal$: Exploration-Exploitation (EE)\ wrapper} \label{algorithm:EEW} \begin{algorithmic}[1] 
\STATE {\bf Inputs:} a training algorithm $\Gcal$ and a base model $\baf$. \STATE Modify the base model $\baf$ to the model $f$ 
\STATE
\textbf{Exploration phase:} 
\STATE Initialize the parameter vector $\theta^0$ of the model $f$  
\begin{ALC@g}
\FOR {$t = 0,1,\ldots, \tau-1$} 
\STATE  \hspace{10pt} $\theta_{(l)}^{t+1} =\theta_{(l)}^{t} - g_{(l)}^t, \ \ g^t_{(l)} \sim \Gcal_{(l)}(\theta^{t}, t), \ \forall l \in [H]$
\ENDFOR
\end{ALC@g}
\STATE \textbf{Exploitation phase:} 
\begin{ALC@g}
\STATE 
Set
$
\theta^{\tau}= \theta^{\tau-1}+\varepsilon \delta \ \text{where } \delta \sim \Ncal_d(0, I)
$
\STATE 
Replace $\sigma_{j}$ by $\sigma_{R^{(j)}}$ with $R^{(j)}=\theta_{(H-1,j)}^{{\tau}}$ 
\FOR {$t = \tau,\tau+1,\ldots $} 
\STATE  \hspace{10pt} $\theta_{(H-1)}^{t+1} =\theta_{(H-1)}^{t} - \tg^t, \ \ \  \tg^t \sim\tGcal_{(H-1)}(\theta^{t}, t)$ 
\ENDFOR   
\end{ALC@g} 
\end{algorithmic}
\end{algorithm}

This subsection defines the details of the model modification at   line  2 of Algorithm \ref{algorithm:EEW}. Given any base  network $\baf$, the wrapper $\Acal$ first checks whether or not the last  two layers of the given network $\baf$ are fully connected. If not, one or two fully-connected last layers  are added
such that  the output of the  network $\baf$ can be written by
$
\baf(x,\bar \theta) = \bW^{(H)} \bar \sigma(\bW^{(H-1)}z_{}(x,\theta_{(1:H-2)})).
$ 
Here, $z(x,\theta_{(1:H-2)})$ is the output of the $(H-2)$-th layer, and the function   $z$ is  arbitrary and can represent various deep networks. Moreover, $\bar \sigma$ is a nonlinear activation function, and $\bW^{(H-1)}$ and $\bW^{(H)} $ are the  weight matrices of the last two layers.  
The wrapper $\Acal$  then modifies these  last two layers as follows.
In  the case of $m_y=1$, the  model  $\baf$   is modified to 
\begin{align} \label{eq:p:2}
f(x,\theta)= W^{(H)}\sigma(W^{(H-1)},z(x,\theta_{(1:H-2)})),
\end{align}
 where  $W^{(H-1)}\in \RR^{m_{H} \times m_{H-1}}$ and $W^{(H)} \in \RR^{m_y \times m_H}$ are the weight matrices of the last two layers. The nonlinear activation $\sigma$ is defined by
$
\sigma(q,q')=\tsigma(qq')\circ (qq'),
$
where $\tsigma$ is some nonlinear function. For example,  we can set   $\tsigma(q)= \frac{1}{1+e^{-\varsigma'  q}}$ (sigmoid) with any hyper-parameter $\varsigma'>0 $, for which it holds that as $\varsigma' \rightarrow \infty$,
\begin{align}
f(x,\theta) \rightarrow W^{(H)}\mathrm{relu}( W^{(H-1)}z_{}(x,\theta_{(1:H-2)})).
\end{align}
We generalize equation \eqref{eq:p:2}  to the case of  $m_y\ge2$ as: 
\begin{align}
f(x,\theta)_{j}=W^{(H)}_{j*} \sigma_{j}(W^{(H-1,j)},z_{}(x,\theta_{(1:H-2)})),
\end{align}
for $j\in[m_{y}]$ where $W^{(H-1,j)}\in \RR^{m_{H} \times m_{H-1}}$ is the  weight matrix at the $(H-1)$-th layer, and $\sigma_{j}=\sigma$ until line 10.  At line 10, the wrapper $\Acal$ replaces $\sigma_{j}$ by $\sigma_{R^{(j)}}$ where
$
\sigma_{R^{(j)}}(q,q')=\tsigma(R^{(j)}q')\circ (qq')
$ with $R^{(j)}=\theta_{(H-1,j)}^\tau$ and   $\theta_{(H-1,j)}=(W^{(H-1,j)})\T$. To consider the bias term,  we include the constant neuron to the output of the $(H-1)$-th layer as $z(x,\theta_{(1:H-2)})=[\bar z(x,\theta_{(1:H-2)})\T, 1]\T\in \RR^{m_{H-1}}$ where $\bar z(x,\theta_{(1:H-2)})$ is the output without the constant neuron.

\subsection{Convergence Analysis} \label{sec:8}

In this subsection, we establish global convergence of the EE wrapper $\Acal$ \textit{without} using  assumptions from the previous section.
Let $\tau$ be an arbitrary positive integer and $\varepsilon$ be an arbitrary positive real number. Let $(\theta^{t})_{t=0}^\infty$ be a sequence generated by the EE\ wrapper $\Acal$. We define $\hLcal(\theta_{(H-1)})=\Lcal(\theta_{(1:H-2)}^{\tau},\theta_{(H-1)},\theta_{(H)}^\tau) $ and  $B_{\bepsilon} = \min_{\theta_{(H-1)} \in\Theta_{\bepsilon} }\|\theta_{(H-1)}- \theta_{(H-1)}^{\tau}\|^{}$ where $ \Theta_{\bepsilon}=\argmin_{\theta_{(H-1)}} \max(\hLcal(\theta_{(H-1)}), \bepsilon)$ for any $\bepsilon \ge0$.

\subsubsection{Safe-exploration Condition} \label{sec:new:new:new:1}

The mathematical analysis in this section relies on the \textit{safe-exploration condition}, which is what allows us to safely explore  deep nonlinear representations in the exploration phase without getting stuck in the states of $\vect(Y_{\ell}) \notin \Col(\frac{\partial \vect(f_{X}(\theta^{t}))}{\partial \theta^{t}})$. The safe-exploration condition is verifiable, time-independent, data-dependent and architecture-dependent. The verifiability and  time-independence makes the assumption strong enough to provide prior guarantees
before training. The data-dependence and architecture-dependence make the assumption weak enough to be applicable for a wide range of practical settings.   

For any $q\in \RR^{ m_{H-1}\times m_{H} }$, we define the matrix-valued function $\phi(q,\theta_{(1:H-2)})\in \RR^{n \times m_H m_{H-1}}$ by 
$$
\phi(q,\theta_{(1:H-2)})=\begin{bmatrix}\tsigma(z_1 \T q_{*1}) z_1\T & \cdots & \tsigma(z_1 \T q_{*m_{H}}^{}) z_1\T \\
\vdots & \ddots\ & \vdots \\
\tsigma(z_n \T q_{*1}^{}) z_n\T & \cdots & \tsigma(z_n \T q_{*m_{H}}^{}) z_n\T \\
\end{bmatrix},  
$$
where $z_i = z(x_i,\theta_{(1:H-2)})\in \RR^{m_{H-1}}$ and $\tsigma(z_i \T q_{*k}) z_i\T \in \RR^{1 \times m_{H-1}}$ for all $i \in [n]$ and $k \in [m_H]$. Using this function, the safe-exploration condition
is formally stated as:
\begin{assumption}
\label{assump:1}
\emph{(Safe-exploration condition)} 
There exist a $q\in \RR^{ m_{H-1}\times m_{H}}$ and a $\theta_{(1:H-2)}\allowbreak \in \RR^{d_{1:H-2}}$ such that $\rank(\phi(q,\theta_{(1:H-2)}))=n$.
\end{assumption}
The safe-exploration condition asks for only the \textit{existence} of one parameter vector in the network architecture such that  $\rank(\phi(q,\theta_{(1:H-2)}))=n$. It is \textit{not} about  the training trajectory $(\theta^t)_t$. Since the matrix $\phi(q,\theta_{(1:H-2)})$  is of size $n \times m_H m_{H-1}$, the safe-exploration condition does \textit{not} require any wide layer of size $m_H \ge n $ or $ m_{H-1}\ge n$. Instead, it  requires a  layer of size    $ m_H m_{H-1}\ge  n$. This is a significant improvement over the most closely related study \citep{kawaguchi2021recipe} where the wide layer of size $m_H \ge n$ was required. Note that having $m_H m_{H-1}\ge  n$ does \textit{not} imply the safe-exploration condition. Instead,  
$m_H m_{H-1}\ge  n$ is a \textit{necessary} condition to satisfy the safe-exploration condition, whereas  $m_H \ge n $ or $ m_{H-1}\ge n$ was a necessary condition to satisfy  assumptions in previous papers, including the most closely related study  \citep{kawaguchi2021recipe}. The safe-exploration condition  is  verified in experiments in Subsection \ref{sec:9}.

\subsubsection{Additional Assumptions}

We also use the following assumptions:

\begin{assumption}\label{assump:3}
For any $i \in [n]$, the function $\ell_i:q\mapsto\ell(q,y_{i}) $ is differentiable, and  $\|\nabla \ell_{i}(q) - \nabla \ell_{i}(q')\| \le L_\ell \|q - q'\|$ for all $q,q' \in \RR$. 
\end{assumption}
\begin{assumption}\label{assump:2}
For each $i \in [n]$,
the functions $\theta_{(1:H-2)} \mapsto z(x_{i},\theta_{(1:H-2)})$ and $q\mapsto \tsigma(q)$ are real analytic.
\end{assumption}
Assumption \ref{assump:3} is satisfied by using standard loss functions such as the squared loss  $\ell(q,y)= \|q-y\|^2$ and
cross entropy loss  $\ell(q,y)=-\sum_{k=1}^{d_y}y_k \log \frac{\exp(q_k)}{\sum_{k'}\exp(q_{k'})}$.
The assumptions of the invexity and convexity of  the function $q\mapsto\ell(q,y_{i})$ in Subsections \ref{sec:11}--\ref{sec:12} 
also  hold  for these standard loss functions.
Using $L_\ell$ in Assumption \ref{assump:3}, we define $\hL = \frac{L_\ell}{n^{}}  \|Z\|^2$ where $Z \in \RR^n$ is defined by    
$
Z_i = \max_{j\in[m_y]}\|[\diag(\theta_{(H,j)}^\tau )\otimes I_{m_{H-1}}](\phi(\theta_{(H-1,j)}^\tau, \allowbreak \theta_{(1:H-2)}^{\tau}\allowbreak )_{i*})\T\|
$ with  $\theta_{(H,j)}=(W^{(H)}_{j*})\T$.

Assumption \ref{assump:2} is satisfied by using any  analytic activation function such as sigmoid, hyperbolic tangents  and softplus activations $q\mapsto \ln(1+\exp(\varsigma q))/\varsigma$ with any hyperparameter $\varsigma>0$. This is because a composition of real analytic functions is real analytic and the following  are all real analytic functions in $\theta_{(1:H-2)}$: the convolution, affine map, average pooling, skip connection, and batch normalization. Therefore, the assumptions can be satisfied by using a wide range of machine learning models, including deep neural networks with convolution, skip connection, and batch normalization.  Moreover, the softplus activation can approximate the ReLU activation  for any desired accuracy: i.e.,  
$\ln(1+\exp(\varsigma q))/\varsigma\rightarrow \mathrm{relu}(q) \text{ as } \varsigma\rightarrow \infty 
$, where $\mathrm{relu}$ represents the ReLU activation.

\subsubsection{Global Optimality at the Limit Point} \label{sec:11}
The following theorem proves the global optimality at  limit points of the EE\ wrapper with a wide range of optimizers, including gradient descent and modified Newton methods:
\begin{theorem} \label{thm:1}
Suppose Assumptions \ref{assump:1}--\ref{assump:2} hold and that  the function $\ell_i:q\mapsto\ell(q,y_{i}) $ is invex for any $i \in [n]$. Assume that there exist $\bc,\uc>0$ such that 
$
 \uc \|\nabla\hLcal(\theta_{(H-1)}^{t})\|^2 \le\nabla\hLcal(\theta_{(H-1)}^{t})\T \hg^t  
$
and
$
 \|\hg^t\|^2 \le \bc \|\nabla\hLcal(\theta_{(H-1)}^{t})\|^2$ for any $t \ge \tau$. Assume that the learning rate sequence $(\alpha^t)_{t\ge \tau}$ satisfies either
(i) $\epsilon \le \alpha^t \le \frac{\uc (2-\epsilon)}{\hL\bc}$ for some $\epsilon>0$, or (ii) $\lim_{t \rightarrow \infty}\alpha^t =0$ and $\sum_{t=\tau}^\infty \alpha^t = \infty$. Then with probability one, every limit point $\htheta$ of the sequence $(\theta^t)_{t}$ is a global minimum
 of $\Lcal$ as $
\Lcal(\htheta) \le \Lcal(\theta)$ for all $\theta \in \RR^d$. 
\end{theorem}

\subsubsection{Global Optimality Gap at Each Iteration} \label{sec:12}
We  now present global convergence   guarantees of the EE wrapper $\Acal$ with gradient decent and SGD:

\begin{theorem} \label{thm:2}
Suppose Assumptions \ref{assump:1}--\ref{assump:2} hold  and that the function $\ell_i:q\mapsto\ell(q,y_{i}) $ is convex  for any $i \in [n]$. Then, with probability one, the following two statements hold:
\begin{enumerate}[leftmargin=0.5cm]
\item[(i)] (Gradient descent) if $\hg^t =\nabla\hLcal(\theta_{(H-1)}^{t})$ and $\alpha^t = \frac{1}{\hL}$ for  $t\ge \tau$,
then  for any $\bepsilon\ge0$ and $t> \tau$,
$$
\Lcal(\theta^{t})\le\inf_{\theta\in \RR^d}\max( \Lcal( \theta),\bepsilon)+  \frac{B^{2}_{\bepsilon }\hL_{} }{2(t-\tau)}. 
$$ 
 
\item[(ii)] 
(SGD) if $\EE [\hg^t |\theta^{t}]=\nabla\hLcal(\theta_{(H-1)}^{t})$ (almost surely) with $\EE[\|\hg^t \|^{2}]\le G^{2}$,  and if $\alpha^t \ge 0$, $\sum_{t=\tau}^{\infty} (\alpha^t)^2 < \infty $ and $\sum_{t=\tau}^{\infty} \alpha^t = \infty$ for $t\ge \tau$, then  for any $\bepsilon\ge0$ and $t> \tau$, 
\begin{align}
\EE[\Lcal(\theta^{t^*}) ] \le \inf_{\theta\in \RR^d}\max( \Lcal( \theta),\bepsilon)+    \frac{B^{2}_{ \bepsilon }+G^{2}  \sum_{k=\tau}^t \alpha_{k}^{2}}{ 2 \sum_{k=\tau}^t \alpha_{k}}
\end{align}
where $t^*\in \argmin_{k \in \{\tau,\tau+1,\dots, t\}}\Lcal(\theta^{k})$.  
\end{enumerate}
\end{theorem}
In Theorem \ref{thm:2}  (ii),  with $\alpha^t \sim O(1/\sqrt{t})$,  the  optimality gap becomes  $$
\EE[\Lcal(\theta^{t^*}) ] -\inf_{\theta\in \RR^d}\max( \Lcal( \theta),\bepsilon)\allowbreak  = \tilde O(1/\sqrt{t}).
$$ 

\subsection{Experiments} \label{sec:9}

This section presents empirical evidence to support our theory and what is predicted by a well-known hypothesis. We note that  there is no related work or algorithm that can guarantee global convergence in the setting of our experiments where the model has   convolutions, skip connections, and batch normalizations without any wide layer (of the width larger than $n$). Moreover, unlike any previous studies that propose new methods, our training framework works by modifying any given method.

\subsubsection{Sine Wave Dataset}
\begin{figure}[!b] 
    \centering
    \begin{minipage}{.5\textwidth}
        \centering
        \includegraphics[width=1.0\textwidth, height=0.6\textwidth]{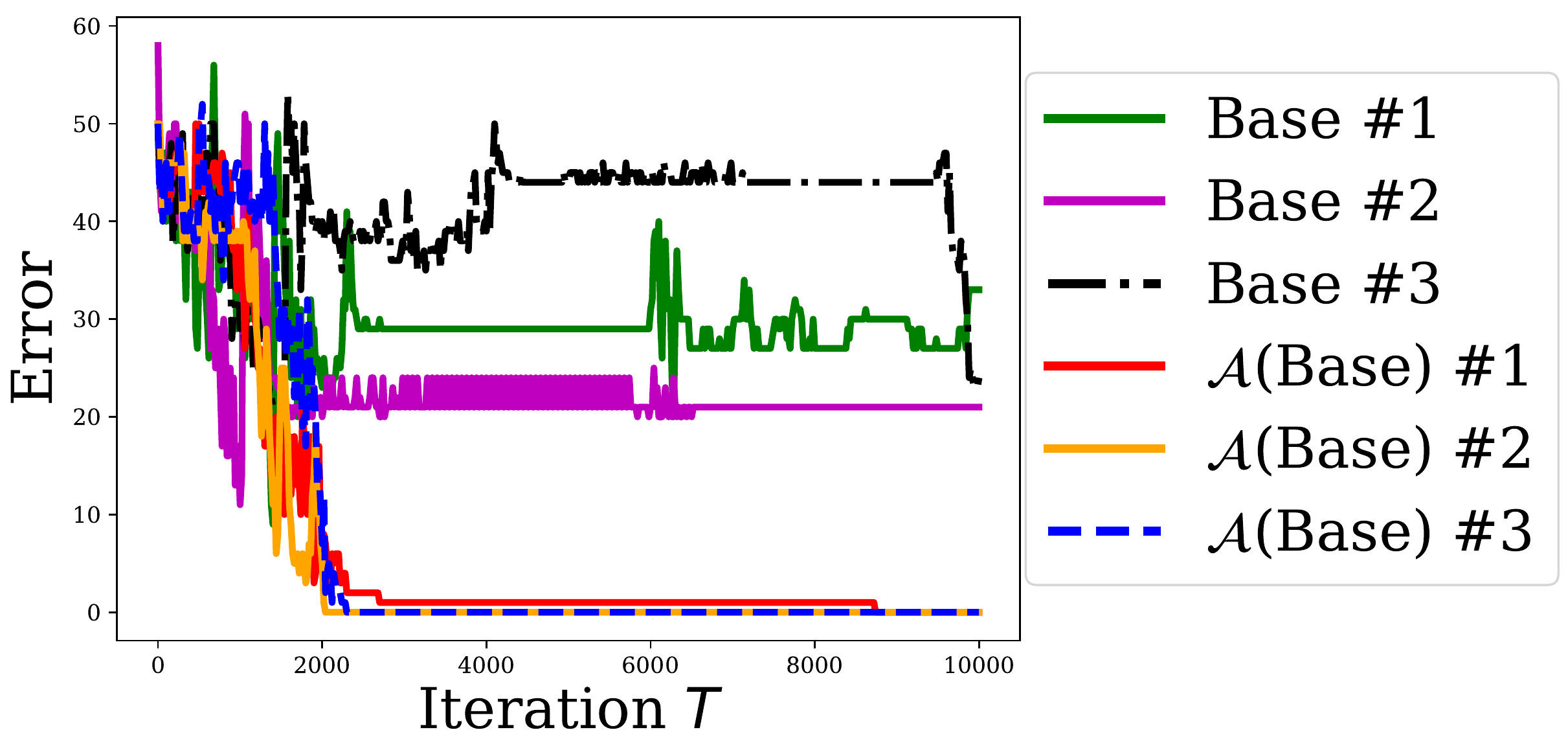}
        \caption{Training errors for three random trials.} 
        \label{fig:3}
    \end{minipage}%
    \hspace{10pt}
    \begin{minipage}{0.42\textwidth}
        \centering
        \includegraphics[width=0.9\textwidth, height=0.6\textwidth]{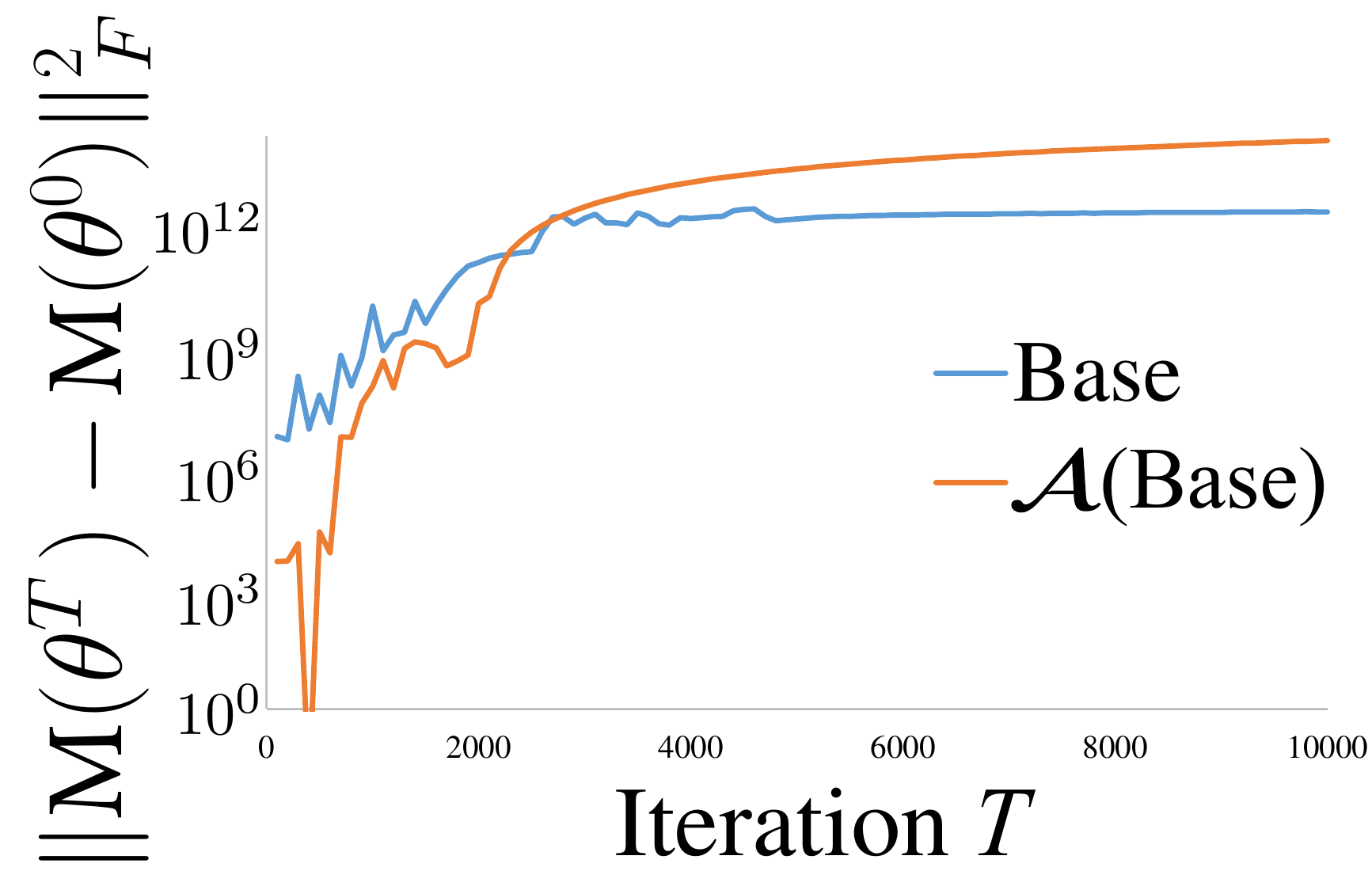}
        \caption{Changes of the representations  where $\mathbf{M}(\theta):=\frac{\partial \vect(f_{X}(\theta))}{\partial \theta} (\frac{\partial \vect(f_{X}(\theta))}{\partial\theta})\T$.} 
        \label{fig:4}
    \end{minipage}
\end{figure}

We have seen in Subsection \ref{sec:3} that gradient descent gets stuck at sub-optimal points for the sine wave dataset. Using the  same setting as that in Subsection \ref{sec:3} with $\varepsilon=0.01, \tau=2000$, and $\tGcal=\Gcal$, we confirm in Figure \ref{fig:3} that the EE wrapper $\Acal$ can modify gradient descent to avoid sub-optimal points and converge to global minima as predicted by our theory. Figure \ref{fig:4} shows the value of the change of the gradient representation, $\|\mathbf{M}(\theta^{T})-\mathbf{M}(\theta^{0})\|_{F}^2$, for each time step $T$. As it can be seen, the values of  $\|\mathbf{M}(\theta^{T})-\mathbf{M}(\theta^{0})\|_{F}^2$ are large for both methods. Notably, the\ EE wrapper $\Acal$ of the base case significantly increases the value of  $\|\mathbf{M}(\theta^{T})-\mathbf{M}(\theta^{0})\|_{F}^2$ even in the exploitation phase after $\tau=2000$ as we are optimizing the hidden layer.  See Appendix \ref{sec:app:2} in the Supplementary Information for more details of the experiments for the sine wave dataset.     

\subsubsection{Image datasets}
The standard convolutional ResNet with $18$ layers \citep{he2016identity} is used as the base model $\baf$.
We use ResNet-18 for the illustration of our theory because it is used in practice and it has convolution, skip connections, and batch normalization without any width larger  than the number of  data points. This setting  is not covered by any of the previous theories for global convergence.   We set the  activation  to be the softplus function $q\mapsto\ln(1+\exp(\varsigma q))/\varsigma$ with $\varsigma=100$ for all layers of the base ResNet. This  approximates the ReLU activation well, as shown in Appendix \ref{sec:app:3} in the Supplementary Information. We employ the cross-entropy loss and        $\tsigma(q)= \frac{1}{1+e^{-q}}$. We use a standard algorithm, SGD, with its standard hyper-parameter setting for the training algorithm $\Gcal$  with $\tGcal_{}=\Gcal$: i.e.,  we let the mini-batch size be 64, the weight decay rate be $10^{-5}$, the momentum coefficient be $0.9$, the learning rate  be   $\alpha^{t}=0.1$, the last epoch $\hat T$ be 200 (with data augmentation) and 100  (without data augmentation). The hyper-parameters $\varepsilon$ and $\tau=\tau_0 \hat T$ were selected from $\varepsilon\in \{10^{-3},10^{-5}\}$  and $\tau_0 \in \{0.4,0.6,0.8\}$ by only using  training data. That is,  we randomly divided  each training data (100\%)\ into a smaller training data (80\%)\ and a validation data (20\%) for a grid search over the hyper-parameters.
 See Appendix \ref{sec:app:2} in the Supplementary Information for the results of the grid search and  details of the experimental setting. This standard setting satisfies Assumptions \ref{assump:3}--\ref{assump:2}, leaving Assumption \ref{assump:1} to be verified.

\begin{table}[b!] 
\centering 
\captionof{table}{  Verification of the safe-exploration condition (Assumption \ref{assump:1})  with $m_H = \ceil{2( n/m_{H-1})}$ where $n$ is the number of training data,  $m_{H}$ is the width of the last hidden layer, and $m_{H-1}$ is the   width of the penultimate hidden layer. }  \label{tbl:1}
\begin{tabular}{lccccc}
\toprule
Dataset &  $n$ & $m_{H-1}$ & $m_{H}$ & Assumption \ref{assump:1}  \\
\midrule
MNIST    &  60000  & 513 & 234 &Verified   \\
\midrule
CIFAR-10  &  50000& 513 & 195 &  Verified   \\
\midrule
CIFAR-100  &  50000 & 513 & 195 & Verified \\
\midrule
Semeion    &  1000  & 513 & 4 & Verified   \\
\midrule
KMNIST     & 60000 & 513 & 234 &Verified   \\
\midrule
SVHN      &  73257 & 513 & 286 &Verified   \\
\bottomrule
\end{tabular} 
\end{table}

\begin{table}[b!] 
\begin{minipage}{.5\textwidth}
\centering 

\caption{\fontsize{11.pt}{11.pt}\selectfont Test error (\%): data augmentation. } \label{tbl:2} 
\begin{tabular}{lcc}
\toprule
 Dataset &  Standard & $\Acal(\text{Standard})$   \\

\midrule
 MNIST     & 0.40 (0.05)   & 0.30 (0.05)  \\
\midrule
CIFAR-10   & 7.80 (0.50) & 7.14 (0.12)  \\
\midrule
CIFAR-100   & 32.26 (0.15)  & 28.38 (0.42) \\
\midrule
Semeion       & 2.59 (0.57) & 2.56 (0.55) \\
\midrule
KMNIST       & 1.48 (0.07)  & 1.36 (0.11)   \\
\midrule
SVHN       & 4.67 (0.05)  & 4.43 (0.11)   \\
\bottomrule
\end{tabular}  
\end{minipage}%
\begin{minipage}{.5\textwidth}
\centering 
\renewcommand{\arraystretch}{0.8} \fontsize{11.pt}{11.pt}\selectfont
\caption{\fontsize{11.pt}{11.pt}\selectfont Test error (\%): no data augmentation. } \label{tbl:4} 
\begin{tabular}{lcc}
\toprule
Dataset &  Standard & $\Acal(\text{Standard})$   \\
\midrule
 MNIST    & 0.52 (0.16)   & 0.49 (0.02) \\
\midrule
 CIFAR-10   & 15.15 (0.87) & 14.56 (0.38)   \\
\midrule
 CIFAR-100   & 54.99 (2.29)  & 46.13 (1.80) \\
\bottomrule
\end{tabular} 
\end{minipage}
\end{table}

\paragraph{Verification of  Assumption \ref{assump:1}.} Table \ref{tbl:1} summarizes the verification results of the safe-exploration condition. Because the  condition only requires an existence of a pair $(\theta, q)$ satisfying the  condition, we  verified it by using a randomly sampled  $q$ from the standard normal distribution and  a $\theta$ returned by a common initialization scheme    \cite{he2015delving}. As $m_{H-1}=513$ (512 + the constant neuron for the bias term) for  the standard ResNet, we set $m_H = \ceil{2( n/m_{H-1})}$ throughout all the experiments with the ResNet. For each dataset,  the rank condition was verified twice by the two standard methods: one  from \citep{press2007numerical} and another from \citep{golub1996matrix}.

\paragraph{Test Performance.}
\rev{One well-known hypothesis is that the success of deep-learning methods partially comes from its ability to automatically learn   deep nonlinear representations suitable for making accurate predictions from data (e.g., \citealp{lecun2015deep}). As the EE wraper $\Acal$   keeps this  ability of representation learning, the hypothesis suggests that the  test performance of the EE wrapper $\Acal$ of a  standard method is approximately comparable  with that of the   standard method.} Unlike typical experimental studies, our objective here is to confirm this prediction, instead of showing  improvements over a previous method.  We empirically confirmed the prediction in Tables \ref{tbl:2} and \ref{tbl:4} where the numbers indicate the mean test errors (and standard deviations are in parentheses) over five random trials. As expected, the values of $\|\mathbf{M}(\theta^{\hat T})-\mathbf{M}(\theta^{0})\|_{2}^2$ were also large: e.g.,   $4.64 \times 10^{12}$ for the standard method and $3.43 \times 10^{12}$ for the wrapper  $\Acal$ of the  method with Semeion dataset.

\paragraph{Training Behavior.}
Figure \ref{fig:2} shows that the EE wrapper $\Acal$  can  improve training loss values of the standard SGD algorithm in the exploitation phase without changing its hyper-parameters because $\tGcal=\Gcal$ in these experiments. 
 In the figure, the plotted lines indicate the mean values over five random trials and the shaded regions show error bars with one standard deviation.

\paragraph{Computational Time.} The EE wrapper $\Acal$     runs the standard SGD $\Gcal$  in the exploration phase and  the  SGD $\tGcal=\Gcal$ only on the subset of the weights $\theta_{(H-1)}$  in the exploitation phase. Thus, the computational time of the EE wrapper $\Acal$  is similar to that of the  SGD in  the exploration phase, and it tends to be faster than the   SGD in the exploitation phase. To confirm this, we measure computational time with Semeion and CIFAR-10 datasets under the same computational resources  (e.g., without running other jobs in parallel) in a local workstation for each method. The mean wall-clock time    (in seconds) over five random trials is summarised in Table \ref{tbl:6}, where  the  numbers in parentheses are standard deviations. It shows that the EE wrapper $\Acal$ is slightly faster than the standard method, as expected.

\begin{figure}[t!]
\center
\begin{subfigure}[b]{0.32\columnwidth}
  \includegraphics[width=\textwidth, height=0.7\textwidth]{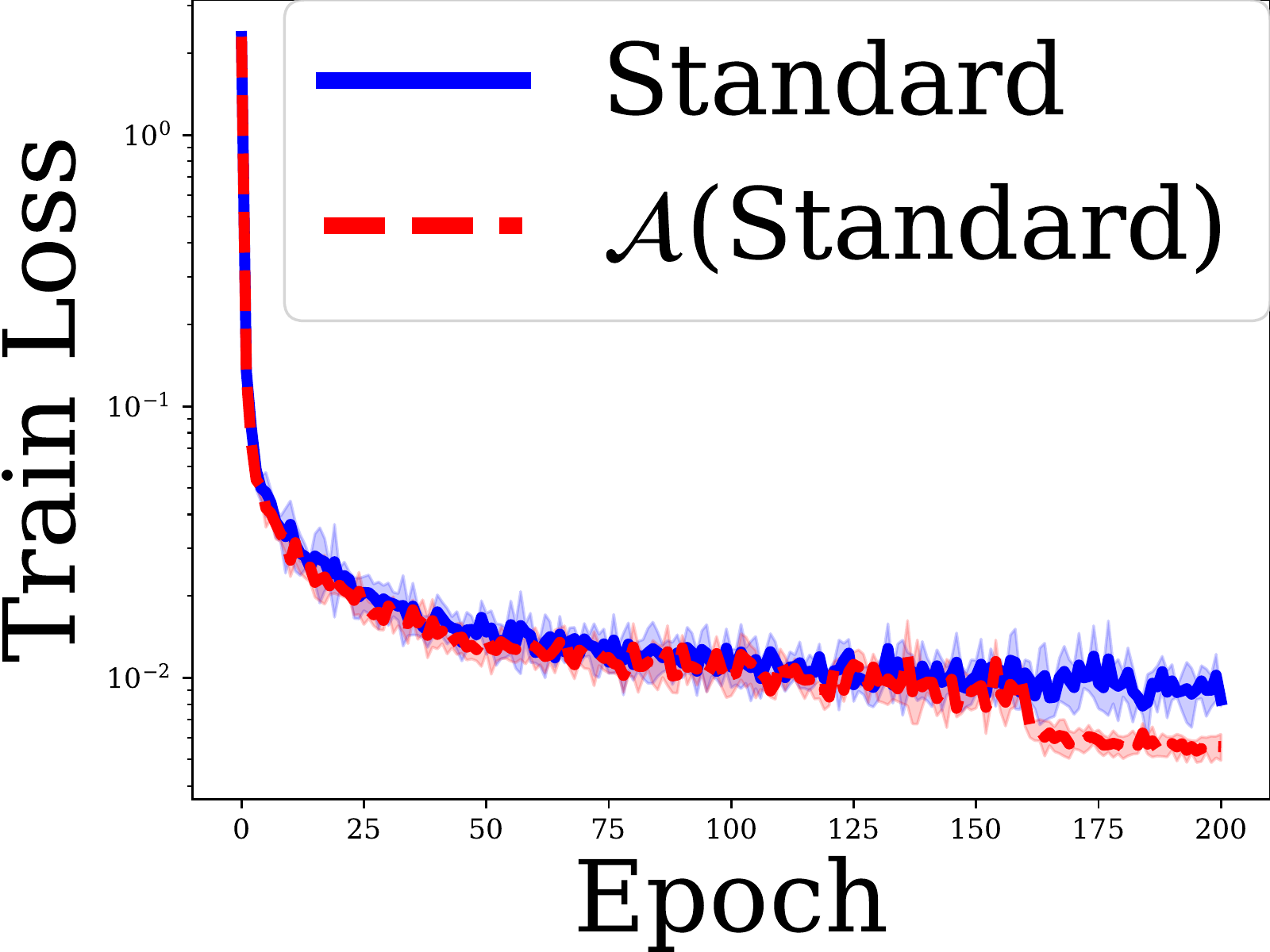}
  \caption{MNIST}
\end{subfigure}
\begin{subfigure}[b]{0.32\columnwidth}
  \includegraphics[width=\textwidth, height=0.7\textwidth]{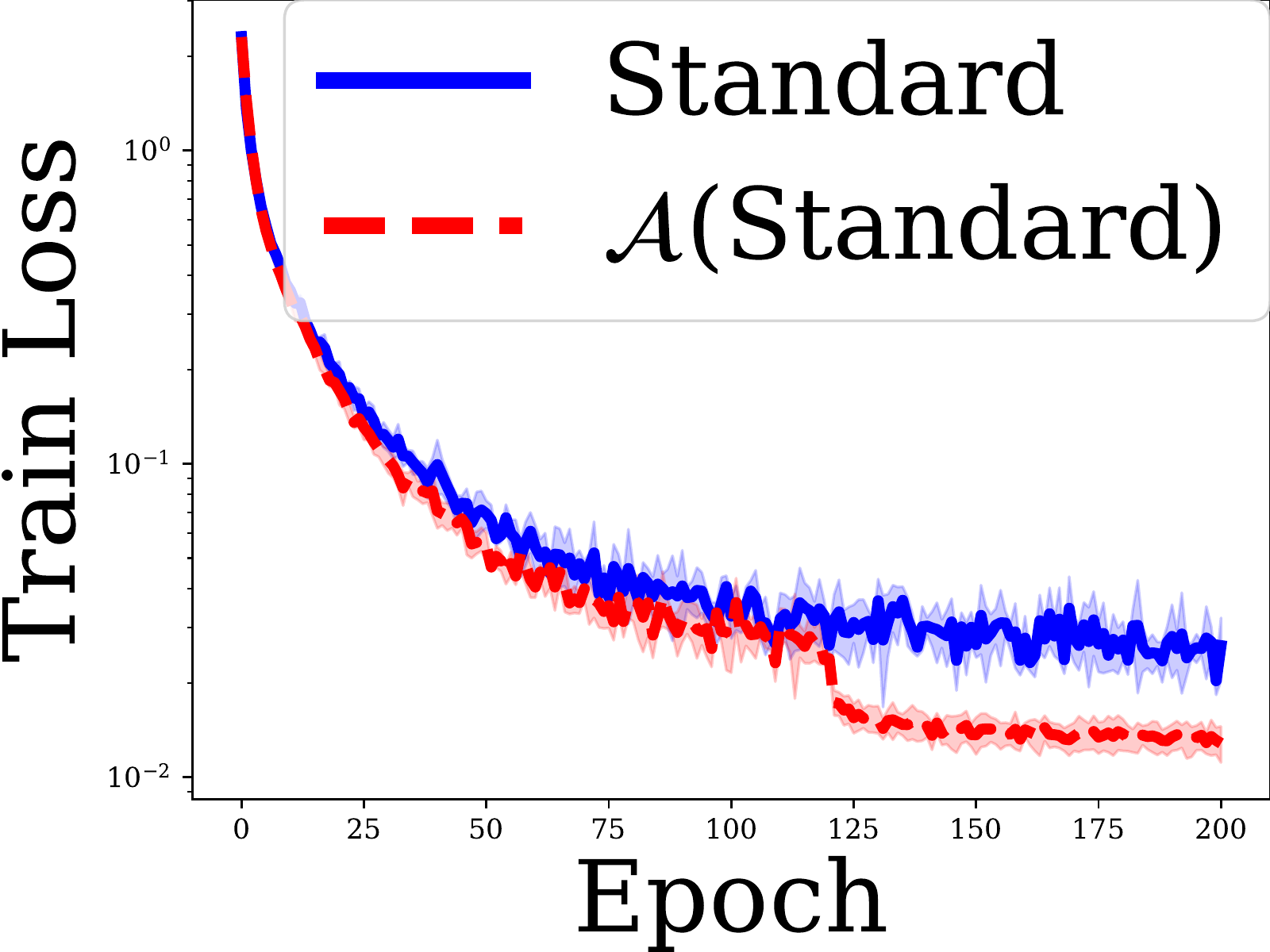}
  \caption{CIFAR-10} 
\end{subfigure}
\begin{subfigure}[b]{0.32\columnwidth}
  \includegraphics[width=\textwidth, height=0.7\textwidth]{fig/train/cifar10.pdf}
  \caption{CIFAR-100} 
\end{subfigure}
\caption{Training loss with data augmentation   } 
\label{fig:2}
\end{figure}

\begin{table}[t!]
\centering 
\caption{Total wall-clock time in a local GPU\ workstation } \label{tbl:6}
\begin{tabular}{lcc}
\toprule
Dataset &  Standard & $\Acal(\text{Standard})$   \\
\midrule
 Semeion    & 364.60 (0.94)   & 356.82 (0.67)\\
\midrule
 CIFAR-10   & 3616.92 (10.57) & 3604.5 (6.80)   \\
\bottomrule
\end{tabular} 
\end{table} 

\begin{table}[t!]
\centering 
\renewcommand{\arraystretch}{0.8} \fontsize{11.pt}{11.pt}\selectfont
\caption{Test errors (\%) of   $\Acal(\text{Standard})$ with $\tGcal$ = L-BFGS.} \label{tbl:5} \begin{subtable}[t]{0.38\columnwidth}
\caption{with data augmentation}
\centering 
\begin{tabular}{|l||*{4}{c|}}
\hline
\backslashbox{$\varepsilon$ \kern-1em}{$\tau_0$} 
&\makebox[1em]{0.4}
&\makebox[1em]{0.6}
&\makebox[1em]{0.8}
\\\hline\hline
${10^{-3}}^{\mathstrut}$ 
& 
0.26
& 
0.38 & 0.37
\\\hline
${10^{-5}}^{\mathstrut}$ 
&
0.37
& 0.32 & 0.37 \\\hline
\end{tabular}
\end{subtable} \hspace{5pt}
\begin{subtable}[t]{0.38\columnwidth}
\caption{without data augmentation}
\centering 
\begin{tabular}{|l||*{4}{c|}}
\hline
\backslashbox{$\varepsilon$ \kern-1em}{$\tau_0$}
&\makebox[1em]{0.4}
&\makebox[1em]{0.6}
&\makebox[1em]{0.8}
\\\hline\hline
${10^{-3}}^{\mathstrut}$  & 
0.36
& 
0.43 & 0.42
\\\hline
${10^{-5}}^{\mathstrut}$  & 
0.42
& 0.35 & 0.35 \\\hline
\end{tabular}
\end{subtable} 
\end{table}

\paragraph{Effect of  Learning Rate and Optimizer.}
We  also  conducted  experiments on the effects of learning rates and optimizers using MNIST dataset with data-augmentation. Using the best learning rate from $\{0.2, 0.1, 0.01, 0.001\}$ for each method  (with $\tGcal=\Gcal$ = SGD), the mean test errors  (\%)  over five random trials were     0.33 (0.03)  for the standard base method, and 0.27 (0.03) for the $\Acal$ wrapper of  the standard base method (the  numbers in parentheses are standard deviations). Moreover, Table \ref{tbl:5} reports the preliminary results on the effect of optimizers with    $\tGcal$ being set to Limited-memory Broyden--Fletcher--Goldfarb--Shanno algorithm (L-BFGS) (with $\Gcal$ = the standard SGD). By comparing Tables \ref{tbl:2} and  \ref{tbl:5}, we can see that using a different optimizer in the exploitation phase can potentially lead to   performance improvements. A  comprehensive study of this phenomena is left to future work.

\section{Conclusion} \label{sec:16}

Despite the nonlinearity of the dynamics and the non-invexity of the objective,  we have rigorously proved convergence of training dynamics to global minima for nonlinear representation learning. Our results apply to a wide range of machine learning models, allowing both under-parameterization and over-parameterization. For example, our results are applicable to the case where the minimum eigenvalue of the matrix $\frac{\partial \vect(f_{X}(\theta^{t}))}{\partial \theta^{t}} (\frac{\partial \vect(f_{X}(\theta^{t}))}{\partial\theta^{t}})\T $ is zero for all $t \ge 0$. Under the common model structure assumption, models that cannot achieve zero error for all datasets (except some  `good' datasets)  are shown to achieve  global optimality with zero error  exactly when the  dynamics satisfy the data-architecture alignment condition. Our results provide  guidance for choosing and designing model structure and algorithms via the common model structure assumption and data-architecture alignment condition.

The key limitation in our analysis is the differentiability of the function $f $. For multilayer neural networks,  this  is satisfied by  using standard activation functions, such as softplus, sigmoid, and hyperbolic tangents. Whereas   softplus  can approximate  ReLU arbitrarily well,  the direct treatment of  ReLU  in nonlinear representation learning is left to future work.

Our theoretical results  and numerical observations  uncover  novel  mathematical properties and provide a basis for  future work. For example, we have  shown global convergence under the data-architecture alignment condition     $\vect(Y_{\ell}) \in \Col(\frac{\partial \vect(f_{X}(\theta^{t}))}{\partial \theta^{t}})$. The EE wrapper $\Acal$ is only one way to ensure this condition. There are many other ways to ensure the data-architecture alignment condition  and each  way can  result in a new algorithm with  guarantees.


\newcommand{\noopsort}[1]{} \newcommand{\printfirst}[2]{#1}
  \newcommand{\singleletter}[1]{#1} \newcommand{\switchargs}[2]{#2#1}


\clearpage
\appendices

\begin{center}
\textbf{\LARGE 
 Supplementary Material\vspace{5pt}}
\end{center}
\allowdisplaybreaks

\section{Proofs} \label{sec:app:1}
In Appendix \ref{sec:app:1}, we present the   proofs of all the theoretical results. We show the proofs of Theorem \ref{thm:3} in Appendix \ref{sec:app:1:1},  Theorem \ref{thm:p:1} in Appendix \ref{sec:app:1:2}, Proposition \ref{prop:1} in Appendix \ref{sec:app:1:3},  Theorem \ref{thm:1} in Appendix \ref{sec:app:1:4}, and Theorem \ref{thm:2} in Appendix \ref{sec:app:1:5}. We also provide a proof idea of Theorem \ref{thm:3} in the beginning of Appendix \ref{sec:app:1:1}. Before starting our  proofs, we  introduce  the additional notations used in the proofs. We define $f_{i}(\theta)= f(x_i,\theta)$,  $\ell_i(q)=\ell_{i}(q,y_{i})$, and $
[k', k] = \{k',k'+1,\dots, k\}$ for any $k,k' \in \NN^+$ with $k>k'$. 
Given a scalar-valued  variable  $a \in \RR$ and a matrix $M \in \RR^{m\times m'}$,
we define
\begin{align}
\frac{\partial a}{\partial M}= \begin{bmatrix}\frac{\partial a}{\partial M_{11}} & \cdots & \frac{\partial a}{\partial M_{1m'}} \\
\vdots & \ddots & \vdots \\
\frac{\partial a}{\partial M_{m1}} & \cdots & \frac{\partial a}{\partial M_{mm'}} \\
\end{bmatrix} \in \RR^{m \times m'},
\end{align}
where $M_{ij}$ represents the $(i,j)$-th entry of the matrix $M$.
Similarly, given a vector-valued variable $a \in \RR^m$ and a column vector $b \in \RR^{m'}$,
we let
\begin{align}
\frac{\partial a}{\partial b}= \begin{bmatrix}\frac{\partial a_{1}}{\partial b_{1}} & \cdots & \frac{\partial a_{1}}{\partial  b_{m'}} \\
\vdots & \ddots & \vdots \\
\frac{\partial a_{m}}{\partial  b_{1}} & \cdots & \frac{\partial a_{m}}{\partial   b_{m'}} \\
\end{bmatrix} \in \RR^{m \times m'},
\end{align}
where $b_{i}$ represents   the $i$-th entry of the column vector  $b$. Given a vector-valued variable $a \in \RR^m$ and a row vector $b \in \RR^{1 \times m'}$,
we write
\begin{align}
\frac{\partial a}{\partial b}= \begin{bmatrix}\frac{\partial a_{1}}{\partial b_{11}} & \cdots & \frac{\partial a_{1}}{\partial  b_{1m'}} \\
\vdots & \ddots & \vdots \\
\frac{\partial a_{m}}{\partial  b_{11}} & \cdots & \frac{\partial a_{m}}{\partial   b_{1m'}} \\
\end{bmatrix} \in \RR^{m \times m'},
\end{align}
where $b_{1i}$ represents   the $i$-th entry of the row vector  $b$.
Given a   function $a :v\mapsto a(v)$,  we define 
$$
\partial a(\bar v)=\left.\frac{\partial a(v)}{\partial v} \right\vert_{v=\bar v}
$$
For completeness, we also  recall the  standard definition of the Kronecker  product of two matrices: for matrices $M \in \RR^{m_M\times m_M'}$ and $\bar M \in \RR^{m_{\bar M} \times m'_{\bar M}}$, 
\begin{align}
M \otimes \bar M= \begin{bmatrix}M_{11}\bar M  & \cdots & M_{1m_M'}\bar M \\
\vdots & \ddots\ & \vdots  \\
M_{m_M1}\bar M & \cdots & M_{m_Mm_M'}\bar M \\
\end{bmatrix} \in \RR^{m_M m_{\bar M} \times m_M' m_{\bar M}'}.
\end{align}

\subsection{Proof of Theorem \ref{thm:new:1}} 
The formal version of Theorem \ref{thm:new:1} is the combination of Theorem \ref{thm:3}, Theorem \ref{thm:p:1}, and Proposition \ref{prop:1}, which are proved in Appendix \ref{sec:app:1:1},   Appendix \ref{sec:app:1:2}, and  Appendix \ref{sec:app:1:3}.

\subsection{Proof of Theorem \ref{thm:3}} \label{sec:app:1:1}
 We begin  with  a proof idea of Theorem \ref{thm:3}.  We first relate the value of $\Lcal(\theta)$ and that of $\|\nabla\Lcal(\theta)\|$ in the way that if $\|\nabla\Lcal(\theta)\| = 0$, then the value of $\Lcal(\theta)$ is  bounded from above by $\Lcal_\theta(\beta) $ for any $\theta, \beta \in \RR^d$. This is a key lemma --- Lemma \ref{lemma:pgd_approx}. This lemma implies that since  every limit point achieves $\|\nabla\Lcal(\theta)\| = 0$, the value of  $\Lcal(\theta)$ has the upper bound of $\Lcal_\theta(\beta)$  at every limit point  $\theta \in \RR^d$. Then, another key observation is that the  upper bound $\Lcal_\theta(\beta)$ on  $\Lcal(\theta)$ can be replaced by $\Lcal^*(\eta Y^{*})$ for any $\eta \in \RR$ if $\vect( Y^{*}) \in \Col(\frac{\partial \vect(f_{X}(\theta^{t}))}{\partial \theta^{t}})$. Here, the condition of $\vect( Y^{*}) \in \Col(\frac{\partial \vect(f_{X}(\theta^{t}))}{\partial \theta^{t}})$ is crucial and used  to ensure the quality of the upper bound $\Lcal_\theta(\beta)$. 

To prove Theorem \ref{thm:3}, we first prove the following key lemma that relates the value of $\Lcal(\theta)$ and that of $\|\nabla\Lcal(\theta)\|$:

\begin{lemma} \label{lemma:pgd_approx}
Assume that  the function $\ell_i:q\mapsto\ell(q,y_{i}) \in \RR_{\ge 0}$ is differentiable and convex for every $i \in \{1,\dots,n\}$. Let $\theta$  be any differentiable point such that  $f_{i}:\theta \mapsto f(x_i, \theta)$ is differentiable at $\theta$  for every $i \in \{1,\dots,n\}$. Then  for any  $\beta\in \RR^d$,   
\begin{align}
\Lcal(\theta)  \le\Lcal_\theta(\beta) +\|\nu(\theta)-\beta \|_2 \|\nabla\Lcal(\theta)\|_2.
\end{align}
where\begin{align}
\Lcal_\theta(\beta)=\frac{1}{n}\sum_{i=1}^n  \ell\left(f_{\theta}(x_i,
\beta), y_{i}\right),
\end{align}
\begin{align}
f_{\theta}(x,\beta)=\sum_{k=1}^{d} 
\beta_k \frac{\partial f(x,\theta)}{\partial \theta_k},
\end{align}
and  $\nu(\theta)_k=\theta_k$ for all $k\in S$ in Assumption \ref{assump:7} and  $\nu(\theta)_k=0$ for all  $k \notin S$.

\end{lemma} 
\begin{proof}[Proof of Lemma \ref{lemma:pgd_approx}]
Let $\theta$  be any differentiable point such that  $f_{i}:\theta \mapsto f(x_i, \theta)$ is differentiable at $\theta$  for every $i \in \{1,\dots,n\}$. Then, for any $\beta\in \RR^d$, 
\begin{align}
 \Lcal_\theta(\beta) &= \frac{1}{n} \sum_{i=1}^n \ell_{i}(f_{\theta}(x_{i},\beta))
 \\  & = \frac{1}{n} \sum_{i=1}^n \ell_{i}(f_{i}(\theta^{})+f_{\theta}(x_{i},\beta)-f_{i}(\theta^{}))
\\ & \ge \frac{1}{n} \sum_{i=1}^n \left[ \ell_{i}(f_{i}(\theta^{}))+ \partial\ell_{i}(f_{i}(\theta^{}))(f_{\theta}(x_{i},\beta)-f_{i}(\theta^{})) \right]
\\ & =\Lcal(\theta) +\frac{1}{n} \sum_{i=1} ^n\partial\ell_{i}(f_{i}(\theta^{}))(f_{\theta}(x_{i},\beta)-f_{i}(\theta^{}))  
\end{align}
where the third line follows from the differentiability and and convexity of   $\ell_{i}$. From Assumption \ref{assump:7}, we have that
\begin{align}
f_{i}(\theta) =\sum_{k=1}^{d} 
\nu(\theta)_k \frac{\partial f(x_{i},\theta)}{\partial \theta_k}. 
\end{align}
 Combining these, \begin{align}
\Lcal_\theta(\beta) & \ge\Lcal(\theta) +\frac{1}{n} \sum_{i=1} ^n\partial\ell_{i}(f_{i}(\theta^{}))(f_{\theta}(x_{i},\beta)-f_{i}(\theta^{}))
\\ & =\Lcal(\theta) +\frac{1}{n} \sum_{i=1} ^n\partial\ell_{i}(f_{i}(\theta))\left(\sum_{k=1}^{d} 
\beta_k \frac{\partial f(x_{i},\theta)}{\partial \theta_k}-\sum_{k=1}^{d} 
\nu(\theta)_k \frac{\partial f(x_{i},\theta)}{\partial \theta_k}\right)
\\ & =\Lcal(\theta) +\frac{1}{n} \sum_{i=1} ^n\partial\ell_{i}(f_{i}(\theta))\sum_{k=1}^{d} 
(\beta_k -\nu(\theta)_k )\frac{\partial f(x_{i},\theta)}{\partial \theta_k}  \\ & =\Lcal(\theta) +\sum_{k=1}^{d} 
(\beta_k -\nu(\theta)_k ) \left(\frac{1}{n}  \sum_{i=1} ^n\partial\ell_{i}(f_{i}(\theta))\frac{\partial f(x_{i},\theta)}{\partial \theta_k} \right)   
\end{align}  
By noticing that $\nabla_{\theta_k}\Lcal(\theta)=\frac{1}{n}  \sum_{i=1} ^n\partial\ell_{i}(f_{i}(\theta))\frac{\partial f(x_{i},\theta)}{\partial \theta_k}$, this implies that
\begin{align}
\Lcal(\theta)   &\le \Lcal_\theta(\beta) -\sum_{k=1}^{d} 
(\beta_k -\nu(\theta)_k )\nabla_{\theta_k}\Lcal(\theta)
\\ & = \Lcal_\theta(\beta)+(\nu(\theta)-\beta  )\T\nabla_{}\Lcal(\theta)
\\ & \le \Lcal_\theta(\beta) +\|\nu(\theta)-\beta \|_2 \|\nabla\Lcal(\theta)\|_2.
\end{align}
\end{proof} 

\

We also utilize the following lemma, which is  a slightly modified version of a well-known fact:

\begin{lemma} \label{lemma:known_1}
For any differentiable function $\varphi: \dom(\varphi) \rightarrow \RR$ with an open  convex domain $\dom(\varphi) \subseteq \RR^{n_\varphi}$, if   $\|\nabla \varphi(z') - \nabla \varphi(z)\| \le L_{ \varphi} \|z'-z\|$ for all $z,z' \in \dom(\varphi)$, then
\begin{align}
\varphi(z') \le \varphi(z) + \nabla \varphi(z)\T (z'-z) + \frac{L_{ \varphi}}{2} \|z'-z\|^2 \quad   \text{for all $z,z' \in \dom(\varphi) $}.
\end{align}
\end{lemma}
\begin{proof}[Proof of Lemma \ref{lemma:known_1}]
Fix $z,z'\in \dom(\varphi) \subseteq \RR^{d_{\varphi}}$. Since $\dom(\varphi)$ is a convex set,  $z+t(z'-z)\in\dom(\varphi)$ for all $t \in [0, 1]$.
Since $\dom(\varphi) $ is open, there exists $\epsilon >0$ such that $z+(1+\epsilon')(z'-z)\in\dom(\varphi)$ and $z+(0-\epsilon')(z'-z)\in\dom(\varphi)$ for all $\epsilon'\le \epsilon$. Fix $\epsilon>0$ to be such a number. Combining these, $z+t(z'-z)\in\dom(\varphi)$ for all $t \in [0-\epsilon, 1+\epsilon]$.

Accordingly, we can define  a function $\bar \varphi: [0-\epsilon, 1+\epsilon] \rightarrow \RR$ by $\bar \varphi(t)=\varphi(z+t(z'-z))$. Then, $\bar \varphi(1)=\varphi(z')$, $\bar \varphi(0)=\varphi(z)$, and $\nabla \bar \varphi(t)=\nabla\varphi(z+t(z'-z))\T (z'-z)$ for $t \in [0, 1] \subset (0-\epsilon,1+\epsilon)$. Since $\|\nabla \varphi(z') - \nabla \varphi(z)\| \le L_{ \varphi} \|z'-z\|$, 
\begin{align}
\|\nabla\bar  \varphi(t')-\nabla\bar  \varphi(t) \| &=\|[\nabla\varphi(z+t'(z'-z)) -\nabla\varphi(z+t(z'-z))\T (z'-z) \|
\\ &\le \|z'-z\|\|\nabla\varphi(z+t'(z'-z)) -\nabla\varphi(z+t(z'-z))  \| \\ & \le  L_{ \varphi}\|z'-z\|\|(t'-t)(z'-z)   \|
\\ & \le L_{ \varphi}\|z'-z\|^{2}\|t'-t   \|.
\end{align}
Thus, $\nabla \bar \varphi:[0, 1]\rightarrow \RR$ is Lipschitz continuous with the Lipschitz constant $L_{ \varphi}\|z'-z\|^{2}$, and hence $\nabla\bar  \varphi$ is continuous. 

By using the fundamental theorem of calculus with the continuous function $\nabla\bar  \varphi:[0, 1]  \rightarrow \RR$,
\begin{align}
\varphi(z')&=\varphi(z)+ \int_0^1 \nabla\varphi(z+t(z'-z))\T (z'-z)dt
\\ &=\varphi(z)+\nabla\varphi(z)\T (z'-z)+ \int_0^1 [\nabla\varphi(z+t(z'-z))-\nabla\varphi(z)]\T (z'-z)dt
\\ & \le \varphi(z)+\nabla\varphi(z)\T (z'-z)+ \int_0^1 \|\nabla\varphi(z+t(z'-z))-\nabla\varphi(z)\| \|z'-z \|dt
\\ & \le \varphi(z)+\nabla\varphi(z)\T (z'-z)+ \int_0^1 t L_{ \varphi}\|z'-z\|^{2}dt
\\ & =  \varphi(z)+\nabla\varphi(z)\T (z'-z)+\frac{L_{ \varphi}}{2}\|z'-z\|^{2}. 
\end{align}
 
\end{proof}

 With Lemmas \ref{lemma:pgd_approx} and \ref{lemma:known_1},  we are now ready to complete the proof of Theorem \ref{thm:3}:  

\begin{proof}[Proof of Theorem \ref{thm:3}] 
 The function $\Lcal$ is differentiable since $\ell_{i}$ is differentiable, $\theta_{}\mapsto f(x_{i},\theta)$ is differentiable, and a composition of differentiable functions is differentiable. We will first show that in both cases of (i) and (ii) for the learning rates, we have $\lim_{t \rightarrow \infty}\nabla\Lcal(\theta_{}^{t}) = 0$. If $\nabla\Lcal(\theta_{}^{t})=0$ at any $t\ge 0$, then
Assumption \ref{assump:5} ($\|\bg^t\|_2^2 \le \bc \|\nabla\Lcal(\theta_{}^{t})\|_2^2$) implies  $\bg^t=0$, which implies 
$$
\theta_{}^{t+1}= \theta^{t}_{} \text{ and } \nabla\Lcal(\theta_{}^{t+1})=\nabla\Lcal(\theta_{}^{t})=0.
$$
This means that  if $\nabla\Lcal(\theta_{}^{t})=0$ at any $t\ge0$, we have that  $\bg^t=0$ and $\nabla\Lcal(\theta_{}^{t})=0$ for all $t \ge 0$ and hence
\begin{align} 
\lim_{t \rightarrow \infty}\nabla\Lcal(\theta_{}^{t}) = 0,
\end{align}
as desired. Therefore, we now focus on the remaining scenario where $\nabla\Lcal(\theta_{}^{t})\neq 0$ for all $t \ge 0$. 

By using Lemma \ref{lemma:known_1}, 
$$
\Lcal(\theta_{}^{t+1})\le \Lcal(\theta_{}^{t})-\alpha^t  \nabla\Lcal(\theta_{}^{t})  \T\bg^t  + \frac{L(\alpha^t)^2  }{2} \|\bg^t \|^2.
$$
By rearranging and using Assumption \ref{assump:5},
\begin{align}
\Lcal(\theta_{}^{t})-\Lcal(\theta_{}^{t+1}) &\ge \alpha^t  \nabla\Lcal(\theta_{}^{t})  \T\bg^t  - \frac{L(\alpha^t)^2  }{2} \|\bg^t \|^2
\\ & \ge\alpha^t  \uc \|\nabla\Lcal(\theta_{}^{t})\|^2 -\frac{L(\alpha^t)^2  }{2}\bc \|\nabla\Lcal(\theta_{}^{t})\|^2.
\end{align}
By simplifying the right-hand-side, 
\begin{align} \label{eq:9_2}
 \Lcal(\theta_{}^{t})-\Lcal(\theta_{}^{t+1}) &\ge\alpha^t   \|\nabla\Lcal(\theta_{}^{t})\|^2 (\uc-\frac{L\alpha^t  }{2}\bc). 
\end{align}

Let us now focus on case (i). Then, using $ \alpha^t \le \frac{\uc\ (2-\epsilon)}{L\bc}$, 
$$
\frac{L\alpha^t}{2}\bc\le\frac{L\uc\ (2-\epsilon)}{2L\bc}\bc =\uc-\frac{\epsilon}{2}\uc.
$$
Using this inequality and using $\epsilon\le \alpha^t$ in equation \eqref{eq:9_2},
\begin{align} \label{eq:10_2}
\Lcal(\theta_{}^{t})-\Lcal(\theta_{}^{t+1}) &\ge\frac{\uc\epsilon^{2}}{2}
\|\nabla\Lcal(\theta_{}^{t})\|^2 .
\end{align}
Since $\nabla\Lcal(\theta_{}^{t})\neq 0$ for any $t\ge 0$ (see above) and $\epsilon >0 $, this means that the sequence $(\Lcal(\theta_{}^{t}))_{t}$ is monotonically decreasing. Since $\Lcal(q) \ge 0$ for any $q$ in its domain, this implies that the sequence $(\Lcal(\theta_{}^{t}))_{t}$ converges. Therefore, $\Lcal(\theta_{}^{t})-\Lcal(\theta_{}^{t+1}) \rightarrow 0$  as $t \rightarrow \infty$. Using equation \eqref{eq:10_2}, this implies that $$
\lim_{t \rightarrow \infty}\nabla\Lcal(\theta_{}^{t}) = 0,
$$
which proves the desired result for the case (i). 

We now focus on the case (ii). Then, we still have equation \eqref{eq:9_2}. Since $\lim_{t \rightarrow \infty}\alpha^t =0$ in equation \eqref{eq:9_2}, the first order term in $\alpha^t$ dominates after sufficiently large $t$: i.e., there exists $\bar t \ge 0$ such that for any $t\ge \bar t$, 
\begin{align} \label{eq:11_2}
\Lcal(\theta_{}^{t})-\Lcal(\theta_{}^{t+1}) \ge c \alpha^t   \|\nabla\Lcal(\theta_{}^{t})\|^2.
\end{align}
for some constant $c>0$. Since $\nabla\Lcal(\theta_{}^{t})\neq 0$ for any $t\ge 0$ (see above) and $c \alpha^t>0$, this means that the sequence $(\Lcal(\theta_{}^{t}))_{t}$ is monotonically decreasing.  Since $\Lcal(q) \ge 0$ for any $q$ in its domain, this implies that the sequence $(\Lcal(\theta_{}^{t}))_{t}$ converges  to a finite value. Thus, by adding Eq. \eqref{eq:11_2} both sides over all $t \ge \bar t$, 
        \begin{align} 
              \infty >  \Lcal(\theta_{}^{\bar t})-\lim_{t \rightarrow \infty}\Lcal(\theta_{}^{t}) \ge c  \sum _{t=\bar t}^\infty \alpha^{t}   \|\nabla\Lcal(\theta_{}^{t})\|^2.
        \end{align}        
Since $\sum _{t=0}^\infty \alpha^{t}  = \infty$, this implies that  $\liminf _{t\to \infty }\|\nabla\Lcal(\theta^{t})\|=0$. We now show that by contradiction, $ \limsup_{t\to \infty }\|\nabla\Lcal(\theta^{t})\|=0$. Suppose that $\limsup_{t\to \infty }\|\nabla\Lcal(\theta^{t})\| > 0$. Then, there exists $\delta>0$ such that $\limsup_{t\to \infty }\|\nabla\Lcal(\theta^{t})\|\ge \delta$. Since $\liminf _{t\to \infty }\|\nabla\Lcal(\theta^{t})\|=0$ and $\limsup_{t\to \infty }\|\nabla\Lcal(\theta^{t})\| \allowbreak \ge \delta$, let $(\rho_j)_{j}$ and $(\rho'_j)_j$ be  sequences of indexes such that $\rho_j<\rho'_j<\rho_{j+1}$, $\|\nabla\Lcal(\theta^{t})\|>\frac{\delta}{3}$ for $\rho_j \le t < \rho_j'$, and  $\|\nabla\Lcal(\theta^{t})\|\le \frac{\delta}{3}$ for $\rho_j '\le t < \rho_{j+1}$. Since $\sum _{t=\bar t}^\infty \alpha^{t}   \|\nabla\Lcal(\theta_{}^{t})\|^2< \infty$, let $\bar j$ be sufficiently large such that  $\sum _{t=\rho_{\bar j}}^\infty \alpha^{t}   \|\nabla\Lcal(\theta_{}^{t})\|^2<\frac{\delta^2}{9L\sqrt{\bc}}$. Then, for any $j \ge \bar j$ and any $\rho$ such that $\rho_j \le \rho \le \rho'_j -1$, we have that 
\begin{align}
\|\nabla\Lcal(\theta_{}^{\rho})\|-\|\nabla\Lcal(\theta_{}^{\rho_j'})\|&\le\|\nabla\Lcal(\theta_{}^{\rho_j'})-\nabla\Lcal(\theta_{}^{\rho})\|
\\ &=\left\|\sum_{t=\rho}^{\rho'_j-1}\nabla\Lcal(\theta_{}^{t+1})-\nabla\Lcal(\theta_{}^{t})\right\|
 \\ & \le \sum_{t=\rho}^{\rho'_j-1}\left\|\nabla\Lcal(\theta_{}^{t+1})-\nabla\Lcal(\theta_{}^{t})\right\|  \\ & \le L  \sum_{t=\rho}^{\rho'_j-1}\left\|\theta_{}^{t+1}-\theta_{}^{t}\right\|
  \\ & \le L\sqrt{\bc}  \sum_{t=\rho}^{\rho'_j-1} \alpha^{t}   \left\|\nabla\Lcal(\theta_{}^{t})\right\|
\end{align}
where the first and third lines use the triangle inequality (and symmetry), the forth line uses the assumption that     $\|\nabla\Lcal(\theta)-\nabla\Lcal(\theta')\|\le L \|\theta-\theta'\|$, and the last line follows the  definition of  $\theta_{}^{t+1}-\theta_{}^{t}=-\alpha^t \bg$ and the assumption of $\|\bg^t\|^2 \le \bc \|\nabla\Lcal(\theta_{}^{t})\|^2
$.  Then, by using the definition of the sequences of the indexes, 
\begin{align}
\|\nabla\Lcal(\theta_{}^{\rho})\|-\|\nabla\Lcal(\theta_{}^{\rho_j'})\|&\le \frac{3L\sqrt{\bc}}{\delta}  \sum_{t=\rho}^{\rho'_j-1} \alpha^{t}   \left\|\nabla\Lcal(\theta_{}^{t})\right\|^2 \le \frac{\delta}{3}.
\end{align}
 Here, since  $\|\nabla\Lcal(\theta_{}^{\rho_j'})\|\le \frac{\delta}{3}$, by rearranging the inequality, we have that for any $\rho\ge \rho_{\bar j}$, 
 \begin{align}
\|\nabla\Lcal(\theta_{}^{\rho})\|\le\frac{2\delta}{3}.
\end{align}
This contradicts the inequality of  $\limsup_{t\to \infty }\|\nabla\Lcal(\theta^{t})\|\ge \delta$. Thus, we have 
 \begin{align}
\limsup_{t\to \infty }\|\nabla\Lcal(\theta^{t})\| =  \liminf _{t\to \infty }\|\nabla\Lcal(\theta^{t})\|= 0.
\end{align}
This implies that 
        \begin{equation}
                \lim_{t \rightarrow \infty}\nabla\Lcal(\theta_{}^{t}) = 0,
        \end{equation}
which proves the desired result for the case (ii). 
Therefore, in both cases of (i) and (ii) for the learning rates, we have $\lim_{t \rightarrow \infty}\nabla\Lcal(\theta_{}^{t}) = 0$. 

We now use the fact that $\lim_{t \rightarrow \infty}\nabla\Lcal(\theta_{}^{t}) = 0$, in order to  prove the statement of this theorem. From Lemma \ref{lemma:pgd_approx} and the differentiability assumption, for any  $\theta \in \RR^d$, it holds that  for any $\beta \in \RR^d$,
$$
\Lcal(\theta)  \le\Lcal_\theta(\beta) +\|\nu(\theta)-\beta \|_2 \|\nabla\Lcal(\theta)\|_2,
$$
where
$$
\Lcal_\theta(\beta)=\frac{1}{n}\sum_{i=1}^n  \ell\left( 
 \sum_{k=1}^{d} 
\beta_k \frac{\partial f(x_{i},\theta)}{\partial \theta_k}, y_{i}\right).
$$
Since
\begin{align}
&\eta Y^* \in \left\{\sum_{k=1}^{d} 
\beta_k \frac{\partial f_{X}(\theta)}{\partial \theta_k}:\beta \in \RR^d \right\} \subseteq \RR^{n \times m_y}
\\  &\Longleftrightarrow\vect(\eta Y^*) \in \left\{\sum_{k=1}^{d} 
\beta_k \frac{\partial \vect(f_{X}(\theta))}{\partial \theta_k}:\beta \in \RR^d \right\}\subseteq \RR^{n m_y} 
\\ &\Longleftrightarrow\vect( Y^*) \in\left\{ \frac{\partial \vect(f_{X}(\theta))}{\partial \theta}\beta:\beta \in \RR^d \right\} \subseteq \RR^{n m_y}
\\ &\Longleftrightarrow\vect(Y^*) \in\Col\left(\frac{\partial \vect(f_{X}(\theta))}{\partial \theta}\right)\subseteq \RR^{n m_y},
\end{align}
 we have that if $\vect(Y^*) \in \Col(\frac{\partial \vect(f_{X}(\theta))}{\partial \theta})\subseteq \RR^{n m_y}$, then $\eta Y^* \in \{\sum_{k=1}^{d} 
\beta_k \frac{\partial f_{X}(\theta)}{\partial \theta_k}:\allowbreak\beta \in \RR^d \} \subseteq \RR^{n \times m_y}$ for any $\eta \in \RR$.
Therefore,  if $\vect(Y^*) \in \Col(\frac{\partial \vect(f_{X}(\theta))}{\partial \theta})$, by setting 
$$\hat \beta(\theta,\eta):=\eta \left(\left(\frac{\partial \vect(f_{X}(\theta))}{\partial \theta}\right)\T \frac{\partial \vect(f_{X}(\theta))}{\partial \theta}\right)^{\dagger}\left(\frac{\partial \vect(f_{X}(\theta))}{\partial \theta}\right)\T \vect(Y^*) ,
$$ 
we have  $\eta Y^*=\sum_{k=1}^{d} 
\hat \beta(\theta,\eta)_k \frac{\partial f_{X}(\theta)}{\partial \theta_k}$ and thus
$$
\Lcal_\theta(\hat \beta(\theta,\eta))=\frac{1}{n}\sum_{i=1}^n  \ell\left( 
 \eta y_i, y_{i}\right)=\Lcal^*(\eta Y^{*})  .
$$
Using this $\hat \beta(\theta,\eta)$ in  Lemma \ref{lemma:pgd_approx}, we have that for any  $\theta$ satisfying $\vect(Y^*) \in \Col(\allowbreak\frac{\partial \vect(f_{X}(\theta))}{\partial \theta})$,
\begin{align}
\Lcal(\theta)  &\le\Lcal_\theta(\hat \beta(\theta,\eta)) +\|\nu(\theta)-\hat \beta(\theta,\eta) \|_2 \|\nabla\Lcal(\theta)\|_2,
\\ & =\Lcal^*(\eta Y^{*}) + \|\nu(\theta)-\hat \beta(\theta,\eta) \|_2 \|\nabla\Lcal(\theta)\|_2.
\end{align}
Note that the statement of this theorem vacuously holds if there is not limit point of the sequence $(\theta^t)_{t}$. Thus,  let $\htheta \in \RR^d$ be a limit point of the sequence $(\theta^t)_{t}$. Then, 
$\lim_{t \rightarrow \infty}\nabla\Lcal(\theta_{}^{t}) = 0$ implies $\nabla \Lcal(\hat \theta) =0$  and that if  $\theta^{t}$ satisfies $\vect(Y^*) \in \Col(\allowbreak\frac{\partial \vect(f_{X}(\theta^{t}))}{\partial \theta^{t}})$,
\begin{align}
\Lcal(\hat \theta)  \le L^*_{\eta } +\|\nu(\hat \theta)-\hat \beta(\hat \theta,\eta) \|_2 \|\nabla\Lcal(\hat \theta)\|_2=\Lcal^*(\eta Y^{*}) ,   
\end{align}
since $\htheta \in \RR^d$ and $\nabla \Lcal(\hat \theta) =0$.

\end{proof} 

\subsection{Proof of Theorem \ref{thm:p:1}} \label{sec:app:1:2}

The proof of Theorem \ref{thm:p:1} builds upon the proof of  Theorem \ref{thm:3} in Section \ref{sec:app:1:1}. More concretely, we   rely on the key lemma, Lemma \ref{lemma:pgd_approx}, from  Section \ref{sec:app:1:1}. As Lemma \ref{lemma:pgd_approx}
is not about the training dynamics, we utilize the following lemma to connect Lemma \ref{lemma:pgd_approx} to  the training dynamics: \begin{lemma} \label{lemma:known+alpha_3}
Assume that $\|\nabla\Lcal(\theta)-\nabla\Lcal(\theta')\|\le L \|\theta-\theta'\|$ for all $\theta,\theta'$ in the domain of $\Lcal$ for some $L \ge 0$. Define the sequence $(\theta^{t})_{t=0}^\infty$ by $\theta^{t+1} = \theta^{t}- \frac{2\alpha}{L}\nabla \Lcal(\theta^{t})$ for any $t \ge 0$ with any initial parameter vector $\theta^{0}$ and any $\alpha \in(0,1) $. Then, for any  $\Tcal\subseteq \NN_0$,
$$
\min_{t \in \Tcal} \|\nabla \Lcal(\theta^{t})\|\le\frac{1}{\sqrt{|\Tcal|}}    \sqrt{\frac{L(\Lcal(\theta^{t_0})-\Lcal(\theta^{t_{|\Tcal|-1}+1}))}{2\alpha(1-\alpha)}}, 
$$   
where $\{t_0,t_1\dots,t_{|\Tcal|-1}\}=\Tcal$ with $t_0 < t_1 < \cdots < t_{|\Tcal|-1}$.
\end{lemma} 
\begin{proof}[Proof of Lemma \ref{lemma:known+alpha_3}]
Using Lemma \ref{lemma:known_1} and $\theta^{t+1} - \theta^{t}= - \frac{2\alpha}{L}\nabla \Lcal(\theta^{t})$,
\begin{align} \label{eq:18}
\nonumber \Lcal(\theta^{t+1}) &\le \Lcal(\theta^{t})  + \nabla \Lcal(\theta^{t})  \T (\theta^{t+1}-\theta^{t}) + \frac{L}{2} \|\theta^{t+1}-\theta^{t}\|^2
\\\nonumber & = \Lcal(\theta^{t})-\frac{2\alpha}{L} \|\nabla \Lcal(\theta^{t})\|^{2} +  \frac{2\alpha^{2}}{L}\|\nabla \Lcal(\theta^{t})\|^2    
\\ & = \Lcal(\theta^{t})-\frac{2\alpha}{L} (1-\alpha)\|\nabla \Lcal(\theta^{t})\|^{2},    
\end{align} which implies that the sequence $(\Lcal(\theta^{t}))_t$ is non-increasing and  that
$$
\|\nabla \Lcal(\theta^{t})\|^{2} \le\frac{L}{2\alpha(1-\alpha)}(\Lcal(\theta^{t})-\Lcal(\theta^{t+1})).
$$
By summing up both sides over the time,
with $\Tcal=\{t_0,t_1\dots,t_{|\Tcal|-1}\}$,
\begin{align}
 \sum_{t \in \Tcal} \|\nabla \Lcal(\theta^{t})\|^{2} &\le\frac{L}{2\alpha(1-\alpha)}   \sum_{t \in \Tcal}(\Lcal(\theta^{t})-\Lcal(\theta^{t+1}))
  \\ & = \frac{L}{2\alpha(1-\alpha)}   \sum_{k=0}^{|\Tcal|}(\Lcal(\theta^{t_k})-\Lcal(\theta^{t_{k}+1}))
  \\ & \le \frac{L}{2\alpha(1-\alpha)} (\Lcal(\theta^{t_0})-\Lcal(\theta^{t_{|\Tcal|-1}+1})),
\end{align}
where the last line follows from the fact that $\Lcal(\theta^{t_{k}+1})\ge \Lcal(\theta^{t_{k+1}})$ and thus $-\Lcal(\theta^{t_{k}+1})+\Lcal(\theta^{t_{k+1}}) \allowbreak \le 0$ for all $k$, since the sequence $(\Lcal(\theta^{t}))_t$ is non-increasing from  equation \eqref{eq:18}. Using this inequality, 
$$
|\Tcal| \left(\min_{ t \in \Tcal} \|\nabla \Lcal(\theta^{t})\|^{2} \right) \le \sum_{t \in \Tcal} \|\nabla \Lcal(\theta^{t})\|^{2} \le\frac{L(\Lcal(\theta^{t_0})-\Lcal(\theta^{t_{|\Tcal|-1}+1}))}{2\alpha(1-\alpha)}   .  
$$

Therefore,
$$
\min_{t \in \Tcal} \|\nabla \Lcal(\theta^{t})\|\le \sqrt{\frac{L(\Lcal(\theta^{t_0})-\Lcal(\theta^{t_{|\Tcal|-1}+1}))}{2\alpha(1-\alpha)|\Tcal|}} =\frac{1}{\sqrt{|\Tcal|}}    \sqrt{\frac{L(\Lcal(\theta^{t_0})-\Lcal(\theta^{t_{|\Tcal|-1}+1}))}{2\alpha(1-\alpha)}}.
$$

\end{proof}  

With Lemmas \ref{lemma:pgd_approx} and \ref{lemma:known+alpha_3},  we are now ready to complete the proof of Theorem \ref{thm:p:1}: 

\begin{proof}[Proof of Theorem \ref{thm:p:1}]
From Lemma \ref{lemma:pgd_approx} and the differentiability assumption, for any  $\theta \in \RR^d$, it holds that  for any $\beta \in \RR^d$,
$$
\Lcal(\theta)  \le\Lcal_\theta(\beta) +\|\nu(\theta)-\beta \|_2 \|\nabla\Lcal(\theta)\|_2,
$$
where
$$
\Lcal_\theta(\beta)=\frac{1}{n}\sum_{i=1}^n  \ell\left( 
 \sum_{k=1}^{d} 
\beta_k \frac{\partial f(x_{i},\theta)}{\partial \theta_k}, y_{i}\right).
$$
Furthermore,
\begin{align}
&\eta Y^{*} \in \left\{\sum_{k=1}^{d} 
\beta_k \frac{\partial f_{X}(\theta)}{\partial \theta_k}:\beta \in \RR^d \right\} \subseteq \RR^{n \times m_y} 
\\ &\Longleftrightarrow\vect(\eta Y^{*}) \in \left\{\sum_{k=1}^{d} 
\beta_k \frac{\partial \vect(f_{X}(\theta))}{\partial \theta_k}:\beta \in \RR^d \right\}\subseteq \RR^{n m_y} 
\\ &\Longleftrightarrow\vect( Y^{*}) \in\left\{ \frac{\partial \vect(f_{X}(\theta))}{\partial \theta}\beta:\beta \in \RR^d \right\} \subseteq \RR^{n m_y}
\\ &\Longleftrightarrow\vect(Y^{*}) \in\Col\left(\frac{\partial \vect(f_{X}(\theta))}{\partial \theta}\right)\subseteq \RR^{n m_y}.
\end{align}
 Thus, we have that if $\vect(Y^{*}) \in \Col(\frac{\partial \vect(f_{X}(\theta))}{\partial \theta})\subseteq \RR^{n m_y}$, then $\eta Y^{*} \in \{\sum_{k=1}^{d} 
\beta_k \frac{\partial f_{X}(\theta)}{\partial \theta_k}:\beta \in \RR^d \} \allowbreak \subseteq \RR^{n \times m_y}$ for any $\eta \in \RR$.
Therefore,  if $\vect(Y^{*}) \in \Col(\frac{\partial \vect(f_{X}(\theta))}{\partial \theta})$, by setting 
$$\hat \beta(\theta,\eta):=\eta \left(\left(\frac{\partial \vect(f_{X}(\theta))}{\partial \theta}\right)\T \frac{\partial \vect(f_{X}(\theta))}{\partial \theta}\right)^{\dagger}\left(\frac{\partial \vect(f_{X}(\theta))}{\partial \theta}\right)\T \vect(Y^{*}) ,
$$ 
we have  $\eta Y^*=\sum_{k=1}^{d} 
\hat \beta(\theta,\eta)_k \frac{\partial f_{X}(\theta)}{\partial \theta_k}$ and thus
$$
\Lcal_\theta(\hat \beta(\theta,\eta))=\frac{1}{n}\sum_{i=1}^n  \ell\left( 
 \eta y_i^{*}, y_{i}\right)=\Lcal^*(\eta Y^{*}) .
$$
Using this $\hat \beta(\theta,\eta)$ in  Lemma \ref{lemma:pgd_approx}, we have that for any  $\theta$ satisfying $\vect(Y^*) \in \Col(\allowbreak\frac{\partial \vect(f_{X}(\theta))}{\partial \theta})$,
\begin{align}
\Lcal(\theta)  &\le\Lcal_\theta(\hat \beta(\theta,\eta)) +\|\nu(\theta)-\hat \beta(\theta,\eta) \|_2 \|\nabla\Lcal(\theta)\|_2,
\\ & =\Lcal^*(\eta Y^{*}) + \|\nu(\theta)-\hat \beta(\theta,\eta) \|_2 \|\nabla\Lcal(\theta)\|_2.
\end{align}
Therefore, if  $\theta^{t}$ satisfies $\vect(Y^*) \in \Col(\frac{\partial \vect(f_{X}(\theta^{t}))}{\partial \theta^{t}})$ for all $t \in\Tcal$, then
\begin{align}
\min_{t\in\Tcal}\Lcal(\theta^{t}) & \le\min_{t\in\Tcal}\left(\Lcal^*(\eta Y^{*})  +\|\nu(\theta^{t})-\hat \beta(\theta^{t},\eta)\|_2 \|\nabla\Lcal(\theta^{t})\|_2\right) 
\\ & \le \Lcal^*(\eta Y^{*})   + \sqrt{B_{\eta}'}\min_{t\in\Tcal}\|\nabla\Lcal(\theta^{t})\|_2,
\end{align}
where 
$$
B_{\eta}':= \max_{t\in\Tcal}\|\nu(\theta^{t})-\hat \beta(\theta^{t},\eta) \|_2^{2}.
$$
Using Lemma \ref{lemma:known+alpha_3}, we have that 
\begin{align}
\min_{t\in\Tcal}\Lcal(\theta^{t}) &\le\Lcal^*(\eta Y^{*})   + \frac{1}{\sqrt{|\Tcal|}}    \sqrt{\frac{B_{\eta}'L(\Lcal(\theta^{t_0})-\Lcal(\theta^{t_{|\Tcal|-1}+1}))}{2\alpha(1-\alpha)}}.
  \\ & \le\Lcal^*(\eta Y^{*})   + \frac{1}{\sqrt{|\Tcal|}}    \sqrt{\frac{L\zeta_{\eta}\Lcal(\theta^{t_0})}{2\alpha(1-\alpha)}},
\end{align}
where the last line follows from the fact that 
\begin{align}
\|\nu(\theta^{t})-\hat \beta(\theta^{t},\eta) \|_2^{2} &\le\|\nu(\theta^{t})\|^{2} + \|\hat \beta(\theta^{t},\eta)\|^2 +  2 \|\nu(\theta^{t})\| \|\hat \beta(\theta^{t},\eta) \|
\\ &\le 4\max(\|\nu(\theta^{t})\|^{2} , \|\hat \beta(\theta^{t},\eta)\|^2 )
\end{align}
By noticing that $t_0 =\min\{t:t \in \Tcal\}$, we complete the proof.

\end{proof}

\subsection{Proof of Proposition \ref{prop:1}} \label{sec:app:1:3}
\begin{proof}
From Assumption \ref{assump:7}, we have that
\begin{align}
f(x_i,\theta) =\sum_{k=1}^{d} 
\nu(\theta)_k \frac{\partial f(x_{i},\theta)}{\partial \theta_k}, 
\end{align}
where $\nu(\theta)_k=\theta_k$ for all $k\in S$ in Assumption \ref{assump:7} and  $\nu(\theta)_k=0$ for all  $k \notin S$. This implies that $$
\vect(f_X(\theta^{t})) \in\Col\left(\frac{\partial \vect(f_{X}(\theta^{t}))}{\partial \theta^{t}}\right). 
$$
If  $\vect(Y^{*}) \notin \Col(\frac{\partial \vect(f_{X}(\theta^{t}))}{\partial \theta^{t}})$,
then $\vect(\eta Y^{*}) \notin \Col(\frac{\partial \vect(f_{X}(\theta^{t}))}{\partial \theta^{t}})$ for  any $\eta\in\RR$ by the definition of the column space. Thus, if  $\vect(Y^{*}) \notin \Col(\frac{\partial \vect(f_{X}(\theta^{t}))}{\partial \theta^{t}})$,
we have  $\vect(f_X(\theta^{t}))\neq\vect(\eta Y^{*})$ for  any $\eta\in\RR$, which implies the statement of this proposition.

\end{proof}

\subsection{Proof of Theorem \ref{thm:1}} \label{sec:app:1:4}
The following two lemmas are the key lemmas in the proof of Theorem \ref{thm:1} --- the safe-exploration condition ensures that the Lebesgue measure of the  set of getting stuck to the states of  $\vect( Y^{*}) \notin \Col(\frac{\partial \vect(f_{X}(\theta^{t}))}{\partial \theta^{t}})$  is zero:

\begin{lemma} \label{lemma:2}
 Suppose Assumptions \ref{assump:1}  and \ref{assump:2} hold. Then, for any $j \in \{1,\dots, m_y\}$, the Lebesgue measure of the  set 
$
 \{[\vect(R^{(j)})\T,\theta_{(1:H-2)}\T, W^{(H)}_{j*}]\T \in \RR^{m_{H} m_{H-1} + d_{1:H-2}+ m_H}  : \rank(\phi(R^{(j)},\theta_{(1:H-2)})[\diag(\allowbreak(W^{(H)}_{j*})\T) \otimes I_{m_{H-1}}])\neq n \}
$
is zero.
\end{lemma}
\begin{proof}[Proof of Lemma \ref{lemma:2}] 
Let $v=[\vect(R^{(j)})\T,\theta_{(1:H-2)}\T ,W^{(H)}_{j*}]\T$ and $J=[\diag((W^{(H)}_{j*})\T) \otimes I_{m_{H-1}}]$. Define
\begin{align}
\varphi(v)= \det(\phi(R^{(j)},\theta_{(1:H-2)})JJ\T\phi(R^{(j)},\theta_{(1:H-2)})^{\top}).
\end{align}
 Since the rank of $\allowbreak\phi(R^{(j)},\theta_{(1:H-2)})$ and the rank of the Gram matrix are equal, we have that 
\begin{align} \label{eq:2}
&\{v \in \RR^{m_{H} m_{H-1} + d_{1:H-2}+ m_H}  : \rank(\phi(R^{(j)},\theta_{(1:H-2)})J)\neq n \}
\\ \nonumber&=\{v\in \RR^{m_{H} m_{H-1}+d_{1:H-2}}:\varphi(v)=0\}.
\end{align}
Using Assumption \ref{assump:1}, 
\begin{align} \label{eq:4}
\exists(R^{(j)}, \theta_{(1:H-2)})\in \RR^{m_{H}\times m_{H-1}} \times \RR^{d_{1:H-2}}  \ \ \ \text{ s.t. } \ \ \   \rank(\phi(R^{(j)},\theta_{(1:H-2)}))=n.
\end{align}
Since $J=[\diag((W^{(H)}_{j*})\T) \otimes I_{m_{H-1}}]$ is a diagonal matrix with repeated entries of $W^{(H)}_{j*}$, by setting $W^{(H)}_{jk}>0$ to be strictly positive for all $k =1,\dots,m_H$, equation \eqref{eq:4} implies that
\begin{align} \label{eq:5}
\scalebox{0.94}{$\displaystyle \exists(R^{(j)}, \theta_{(1:H-2)},W^{(H)}_{j*})\in \RR^{m_{H}\times m_{H-1}} \times \RR^{d_{1:H-2}} \times \RR^{m_H} \ \text{s.t.}  \   \rank(\phi(R^{(j)},\theta_{(1:H-2)})J)
=n.$}
\end{align}
Combining  \eqref{eq:2} and \eqref{eq:5} with $v=[\vect(R^{(j)})\T,\theta_{(1:H-2)}\T ,W^{(H)}_{j*}]\T$, we have that
\begin{align} \label{eq:3}
\exists v\in\RR^{m_{H}\times m_{H-1}} \times \RR^{d_{1:H-2}} \times \RR^{m_H}  \ \ \ \text{ s.t. } \ \ \   \varphi(v)\neq 0.
\end{align}

Since the functions $\theta_{(1:H-2)} \mapsto z(x_{i},\theta_{(1:H-2)})$ and $q\mapsto \tsigma(q)$ are real analytic for each  $i \in \{1,\dots, n\}$ (Assumption \ref{assump:2}), the function $\phi$ is real analytic (because a composition of  real analytic functions is real analytic and an affine map is real analytic). This implies that the function $\varphi$ is   real analytic  since $\phi$ is real analytic, the function $\det$ is real analytic, and a  composition of   real analytic functions is real analytic. Since $\varphi$ is real analytic, if $\varphi$ is not identically zero ($\varphi\neq 0$),  the Lebesgue measure of its zero set 
$$
\{v\in \RR^{m_{H}\times m_{H-1}} \times \RR^{d_{1:H-2}} \times \RR^{m_H}:\varphi(v)=0\}
$$
is zero \citep{mityagin2015zero}. Equation \eqref{eq:3} shows that  $\varphi$ is not identically zero ($\varphi\neq 0$), and thus   the  Lebesgue measure of the set 
$$\{v\in \RR^{m_{H}\times m_{H-1}} \times \RR^{d_{1:H-2}} \times \RR^{m_H}:\varphi(v)=0\}
$$ is zero. Using equation \eqref{eq:2}, this implies that  the  Lebesgue measure of the set
$$\{v \in \RR^{m_{H} m_{H-1} + d_{1:H-2}+ m_H}  : \rank(\phi(R^{(j)},\theta_{(1:H-2)})J)\neq n \}
$$
is zero. This implies the statement of this lemma by recalling that $v=[\vect(R^{(j)})\T,\allowbreak \theta_{(1:H-2)}\T ,W^{(H)}_{j*}]\T$ and $J=[\diag((W^{(H)}_{j*})\T) \otimes I_{m_{H-1}}]$.

\end{proof}
\begin{lemma} \label{lemma:4}
Suppose Assumptions \ref{assump:1}  and \ref{assump:2} hold. Therefore, with probability one, for all $j \in \{1,\dots,m_y\}$,
\begin{align}
\rank(\phi(\theta_{(H-1,j)}^\tau,\theta_{(1:H-2)}^{\tau})[\diag(\theta_{(H,j)}^\tau )\otimes I_{m_{H-1}}])=n.
\end{align}
\end{lemma} 
\begin{proof}[Proof of Lemma \ref{lemma:4}]
From Lemma \ref{lemma:2}, the Lebesgue measure of the  set 
$
 \{[\vect(R^{(j)})\T,\allowbreak\theta_{(1:H-2)}\T, \theta_{(H,j)}\T]\T \in \RR^{m_{H} m_{H-1} + d_{1:H-2}+ m_H}  : \rank(\phi(R^{(j)},\theta_{(1:H-2)})[\diag(\theta_{(H,j)}) \otimes I_{m_{H-1}}]\allowbreak)\neq n \}
$
is zero. Moreover, $\theta^{\tau}=\theta^{\tau-1}+\varepsilon \delta$ defines a non-degenerate Gaussian measure with the mean shifted by $\theta^{\tau-1}$ and a non-degenerate Gaussian measure
with any mean and variance is absolutely continuous with respect to the Lebesgue measure. Therefore,  for each   $j \in \{1,\dots,m_y\}$,
 with probability one,
\begin{align}
\rank(\phi(\theta_{(H-1,j)}^\tau,\theta_{(1:H-2)}^{\tau})[\diag(\theta_{(H,j)}^\tau )\otimes I_{m_{H-1}}])=n.
\end{align}
Since a product of ones is one,  this holds with probability one for all $j \in \{1,\dots,m_y\}$. 
\end{proof}

In the following, we will use  the following additional notation:
$$
\hat f_X(\theta_{(H-1)}):=f_{X}(\theta_{(1:H-2)}^{\tau},\theta_{(H-1)},\theta_{(H)}^{\tau})
$$
$$
\hf_i(\theta_{(H-1)}):=f_{i}(\theta_{(1:H-2)}^{\tau},\theta_{(H-1)},\theta_{(H)}^\tau) $$
The following lemma relates the two functions $f_X$ and $\phi$:
\begin{lemma} \label{lemma:1}
For $\theta_{(H-1)} \in \RR^{d_{H-1}}$ and $j \in \{1,\dots,m_y\}$, the following holds:
\begin{align} 
\hf_{X}(\theta_{(H-1)})_{*j} =\phi_{}(\theta_{(H-1,j)}^\tau,\theta_{(1:H-2)}^{\tau})[\diag(\theta_{(H,j)}^\tau )\otimes I_{m_{H-1}}]\vect(\theta_{(H-1,j)}^{}).
\end{align}
\end{lemma}
\begin{proof}[Proof of Lemma \ref{lemma:1}] 
Using the definition of $f$ and $z= z(x,\theta_{(1:H-2)})$,
we rewrite the inner product in the definition of $f(x,\theta)_{j}$ with a sum of scalar multiplications and then rewrite the sum  to a dot product of two block vectors: \begin{align}
f(x,\theta)_{j}&=W^{(H)}_{j*} \sigma(W^{(H-1,j)},z) 
\\ & = \sum_{k=1}^{m_H}  W^{(H)}_{jk} \sigma(W^{(H-1,j)},z)_{k}
\\ & = \sum_{k=1}^{m_H}  W^{(H)}_{jk} \left(\tsigma(W^{(H-1,j)}z)\circ (W^{(H-1,j)}z)\right)_k
\\ & = \sum_{k=1}^{m_H}  W^{(H)}_{jk}\tsigma(W^{(H-1,j)}_{k*}z) W^{(H-1,j)}_{k*}z
\\ & = \sum_{k=1}^{m_H}  W^{(H)}_{jk}\tsigma(W^{(H-1,j)}_{k*}z)z\T  (W^{(H-1,j)}_{k*})\T
\\ & = \begin{bmatrix}W^{(H)}_{j1}\tsigma(W^{(H-1,j)}_{1*}z)z\T   & \cdots & W^{(H)}_{jm_{H}}\tsigma(W^{(H-1,j)}_{m_{H}*}z)z\T \\
\end{bmatrix}\begin{bmatrix}(W^{(H-1,j)}_{1*})\T \\
\vdots \\
(W^{(H-1,j)}_{m_{H}*})\T \\
\end{bmatrix} 
\end{align}
Here, noticing that 
$$
\begin{bmatrix}(W^{(H-1,j)}_{1*})\T \\
\vdots \\
(W^{(H-1,j)}_{m_{H}*})\T \\
\end{bmatrix}  = \vect((W^{(H-1,j)})\T),
$$
we have that 
\begin{align} 
f(x,\theta)_{j} = \begin{bmatrix}W^{(H)}_{j1}\tsigma(W^{(H-1,j)}_{1*}z)z\T   & \cdots & W^{(H)}_{jm_{H}}\tsigma(W^{(H-1,j)}_{m_{H}*}z)z\T \\
\end{bmatrix}\vect((W^{(H-1,j)})\T). 
\end{align}
By arranging the model output $f(x,\theta)_{j}$ over sample points $(x_{1},\dots,x_n)$, it yields that
\begin{align} \label{eq:1}
\scalebox{0.97}{$\displaystyle f_{X}(\theta)_{*j} = \begin{bmatrix}W^{(H)}_{j1}\tsigma(W^{(H-1,j)}_{1*}z_{1})z_{1}\T & \cdots & W^{(H)}_{jm_{H}}\tsigma(W^{(H-1,j)}_{m_{H}*}z_{1})z_{1}\T \\
\vdots & \ddots & \vdots \\
W^{(H)}_{j1}\tsigma(W^{(H-1,j)}_{1*}z_{n})z_{n}\T & \cdots & W^{(H)}_{jm_{H}}\tsigma(W^{(H-1,j)}_{m_{H}*}z_{n})z_{n}\T \\
\end{bmatrix} \vect((W^{(H-1,j)})\T). $}
\end{align}
We can further simplify the first factor in the right hand side of this  equation as follows:
\begin{align}
& \begin{bmatrix}W^{(H)}_{j1}\tsigma(W^{(H-1,j)}_{1*}z_{1})z_{1}\T & \cdots & W^{(H)}_{jm_{H}}\tsigma(W^{(H-1,j)}_{m_{H}*}z_{1})z_{1}\T \\
\vdots & \ddots & \vdots \\
W^{(H)}_{j1}\tsigma(W^{(H-1,j)}_{1*}z_{n})z_{n}\T & \cdots & W^{(H)}_{jm_{H}}\tsigma(W^{(H-1,j)}_{m_{H}*}z_{n})z_{n}\T \\
\end{bmatrix}
\\ & = \begin{bmatrix}\tsigma(W^{(H-1,j)}_{1*}z_{1})z_{1}\T & \cdots & \tsigma(W^{(H-1,j)}_{m_{H}*}z_{1})z_{1}\T \\
\vdots & \ddots & \vdots \\
\tsigma(W^{(H-1,j)}_{1*}z_{n})z_{n}\T & \cdots & \tsigma(W^{(H-1,j)}_{m_{H}*}z_{n})z_{n}\T \\
\end{bmatrix} [\diag((W^{(H)}_{j*})\T) \otimes I_{m_{H-1}}].
\end{align}
By using the definitions of $\hf_{X}(\theta_{(H-1)}):=f_{X}(\theta_{(1:H-2)}^{\tau},\theta_{(H-1)},\theta_{(H)}^{\tau})$ and $\phi_{}(q,\theta_{(1:H-2)})$,
this implies that 
\begin{align}
\hf_{X}(\theta_{(H-1)})_{*j} =\phi_{}(\theta_{(H-1,j)}^\tau,\theta_{(1:H-2)}^{\tau})[\diag(\theta_{(H,j)}^\tau )\otimes I_{m_{H-1}}]\vect(\theta_{(H-1,j)}). 
\end{align}

\end{proof}

The following lemma relates the two functions $\theta_{(H-1)} \mapsto \frac{\partial \hLcal(\theta_{(H-1)})}{\partial\theta_{(H-1)} }$ and $\phi$:

\begin{lemma} \label{lemma:3}
Suppose Assumption \ref{assump:3} holds. Then, for any $t \ge \tau$,
\begin{align}
\left(\frac{\partial \hLcal(\theta_{(H-1)})}{\partial\theta_{(H-1)} }\right)\T=\begin{bmatrix}[\diag(\theta_{(H,1)}^\tau )\otimes I_{m_{H-1}}]\phi(\theta_{(H-1,1)}^\tau,\theta_{(1:H-2)}^{\tau})\T   r_{1}(\theta_{(H-1)}) \\
\vdots \\
[\diag(\theta_{(H,m_{y})}^\tau )\otimes I_{m_{H-1}}]\phi(\theta_{(H-1,m_{y})}^\tau,\theta_{(1:H-2)}^{\tau})\T   r_{m_{y}}(\theta_{(H-1)}) \\
\end{bmatrix} ,
\end{align}
where $\left(\frac{\partial \hLcal(\theta_{(H-1)})}{\partial\theta_{(H-1)} }\right)\T\in \RR^{m_{y}m_{H}m_{H-1} }$ and
$$
r_{j}(\theta_{(H-1)})=\frac{1}{n}\begin{bmatrix}\partial_{j}\ell_{1}(\hf_1(\theta_{(H-1)})) \\
\vdots \\
\partial_{j}\ell_{n}(\hf_n(\theta_{(H-1)})) \\
\end{bmatrix} \in \RR^{n}.
$$
\end{lemma}
\begin{proof}[Proof of Lemma \ref{lemma:3}]
The function $\hLcal$ is differentiable since $\ell_{i}$ is differentiable  (Assumption \ref{assump:3}), $\theta_{(H-1)}\mapsto \hf_{X}(\theta_{(H-1)})_{ij}$ is differentiable, and a composition of differentiable functions is differentiable.  From Lemma \ref{lemma:1}, \begin{align} \label{eq:6} 
\hf_{X}(\theta_{(H-1)})_{*j} =\phi_{}(\theta_{(H-1,j)}^\tau,\theta_{(1:H-2)}^{\tau})[\diag(\theta_{(H,j)}^\tau )\otimes I_{m_{H-1}}]\vect(\theta_{(H-1,j)}^{})\in \RR^{n }.
\end{align}
This implies that
\begin{align} \label{eq:7}
\partial_{\vect(\theta_{(H-1,j)}^{t})} \hf_{X}(\theta_{(H-1)})_{*j} =\phi_{}(\theta_{(H-1,j)}^\tau,\theta_{(1:H-2)}^{\tau})[\diag(\theta_{(H,j)}^\tau )\otimes I_{m_{H-1}}]\in \RR^{n \times  m_{H}m_{H-1}}.
\end{align}
Since $\partial_{\vect(\theta_{(H-1,j)})}f(x_{i},\theta)_{k}=0$ when $k\neq j$, 
\begin{align} \label{eq:8}
\nonumber \partial_{\vect(\theta_{(H-1,j)})}\hLcal(\theta_{(H-1)}) &=\frac{1}{n} \sum_{i=1}^n \partial\ell_{i}(\hf_i(\theta_{(H-1)}))\T\partial_{\vect(\theta_{(H-1,j)})}\hf_{i}(\theta_{(H-1)})
\\ \nonumber & =\frac{1}{n} \sum_{i=1}^n  \sum_{k=1}^{m_y} \partial_{k}\ell_{i}(\hf_i(\theta_{(H-1)}))\partial_{\vect(\theta_{(H-1,j)})}\hf_{i}(\theta_{(H-1)})_{k} \\\nonumber & =\frac{1}{n}  \sum_{i=1}^n \partial_{j}\ell_{i}(\hf_i(\theta_{(H-1)}))\partial_{\vect(\theta_{(H-1,j)})}\hf_{i}(\theta_{(H-1)})_{j} \\ & =r_{j}(\theta_{(H-1)})\T\partial_{\vect(\theta_{(H-1,j)})} f_{X}(\theta)_{*j}. 
\end{align}
Combining equations \eqref{eq:7} and \eqref{eq:8}, 
\begin{align}
\scalebox{0.94}{$\displaystyle \partial_{\vect(\theta_{(H-1,j)})}\Lcal(\theta) =r_{j}(\theta_{(H-1)})\T\phi_{}(\theta_{(H-1,j)}^\tau,\theta_{(1:H-2)}^{\tau})[\diag(\theta_{(H,j)}^\tau )\otimes I_{m_{H-1}}]\in \RR^{1 \times  m_{H}m_{H-1} }. $}
\end{align} 
This implies that 
\begin{align}
\left(\frac{\partial \hLcal(\theta_{(H-1)})}{\partial\theta_{(H-1)} }\right)\T=\begin{bmatrix}[\diag(\theta_{(H,1)}^\tau )\otimes I_{m_{H-1}}]\phi(\theta_{(H-1,1)}^\tau,\theta_{(1:H-2)}^{\tau})\T   r_{1}(\theta_{(H-1)}) \\
\vdots \\
[\diag(\theta_{(H,m_{y})}^\tau )\otimes I_{m_{H-1}}]\phi(\theta_{(H-1,m_{y})}^\tau,\theta_{(1:H-2)}^{\tau})\T   r_{m_{y}}(\theta_{(H-1)}) \\
\end{bmatrix}. 
\end{align}

\end{proof}

The following lemma proves the Lipschitz continuity of $\hLcal$ and derives the Lipschitz constant $\hL$:

\begin{lemma} \label{lemma:5}
Suppose Assumption \ref{assump:3} holds. Then, for any $\theta_{(H-1)}',\theta_{(H-1)} \in \RR^{d_{H-1}}$,
$$
\|\nabla \hLcal(\theta_{(H-1)}')-\nabla \hLcal(\theta_{(H-1)})\| \le \hL\|\theta'_{(H-1)}-\theta_{(H-1)}\|,
$$
where $\hL := \frac{L_\ell}{n^{}}  \|Z\|^2$, and $Z \in \RR^n$ with  
$$
Z_i := \max_j\|[\diag(\theta_{(H,j)}^\tau )\otimes I_{m_{H-1}}](\phi_{}(\theta_{(H-1,j)}^\tau,\theta_{(1:H-2)}^{\tau})_{i*})\T\|.
$$
\end{lemma}
\begin{proof}[Proof of Lemma \ref{lemma:5}] 
Using Lemma \ref{lemma:3}, since $\nabla \hLcal(\theta_{(H-1)})=\left(\frac{\partial \hLcal(\theta_{(H-1)})}{\partial\theta_{(H-1)} }\right)\T$,
$$
 \nabla \hLcal(\theta_{(H-1)})=\begin{bmatrix}[\diag(\theta_{(H,1)}^\tau )\otimes I_{m_{H-1}}]\phi(\theta_{(H-1,1)}^\tau,\theta_{(1:H-2)}^{\tau})\T   r_{1}(\theta_{(H-1)}) \\
\vdots \\
[\diag(\theta_{(H,m_{y})}^\tau )\otimes I_{m_{H-1}}]\phi(\theta_{(H-1,m_{y})}^\tau,\theta_{(1:H-2)}^{\tau})\T   r_{m_{y}}(\theta_{(H-1)}) \\
\end{bmatrix}.  
$$
This implies that
for any $\theta_{(H-1)}',\theta_{(H-1)}$,
\begin{align}
&\|\nabla \hLcal(\theta_{(H-1)}')-\nabla \hLcal_{}^{}(\theta_{(H-1)})\|^{2}
\\ & = \left\|\begin{bmatrix}[\diag(\theta_{(H,1)}^\tau )\otimes I_{m_{H-1}}]\phi(\theta_{(H-1,1)}^\tau,\theta_{(1:H-2)}^{\tau})\T(   r_{1}(\theta_{(H-1)}')-   r_{1}(\theta_{(H-1)})) \\
\vdots \\
[\diag(\theta_{(H,m_{y})}^\tau )\otimes I_{m_{H-1}}]\phi(\theta_{(H-1,m_{y})}^\tau,\theta_{(1:H-2)}^{\tau})\T   (r_{m_{y}}(\theta_{(H-1)}')-r_{m_{y}}(\theta_{(H-1)})) \\
\end{bmatrix} \right\|^2
\\ & = \sum_{j=1}^{m_y}  \| [\diag(\theta_{(H,j)}^\tau )\otimes I_{m_{H-1}}]\phi(\theta_{(H-1,j)}^\tau,\theta_{(1:H-2)}^{\tau})\T(   r_{j}(\theta_{(H-1)}')-   r_{j}(\theta_{(H-1)})) \|^2
\\ & \le\sum_{j=1}^{m_y}  \| [\diag(\theta_{(H,j)}^\tau )\otimes I_{m_{H-1}}]\phi(\theta_{(H-1,j)}^\tau,\theta_{(1:H-2)}^{\tau})\T\|^2 _{2}\|   r_{j}(\theta_{(H-1)}')-   r_{j}(\theta_{(H-1)}) \|^2   
\\ & \scalebox{0.87}{$\displaystyle =\frac{1}{n^{2}} \sum_{j=1}^{m_y}  \| [\diag(\theta_{(H,j)}^\tau )\otimes I_{m_{H-1}}]\phi(\theta_{(H-1,j)}^\tau,\theta_{(1:H-2)}^{\tau})\T\|^2_2  \sum_{i=1}^n (\partial_{j}\ell_{i}(\hf_i(\theta_{(H-1)}'))-   \partial_{j}\ell_{i}(\hf_i(\theta_{(H-1)})))^2 $} 
\end{align} 
From Assumption \ref{assump:3}, we also have that for all $q,q' \in \RR$, 
$$
\sum_{j=1}^{m_y}  (\partial_{j} \ell_{i}(q) - \partial_{j} \ell_{i}(q'))^2 =\|\nabla \ell_{i}(q) - \nabla \ell_{i}(q')\|^{2} \le L_\ell ^{2}\|q - q'\|^{2}.
$$ 
Using this and $\bar M:=\max_j\| [\diag(\theta_{(H,j)}^\tau )\otimes I_{m_{H-1}}]\phi(\theta_{(H-1,j)}^\tau,\theta_{(1:H-2)}^{\tau})\T\|_2$, 
\begin{align}
&\|\nabla \hLcal(\theta_{(H-1)}')-\nabla \hLcal_{}^{}(\theta_{(H-1)})\|^{2}
\\ & \le \scalebox{0.85}{$\displaystyle \frac{1}{n^{2}} \sum_{j=1}^{m_y}  \| [\diag(\theta_{(H,j)}^\tau )\otimes I_{m_{H-1}}]\phi(\theta_{(H-1,j)}^\tau,\theta_{(1:H-2)}^{\tau})\T\|^2_2  \sum_{i=1}^n (\partial_{j}\ell_{i}(\hf_i(\theta_{(H-1)}'))-   \partial_{j}\ell_{i}(\hf_i(\theta_{(H-1)})))^2 $}
\\ & \le \frac{1}{n^{2}}\bar M^{2} \sum_{i=1}^n  \sum_{j=1}^{m_y} (\partial_{j}\ell_{i}(\hf_i(\theta_{(H-1)}'))-   \partial_{j}\ell_{i}(\hf_i(\theta_{(H-1)})))^2
\\ & \le \frac{1}{n^{2}}L_\ell ^{2}\bar M^{2} \sum_{i=1}^n     \|\hf_i(\theta_{(H-1)}')- \hf_i(\theta_{(H-1)})\|^{2}.
\end{align}
Since $\hf_i(\theta_{(H-1)})_{j} =\phi_{}(\theta_{(H-1,j)}^\tau,\theta_{(1:H-2)}^{\tau})_{i*}[\diag(\theta_{(H,j)}^\tau )\otimes I_{m_{H-1}}]\vect(\theta_{(H-1,j)}^{})$ from Lemma \ref{lemma:1}, \begin{align}
&\|\nabla \hLcal(\theta_{(H-1)}')-\nabla \hLcal_{}^{}(\theta_{(H-1)})\|^{2}
\\ & \le\frac{1}{n^{2}} L_\ell ^{2}\bar M^{2} \sum_{i=1}^n     \|\hf_i(\theta_{(H-1)}')- \hf_i(\theta_{(H-1)})\|^{2}
\\ & \le \scalebox{0.8}{$\displaystyle \frac{1}{n^{2}}L_\ell ^{2}\bar M^{2} \sum_{i=1}^n \sum_{j=1}^{m_y} \|[\diag(\theta_{(H,j)}^\tau )\otimes I_{m_{H-1}}](\phi_{}(\theta_{(H-1,j)}^\tau,\theta_{(1:H-2)}^{\tau})_{i*})\T\|^{2}   \|\vect(\theta'_{(H-1,j)})-\vect(\theta_{(H-1,j)})\|^{2}.$}
\end{align}
Using $Z_i=\max_j\|[\diag(\theta_{(H,j)}^\tau )\otimes I_{m_{H-1}}](\phi_{}(\theta_{(H-1,j)}^\tau,\theta_{(1:H-2)}^{\tau})_{i*})\T\|$,
\begin{align}
&\|\nabla \hLcal(\theta_{(H-1)}')-\nabla \hLcal_{}^{}(\theta_{(H-1)})\|^{2}
\\ & \le \frac{1}{n^{2}}L_\ell ^{2}\bar M^{2} \sum_{i=1}^n Z_i^{2}\sum_{j=1}^{m_y} \|\vect(\theta'_{(H-1,j)})-\vect(\theta_{(H-1,j)})\|^{2}. \\ & = \left(L_\ell ^{2}\bar M^{2} \frac{1}{n^{2}}\sum_{i=1}^n Z_i^{2}\right) \|\theta'_{(H-1)}-\theta_{(H-1)}\|^{2}. 
\end{align}
This implies that 
$$
\|\nabla \hLcal(\theta_{(H-1)}')-\nabla \hLcal(\theta_{(H-1)})\| \le \left(L_\ell \bar M\frac{1}{n^{}} \sqrt{\sum_{i=1}^n Z_i^{2}}\right) \|\theta'_{(H-1)}-\theta_{(H-1)}\| 
$$
Since
\begin{align}
\bar M^{2}=\max_j\| [\diag(\theta_{(H,j)}^\tau )\otimes I_{m_{H-1}}]\phi(\theta_{(H-1,j)}^\tau,\theta_{(1:H-2)}^{\tau})\T\|_2^{2} \le\sum_{i=1}^n Z_i^{2}, 
\end{align} 
we have that 
$$
\|\nabla \hLcal(\theta_{(H-1)}')-\nabla \hLcal(\theta_{(H-1)})\| \le \left( \frac{L_\ell}{n^{}} \sum_{i=1}^n Z_i^{2}\right) \|\theta'_{(H-1)}-\theta_{(H-1)}\|.
$$
By defining $Z \in \RR^n$ with $Z_{i} = \max_j\|[\diag(\theta_{(H,j)}^\tau )\otimes I_{m_{H-1}}](\phi_{}(\theta_{(H-1,j)}^\tau,\theta_{(1:H-2)}^{\tau})_{i*})\T\|$,
we have 
$$
\|\nabla \hLcal(\theta_{(H-1)}')-\nabla \hLcal(\theta_{(H-1)})\| \le \frac{L_\ell}{n^{}}  \|Z\|^2 \|\theta'_{(H-1)}-\theta_{(H-1)}\|.
$$

\end{proof}

For the  convenient referencing during our proofs, we label the assumption  on $\bg^t$  in Theorem \ref{thm:1} as follows:

\begin{assumption} \label{assump:4}
There exist $\bc>0$ and $\uc>0$ such that for any $t \ge \tau$,
$$
\uc \|\nabla\hLcal(\theta_{(H-1)}^{t})\|_2^2 \le\nabla\hLcal(\theta_{(H-1)}^{t})\T \bg^t, \text{ and }  
$$
$$
 \|\bg^t\|_2^2 \le \bc \|\nabla\hLcal(\theta_{(H-1)}^{t})\|_2^2.
$$
\end{assumption}
Note that this is not an additional assumption. This is already assumed in Theorem \ref{thm:1}.
For example, Assumption \ref{assump:4} is satisfied by using any optimizer $\tGcal$ such that for all $t\ge \tau$, 
$$
\bg^t =  D^t\nabla\hLcal(\theta_{(H-1)}^{t}), 
$$
where $D$ is a positive definite symmetric matrix with eigenvalues in the interval $[\uc, \sqrt{\bc}]$. If we set $D^{t}=I$, then we satisfy  Assumption \ref{assump:4} with $\uc=\bc =1$. 

With Lemmas \ref{lemma:2}--\ref{lemma:5} along with the labeling of the assumption, we are now ready to provide the proof of Theorem \ref{thm:1}:

\begin{proof}[Proof of Theorem \ref{thm:1}] 
 The function $\hLcal$ is differentiable since $\ell_{i}$ is differentiable  (Assumption \ref{assump:3}), $\theta_{(H-1)}\mapsto \hf_{X}(\theta_{(H-1)})_{ij}$ is differentiable, and a composition of differentiable functions is differentiable. We will first show that in both cases of (i) and (ii) for the learning rates, we have $\lim_{t \rightarrow \infty}\nabla\hLcal(\theta_{(H-1)}^{t}) = 0$. If $\nabla\hLcal(\theta_{(H-1)}^{t})=0$ at any $t\ge \tau$, then
Assumption \ref{assump:4} ($\|\bg^t\|_2^2 \le \bc \|\nabla\hLcal(\theta_{(H-1)}^{t})\|_2^2$) implies  $\bg^t=0$, which implies 
$$
\theta_{(H-1)}^{t+1}= \theta^{t}_{(H-1)} \text{ and } \nabla\hLcal(\theta_{(H-1)}^{t+1})=\nabla\hLcal(\theta_{(H-1)}^{t})=0.
$$
This means that  if $\nabla\hLcal(\theta_{(H-1)}^{t})=0$ at any $t\ge \tau$, we have that  $\bg^t=0$ and $\nabla\hLcal(\theta_{(H-1)}^{t})=0$ for all $t \ge \tau$ and hence
\begin{align} 
\lim_{t \rightarrow \infty}\nabla\hLcal(\theta_{(H-1)}^{t}) = 0,
\end{align}
as desired. Therefore, we now focus on the remaining scenario where $\nabla\hLcal(\theta_{(H-1)}^{t})\allowbreak \neq 0$ for all $t \ge \tau$. 

By using Lemma \ref{lemma:known_1} and Lemma \ref{lemma:5}, 
$$
\hLcal(\theta_{(H-1)}^{t+1})\le \hLcal(\theta_{(H-1)}^{t})-\alpha^t  \nabla\hLcal(\theta_{(H-1)}^{t})  \T\bg^t  + \frac{\hL(\alpha^t)^2  }{2} \|\bg^t \|^2.
$$
By rearranging and using Assumption \ref{assump:4},
\begin{align}
\hLcal(\theta_{(H-1)}^{t})-\hLcal(\theta_{(H-1)}^{t+1}) &\ge \alpha^t  \nabla\hLcal(\theta_{(H-1)}^{t})  \T\bg^t  - \frac{\hL(\alpha^t)^2  }{2} \|\bg^t \|^2
\\ & \ge\alpha^t  \uc \|\nabla\hLcal(\theta_{(H-1)}^{t})\|^2 -\frac{\hL(\alpha^t)^2  }{2}\bc \|\nabla\hLcal(\theta_{(H-1)}^{t})\|^2.
\end{align}
By simplifying the right-hand-side, 
\begin{align} \label{eq:9}
 \hLcal(\theta_{(H-1)}^{t})-\hLcal(\theta_{(H-1)}^{t+1}) &\ge\alpha^t   \|\nabla\hLcal(\theta_{(H-1)}^{t})\|^2 (\uc-\frac{\hL\alpha^t  }{2}\bc). 
\end{align}

Let us now focus on case (i). Then, using $ \alpha^t \le \frac{\uc\ (2-\epsilon)}{\hL\bc}$, 
$$
\frac{\hL\alpha^t}{2}\bc\le\frac{\hL\uc\ (2-\epsilon)}{2\hL\bc}\bc =\uc-\frac{\epsilon}{2}\uc.
$$
Using this inequality and using $\epsilon\le \alpha^t$ in equation \eqref{eq:9},
\begin{align} \label{eq:10}
\hLcal(\theta_{(H-1)}^{t})-\hLcal(\theta_{(H-1)}^{t+1}) &\ge\frac{\uc\epsilon^{2}}{2}
\|\nabla\hLcal(\theta_{(H-1)}^{t})\|^2 .
\end{align}
Since $\nabla\hLcal(\theta_{(H-1)}^{t})\neq 0$ for any $t\ge \tau$ (see above) and $\epsilon >0 $, this means that the sequence $(\hLcal(\theta_{(H-1)}^{t}\allowbreak ))_{t\ge \tau}$ is monotonically decreasing. Since $\hLcal(q) \ge 0$ for any $q$ in its domain, this implies that the sequence $(\hLcal(\theta_{(H-1)}^{t}))_{t\ge \tau}$ converges. Therefore, $\hLcal(\theta_{(H-1)}^{t})-\hLcal(\theta_{(H-1)}^{t+1}) \rightarrow 0$  as $t \rightarrow \infty$. Using equation \eqref{eq:10}, this implies that $$
\lim_{t \rightarrow \infty}\nabla\hLcal(\theta_{(H-1)}^{t}) = 0,
$$
which proves the desired result for the case (i). 

We now focus on the case (ii). Then,  equation \eqref{eq:9} still holds. Since $\lim_{t \rightarrow \infty}\alpha^t =0$ in equation \eqref{eq:9}, the first order term in $\alpha^t$ dominates after sufficiently large $t$: i.e., there exists $\bar t \ge \tau$ such that for any $t\ge \bar t$, 
\begin{align} \label{eq:11}
\hLcal(\theta_{(H-1)}^{t})-\hLcal(\theta_{(H-1)}^{t+1}) \ge c \alpha^t   \|\nabla\hLcal(\theta_{(H-1)}^{t})\|^2.
\end{align}
for some constant $c>0$. Since $\nabla\hLcal(\theta_{(H-1)}^{t})\neq 0$ for any $t\ge \tau$ (see above) and $c \alpha^t>0$, this means that the sequence $(\hLcal(\theta_{(H-1)}^{t}))_{t\ge \tau}$ is monotonically decreasing.  Since $\hLcal(q) \ge 0$ for any $q$ in its domain, this implies that the sequence $(\hLcal(\theta_{(H-1)}^{t}))_{t\ge \tau}$ converges to a finite value. Thus, by adding Eq. \eqref{eq:11} both sides over all $t \ge \bar t$, 
        \begin{align} 
              \infty >  \hLcal(\theta_{(H-1)}^{\bar t})-\lim_{t \rightarrow \infty}\hLcal(\theta_{(H-1)}^{t}) \ge c  \sum _{t=\bar t}^\infty \alpha^{t}   \|\nabla\hLcal(\theta_{(H-1)}^{t})\|^2.
        \end{align}        
Since $\sum _{t=0}^\infty \alpha^{t}  = \infty$, this implies that  $\liminf _{t\to \infty }\|\nabla\hLcal(\theta^{t}_{(H-1)})\|=0$. We now show that by contradiction, $ \limsup_{t\to \infty }\|\nabla\hLcal(\theta^{t}_{(H-1)})\|=0$. Suppose that $\limsup_{t\to \infty }\|\allowbreak \nabla\hLcal(\theta^{t}_{(H-1)})\| > 0$. Then, there exists $\delta>0$ such that $\limsup_{t\to \infty }\|\nabla\hLcal(\theta^{t}_{(H-1)})\|\ge \delta$. Since $\liminf _{t\to \infty }\|\nabla\hLcal(\theta^{t}_{(H-1)})\|=0$ and $\limsup_{t\to \infty }\|\nabla\hLcal(\theta^{t}_{(H-1)})\| \allowbreak \ge \delta$, let $(\rho_j)_{j}$ and $(\rho'_j)_j$ be  sequences of indexes such that $\rho_j<\rho'_j<\rho_{j+1}$, $\|\nabla\hLcal(\theta^{t}_{(H-1)})\|>\frac{\delta}{3}$ for $\rho_j \le t < \rho_j'$, and  $\|\nabla\hLcal(\theta^{t}_{(H-1)})\|\le \frac{\delta}{3}$ for $\rho_j '\le t < \rho_{j+1}$. Since $\sum _{t=\bar t}^\infty \alpha^{t}   \|\nabla\hLcal(\theta_{(H-1)}^{t})\|^2< \infty$, let $\bar j$ be sufficiently large such that  $\sum _{t=\rho_{\bar j}}^\infty \alpha^{t}   \|\nabla\hLcal(\theta_{(H-1)}^{t})\|^2<\frac{\delta^2}{9 \hL\sqrt{\bc}}$. Then, for any $j \ge \bar j$ and any $\rho$ such that $\rho_j \le \rho \le \rho'_j -1$, we have that 
\begin{align}
\|\nabla\hLcal(\theta_{(H-1)}^{\rho})\|-\|\nabla\hLcal(\theta_{(H-1)}^{\rho_j'})\|&\le\|\nabla\hLcal(\theta_{(H-1)}^{\rho_j'})-\nabla\hLcal(\theta_{(H-1)}^{\rho})\|
\\ &=\left\|\sum_{t=\rho}^{\rho'_j-1}\nabla\hLcal(\theta_{(H-1)}^{t+1})-\nabla\hLcal(\theta_{(H-1)}^{t})\right\|
 \\ & \le \sum_{t=\rho}^{\rho'_j-1}\left\|\nabla\hLcal(\theta_{(H-1)}^{t+1})-\nabla\hLcal(\theta_{(H-1)}^{t})\right\|  \\ & \le  \hL  \sum_{t=\rho}^{\rho'_j-1}\left\|\theta_{(H-1)}^{t+1}-\theta_{(H-1)}^{t}\right\|
  \\ & \le  \hL\sqrt{\bc}  \sum_{t=\rho}^{\rho'_j-1} \alpha^{t}   \left\|\nabla\hLcal(\theta_{(H-1)}^{t})\right\|
\end{align}
where the first and third lines use the triangle inequality (and symmetry), the forth line uses the fact that     $\|\nabla\hLcal(\theta_{(H-1)})-\nabla\hLcal(\theta'_{(H-1)})\|\le \hL \|\theta_{(H-1)}-\theta'_{(H-1)}\|$, and the last line follows the  facts of  $\theta_{(H-1)}^{t+1}-\theta_{(H-1)}^{t}=-\alpha^t \bg$ and  $\|\bg^t\|^2 \le \bc \|\nabla\hLcal(\theta_{(H-1)}^{t})\|^2
$.  Then, by using the definition of the sequences of the indexes, 
\begin{align}
\|\nabla\hLcal(\theta_{(H-1)}^{\rho})\|-\|\nabla\hLcal(\theta_{(H-1)}^{\rho_j'})\|&\le \frac{3 \hL\sqrt{\bc}}{\delta}  \sum_{t=\rho}^{\rho'_j-1} \alpha^{t}   \left\|\nabla\hLcal(\theta_{(H-1)}^{t})\right\|^2 \le \frac{\delta}{3}.
\end{align}
 Here, since  $\|\nabla\hLcal(\theta_{(H-1)}^{\rho_j'})\|\le \frac{\delta}{3}$, by rearranging the inequality, we have that for any $\rho\ge \rho_{\bar j}$, 
 \begin{align}
\|\nabla\hLcal(\theta_{(H-1)}^{\rho})\|\le\frac{2\delta}{3}.
\end{align}
This contradicts the inequality of  $\limsup_{t\to \infty }\|\nabla\hLcal(\theta^{t}_{(H-1)})\|\ge \delta$. Thus, we have 
 \begin{align}
\limsup_{t\to \infty }\|\nabla\hLcal(\theta^{t}_{(H-1)})\| =  \liminf _{t\to \infty }\|\nabla\hLcal(\theta^{t}_{(H-1)})\|= 0.
\end{align}
This implies that 
        \begin{equation}
                \lim_{t \rightarrow \infty}\nabla\hLcal(\theta_{(H-1)}^{t}) = 0,
\end{equation}
which proves the desired result for the case (ii). 
Therefore, in both cases of (i) and (ii) for the learning rates, we have $\lim_{t \rightarrow \infty}\nabla\hLcal(\theta_{(H-1)}^{t}) = 0$. 

We now use the fact that $\lim_{t \rightarrow \infty}\nabla\hLcal(\theta_{(H-1)}^{t}) = 0$, in order to  prove the statement of this theorem. Using Lemma \ref{lemma:3}, $\lim_{t \rightarrow \infty}\nabla\hLcal(\theta_{(H-1)}^{t}) = 0$ implies that for all $j\in\{1,\dots,m_y\}$,
\begin{align} \label{eq:12}
[\diag(\theta_{(H,j)}^\tau )\otimes I_{m_{H-1}}]\phi(\theta_{(H-1,j)}^\tau,\theta_{(1:H-2)}^{\tau})\T   \left(\lim_{t \rightarrow \infty}r_{j}(\theta_{(H-1)}^{t})\right) = 0.
\end{align}
Here, using Lemma \ref{lemma:4}, with probability one, for any $j \in \{1,\dots, m_y\}$,
\begin{align} \label{eq:13}
\rank([\diag(\theta_{(H,j)}^\tau )\otimes I_{m_{H-1}}]\phi(\theta_{(H-1,j)}^\tau,\theta_{(1:H-2)}^{\tau})\T)=n.
\end{align}
Because equation \eqref{eq:12} implies that $\lim_{t \rightarrow \infty}r_{1}(\theta_{(H-1)}^{t}) $ is in the null space of $[\diag(\theta_{(H,1)}^\tau \allowbreak )\otimes \allowbreak I_{m_{H-1}}]\phi(\allowbreak \theta_{(H-1,1)}^\tau,\theta_{(1:H-2)}^{\tau})\T   $ and because equation \eqref{eq:13} implies that the null space of $[\diag(\theta_{(H,1)}^\tau )\otimes I_{m_{H-1}}]\phi(\allowbreak \theta_{(H-1,1)}^\tau,\theta_{(1:H-2)}^{\tau})\T$ contains only zero, we have that with probability one, for any $j \in \{1,\dots, m_y\}$,
\begin{align} \label{eq:15}
\lim_{t \rightarrow \infty}r_{j}(\theta_{(H-1)}^{t}) = 0.
\end{align}
Note that the statement of this theorem vacuously holds if there is not limit point of the sequence $(\theta^t)_{t}$. Thus, let $\htheta$ be a limit point of  of the sequence $(\theta^t)_{t}$, and  $\htheta_{(H-1)}$ be the corresponding limit point for the $H-1$ layer. Then, equation \eqref{eq:15} implies that 
$$
r_{j}(\htheta_{(H-1)}) = 0.
$$
Using the definition of $r_{j}$, this implies that with probability one, 
\begin{align} \label{eq:14}
\forall j \in \{1,\dots, m_y\}, \forall i \in \{1,\dots,n\}, \ \
 \partial_{j}\ell_{i}(\hf_i(\htheta_{(H-1)}^{})) = 0.
\end{align}
Since  the function $\ell_i$ is invex, there exists a function $\varphi_{i}: \RR^{m_y}  \times \RR^{m_y} \rightarrow  \RR^{m_y}$ such that  for any $q_{i}\in \RR^{m_y}$,
 $$
\ell_{i}(q_{i}) \ge   \ell_{i}(\hf_i(\htheta_{(H-1)}^{})) +\partial\ell_{i}(\hf_i(\htheta_{(H-1)}^{}))  \varphi_{i}(q,\hf_i(\htheta_{(H-1)}^{})) 
$$ 
Combining this with equation \eqref{eq:14} yields that $$
 \forall i \in \{1,\dots,n\}, \forall q_{i}\in \RR^{m_y}, \ \ \ell_{i}(q_{i}) \ge      \ell_{i}(\hf_i(\htheta_{(H-1)}^{})). 
$$
This implies that
$$
\forall q\in\RR^{m_y \times n},  \ \ \frac{1}{n} \sum_{i=1}^n \ell_{i}(q_{*i}) \ge\frac{1}{n}\sum_{i=1}^{n}\ell_{i}(\hf_i(\htheta_{(H-1)}^{}))=\Lcal(\htheta). 
$$
This proves that with probability one, $\Lcal(\htheta) \ge \Lcal(\theta^{*})$ for all $\theta^{*} \in \RR^d$, since for all $\theta^{*} \in \RR^d$, there exists a $q\in\RR^{m_y \times n}$ such that $\Lcal(\theta^{*} ) \ge \frac{1}{n} \sum_{i=1}^n \ell_{i}(q_{*i}).$ Since $\htheta$ was chosen to be  an arbitrary limit point of   the sequence $(\theta^t)_{t}$, this completes the proof of this theorem.

\end{proof}

\subsection{Proof of Theorem \ref{thm:2}} \label{sec:app:1:5}
We  build up on the proof of  Theorem \ref{thm:1} from the previous subsection in order to let the proof of Theorem \ref{thm:2} concise. The function $\hLcal$ is differentiable since $\ell_{i}$ is differentiable  (Assumption \ref{assump:3}), $\theta_{(H-1)}\mapsto \hf_{X}(\theta_{(H-1)})_{ij}$ is differentiable, and a composition of differentiable functions is differentiable. From Lemma \ref{lemma:1} from the   proof of  Theorem \ref{thm:1}, \begin{align} \label{eq:16}
\hf_{X}(\theta_{(H-1)})_{*j} =\phi(\theta_{(H-1,j)}^\tau,\theta_{(1:H-2)}^{\tau})[\diag(\theta_{(H,j)}^\tau )\otimes I_{m_{H-1}}]\vect(\theta_{(H-1,j)}^{}).
\end{align}
Using Lemma \ref{lemma:4}, with probability one, for any $j \in \{1,\dots, m_y\}$,
\begin{align} \label{eq:17}
\rank([\diag(\theta_{(H,j)}^\tau )\otimes I_{m_{H-1}}]\phi(\theta_{(H-1,j)}^\tau,\theta_{(1:H-2)}^{\tau})\T)=n.
\end{align}
Since  $f_X(w)_{*j} \in \RR^{n}$, equations \eqref{eq:16}--\eqref{eq:17} imply that 
with probability one, 
$$
\{f_X(\theta) \in \RR^{n \times m_y}:\theta\in \RR^d \} = \{\hf_{X}(\theta_{(H-1)})_{} \in \RR^{n \times m_y}:\theta_{(H-1)}\in \RR^{d_{H-1}} \}.
$$
Thus,  for any     $\theta^{*} \in \RR^d$, there exists $\theta_{(H-1)}'$ such that $\hf_{X}(\theta_{(H-1)}')=f_X(\theta^{*})$ and   
$\hLcal(\allowbreak \theta_{(H-1)}')=\Lcal(\theta^*)$. This implies that for any $\bepsilon \ge 0$ and $\theta^{*} \in \RR^d$, there exists $\theta_{(H-1)}'$ such that\begin{align}
  \max(\hLcal(\theta_{(H-1)}'), \bepsilon)=\max(\Lcal(\theta^*), \bepsilon).
\end{align}
 This further implies that  for any $ \bepsilon \ge 0$ and  $\theta_{(H-1)} \in\Theta_{\bepsilon}$,
\begin{align}
 \max(\hLcal(\theta_{(H-1)}), \bepsilon)\le \max(\Lcal(\theta^*), \bepsilon), \ \ \ \forall \theta^{*} \in \RR^d.
\end{align}
Therefore, for any $ \bepsilon \ge 0$, by defining $\theta_{(H-1)}^* = \argmin_{\theta_{(H-1)} \in\Theta_{\bepsilon} }\|\theta_{(H-1)}- \theta_{(H-1)}^{\tau}\|$, 
\begin{align} \label{eq:step_2_theorem_0_2}
\hLcal(\theta_{(H-1)}^{*}) \le \max(\hLcal(\theta_{(H-1)}^{*}), \bepsilon)\le \max(\Lcal(\theta^*), \bepsilon), \ \ \ \forall \theta^{*} \in \RR^d.
\end{align}

The proof of Theorem \ref{thm:2} utilizes \eqref{eq:step_2_theorem_0_2} along with  the following well-known fact:

\begin{lemma} \label{lemma:known_2}
 For any differentiable function $\varphi: \RR^{d_{\varphi}} \rightarrow \RR$ and  any $q^{t+1},q^t\in \RR^{d_{\varphi}}$ such that  $q^{t+1}=q^t- \frac{1}{\hL}\nabla \varphi(q^t)$, the following holds:
\text{for all $q\in \RR^{d_{\varphi}}$},
$$
\nabla \varphi(q^{t})\T (q-q^{t}) + \frac{\hL}{2} \|q-q^{t}\|^2 - \frac{\hL}{2} \|q-q^{t+1}\|^2 =\nabla \varphi(q^{t})\T (q^{t+1}-q^{t})+\frac{\hL}{2} \|q^{t+1}-q^{t}\|^2. $$
\end{lemma}
\begin{proof}[Proof of Lemma \ref{lemma:known_2}]
Using $q^{t+1}=q^t- \frac{1}{\hL}\nabla \varphi(q^t)$, which implies $\nabla \varphi(q^t)=\hL(q^t- q^{t+1})$, 
\begin{align}
&\nabla \varphi(q^{t})\T (q-q^{t}) + \frac{\hL}{2} \|q-q^{t}\|^2 - \frac{\hL}{2} \|q-q^{t+1}\|^2 
\\ &=\hL(q^t- q^{t+1})\T(q-q^{t})   + \frac{\hL}{2} \|q-q^{t}\|^2 - \frac{\hL}{2} \|q-q^{t+1}\|^2 
\\ & = \scalebox{0.82}{$\displaystyle \hL\left( q\T\ q^t - \|q^t\|^2-q\T q^{t+1}+ (q^{t})\T q^{t+1}+ \frac{1}{2} \|q\|^2+ \frac{1}{2} \|q^{t}\|^2 -q\T q^t- \frac{1}{2} \|q\|^2-\frac{1}{2} \|q^{t+1}\|^2+q\T q^{t+1} \right) $}
\\ &=\hL\left(  (q^{t})\T q^{t+1}- \frac{1}{2} \|q^{t}\|^2 -\frac{1}{2} \|q^{t+1}\|^2 \right)
\\ & = - \frac{\hL}{2} \|q^{t+1} - q^{t}\|^{2}
\\ & = -\hL(q^{t+1} - q^{t})\T (q^{t+1} - q^{t})+ \frac{\hL}{2}   \|q^{t+1} - q^{t}\|^{2}
\\ & = \nabla \varphi(q^t)(q^{t+1} - q^{t})+ \frac{\hL}{2}   \|q^{t+1} - q^{t}\|^{2}.
\end{align}
\end{proof}

With equation \eqref{eq:step_2_theorem_0_2} and Lemma  \ref{lemma:known_2} along with lemmas from the previous subsection, we are ready to complete  the proof of Theorem \ref{thm:2}:

\begin{proof}[Proof of Theorem \ref{thm:2} (i)]
 Let $t > \tau$.  Using Lemma \ref{lemma:known_1} and Lemma \ref{lemma:5},\begin{align} \label{eq:step_2_theorem_1_2}
\hLcal(\theta_{(H-1)}^{t+1})\le \hLcal(\theta_{(H-1)}^{t})+ \nabla  \hLcal(\theta_{(H-1)}^{t})\T (\theta_{(H-1)}^{t+1}-\theta_{(H-1)}^{t}) + \frac{\hL}{2} \|\theta_{(H-1)}^{t+1}-\theta_{(H-1)}^{t}\|^2 .
\end{align}
Using \eqref{eq:step_2_theorem_1_2} and $\nabla\hLcal(\theta_{(H-1)}^{t})=\hL_{}(\theta_{(H-1)}^t-\theta_{(H-1)}^{t+1}) $,
\begin{align} \label{eq:step_2_theorem_2_2}
\nonumber \hLcal^{}(\theta_{(H-1)}^{t+1}) & \le\hLcal^{}(\theta_{(H-1)}^{t})- \hL_{} \|\theta_{(H-1)}^{t+1}-\theta_{(H-1)}^{t}\|^2 +\frac{\hL}{2} \|\theta_{(H-1)}^{t+1}-\theta_{(H-1)}^{t}\|^2
\\\nonumber  & =\hLcal^{}(\theta_{(H-1)}^{t})- \frac{\hL}{2} \|\theta_{(H-1)}^{t+1}-\theta_{(H-1)}^{t}\|^2 \\ & \le\hLcal(\theta_{(H-1)}^{t}),
 \end{align}
which shows that $\hLcal(\theta_{(H-1)}^{t})$ is non-increasing in $t$.  
Using \eqref{eq:step_2_theorem_1_2} and Lemma \ref{lemma:known_2}, for any $z\in \RR^{d_{H-1}}$, 
\begin{align}  \label{eq:step_2_theorem_3_2}
&\hLcal^{}(\theta_{(H-1)}^{t+1})
\\ \nonumber &\le \hLcal^{}(\theta_{(H-1)}^{t})+ \nabla\hLcal^{}(\theta_{(H-1)}^{t})\T (z^{}-\theta_{(H-1)}^{t}) + \frac{\hL}{2} \|z^{}-\theta_{(H-1)}^{t}\|^2 -\frac{\hL}{2} \|z-\theta_{(H-1)}^{t+1}\|^2 .  
\end{align}
Using \eqref{eq:step_2_theorem_3_2} and the facts that    $\ell_{i}$ is convex (from the condition of this theorem) and that $\hLcal^{}$ is a nonnegative sum of the compositions of  $\ell_{i}$ and the affine map (Lemma \ref{lemma:1}), we have that for any $z\in \RR^{d_{H-1}}$, 
\begin{align} \label{eq:step_2_theorem_4_2}
\hLcal(\theta_{(H-1)}^{t+1}) \le\hLcal^{}(z)+ \frac{\hL}{2} \|z^{}-\theta_{(H-1)}^{t}\|^2 -\frac{\hL}{2} \|z-\theta_{(H-1)}^{t+1}\|^2 . 
\end{align}
Summing up both sides of \eqref{eq:step_2_theorem_4_2} and using \eqref{eq:step_2_theorem_2_2}, \begin{align} 
 &(t-\tau)\hLcal^{}(\theta_{(H-1)}^{t})
 \\& \le\sum_{k=\tau}^{t-1} \hLcal(\theta_{(H-1)}^{t+1}) \le (t-\tau)\hLcal^{}(z)+ \frac{\hL}{2} \|z^{}-\theta_{(H-1)}^{\tau}\|^2 -\frac{\hL}{2} \|z-\theta_{(H-1)}^{t}\|^2,   
\end{align}
which implies that for any $z\in \RR^{d_{H-1}}$, 
\begin{align} 
\hLcal^{}(\theta_{(H-1)}^{t})\le\hLcal(z)+ \frac{\hL\|z^{}-\theta_{(H-1)}^{\tau}\|^2}{2(t-\tau)}.
\end{align}
By setting $z=\theta_{(H-1)}^*$ and using \eqref{eq:step_2_theorem_0_2}, we have that for any $\bepsilon\ge 0$ and for any $\theta^{*} \in \RR^d$,
$$
\Lcal(\theta^t) = \hLcal(\theta_{(H-1)}^{t})\le\hLcal(\theta_{(H-1)}^{*})+ \frac{\hL\|\theta_{(H-1)}^{*}-\theta_{(H-1)}^{\tau}\|^2}{2(t-\tau)}\le \max(\Lcal(\theta^*), \bepsilon)+ \frac{\hL B^2_{\bepsilon}}{2(t-\tau)}.
$$
This implies the statement of  Theorem \ref{thm:2} \textit{(i)}.
\end{proof}

\begin{proof}[Proof of Theorem \ref{thm:2} (ii)]
Let $t > \tau$. Using the conditions of Theorem \ref{thm:2} (ii), with probability one, 
\begin{align}
&\EE[\|\theta_{(H-1)}^{t+1}-\theta_{(H-1)}^{*}\| ^{2}\mid \theta^{t}]
\\ &=\EE[\|\theta_{(H-1)}^{t}-\theta_{(H-1)}^{*}-\alpha^t \bg^t \| ^{2}\mid \theta^{t}]
\\ & =\|\theta_{(H-1)}^{t}-\theta_{(H-1)}^{*}\|^{2}  -2\alpha^t \EE[ \bg^t \mid \theta^{t}]\T (\theta_{(H-1)}^{t}-\theta_{(H-1)}^{*})+(\alpha^t)^{2}\EE[ \|\bg^t\|^2 \mid \theta^{t}] 
\\ & \le\|\theta_{(H-1)}^{t}-\theta_{(H-1)}^{*}\|^{2}  -2\alpha^t (\hLcal^{}(\theta_{(H-1)}^{t})-\hLcal^{}(\theta_{(H-1)}^{*}))+(\alpha^t)^{2}\EE[ \|\bg^t\|^2 \mid \theta^{t}], 
\end{align}
where the last line follows from the facts that    $\ell_{i}$ is convex (from the condition of this theorem) and that $\hLcal$ is a nonnegative sum of the compositions of  $\ell_{i}$ and the affine map (Lemma \ref{lemma:1}). Taking expectation over $\theta^{t}$,
\begin{align}
&\EE[\|\theta_{(H-1)}^{t+1}-\theta_{(H-1)}^{*}\| ^{2}] 
\\ &\le \EE[\|\theta_{(H-1)}^{t}-\theta_{(H-1)}^{*}\|^{2}]-2(\alpha^t) (\EE[\hLcal^{}(\theta_{(H-1)}^{t})]-\hLcal^{}(\theta_{(H-1)}^{*}))+(\alpha^t)^{2}G^{2}.  
\end{align}
By recursively applying this inequality over $t$,
\begin{align}
&\EE[\|\theta_{(H-1)}^{t+1}-\theta_{(H-1)}^{*}\| ^{2}] 
\\ & \le\|\theta_{(H-1)}^{\tau}-\theta_{(H-1)}^{*}\| ^{2}-2 \sum_{k=\tau}^t \alpha_{k} (\EE[\hLcal^{}(\theta_{(H-1)}^{k})]-\hLcal^{}(\theta_{(H-1)}^{*}))+G^{2}  \sum_{k=\tau}^t\alpha_{k}^{2}.  
\end{align}
Since $\|\theta_{(H-1)}^{t+1}-\theta_{(H-1)}^{*}\| ^{2} \ge 0$,
 \begin{align}
 2 \sum_{k=\tau}^t \alpha_{k}  \EE[\hLcal^{}(\theta_{(H-1)}^{k})] \le \left(2 \sum_{k=\tau}^t  \alpha_{k} \right) \hLcal^{}(\theta_{(H-1)}^{*})+ B^{2}_{\bepsilon}+G^{2}  \sum_{k=\tau}^t\alpha_{k}^{2}.
\end{align} 
 Using equation \eqref{eq:step_2_theorem_0_2}, we have that for any $\bepsilon\ge 0$, 
\begin{align}
 \min_{k=\tau,\tau+1,\dots, t}\EE[\Lcal_{}^{}(\theta_{}^{k})]   \le  \max(\Lcal(\theta^*), \bepsilon)+ \frac{B^{2}_{\bepsilon}+G^{2}  \sum_{k=\tau}^t\alpha_{k}^{2}}{ 2 \sum_{k=\tau}^t \alpha_{k}}, \ \ \ \forall \theta^{*} \in \RR^d.
\end{align} 
Using Jensen's inequality and the concavity of the minimum function, it holds that for any $\bepsilon\ge 0$,  
\begin{align}
 \EE\left[\min_{k=\tau,\tau+1,\dots, t}\Lcal_{}(\theta_{}^{k})] \right]  \le\max(\Lcal(\theta^*), \bepsilon)+ \frac{B^{2}_{\bepsilon}+G^{2}  \sum_{k=\tau}^t\alpha_{k}^{2}}{ 2 \sum_{k=\tau}^t \alpha_{k}}, \ \ \ \forall \theta^{*} \in \RR^d.
\end{align} 
This implies the statement of  Theorem \ref{thm:2} \textit{(ii)}.

\end{proof}

\section{Experiments} \label{sec:app:2}
The purpose of experiments in this paper differs from those of  empirical studies. In empirical studies,   experiments are often the main source of the arguments. However, the main claims in this paper are mathematically proved. Accordingly,  the purpose of the experiments in this  paper is to deepen our theoretical understanding (with experiments in  Section \ref{sec:3}) and provide  evidence to additionally support our theory and what is already predicted by a well-known hypothesis (with experiments in  Section \ref{sec:9}). The experiments were implemented in PyTorch \citep{paszke2019pytorch}.
We additionally provide experimental details in the following.

\subsection{Experimental Setup for Section \ref{sec:3}}

\paragraph{Model.} 
We used a fully-connected   deep neural network and a convolutional deep neural network.
For the fully-connected    network,
we used four layers and 300 neurons per hidden layer. The input dimensions were two for the two-moons dataset, and one for the
sine wave dataset. The output dimension of the fully-connected network was one for both datasets. Each  entry of the weight matrices in the fully-connected layers was  initialized independently by the normal distribution with the standard deviation of $1/\sqrt{m}$  with $m=300$ for all layers. We set the  activation functions of all layers  to be softplus $\sigma(z)=\ln(1+\exp(\varsigma z))/\varsigma$ with $\varsigma=100$, which  approximately behaves as the ReLU activation   as shown in Appendix \ref{sec:app:3}. For the convolutional   deep neural network, we used  the following standard variant of LeNet \citep{lecun1998gradient} with five layers: (1) input layer, (2) convolutional hidden layer  with $8m_c  $ filters of size $5$-by-$5$ with softplus activation $\sigma(z)=\ln(1+\exp(\varsigma z))/\varsigma$ with $\varsigma=100$, (3) convolutional hidden layer  with $8m_c  $ filters of size $5$-by-$5$ with softplus activation $\sigma(z)=\ln(1+\exp(\varsigma z))/\varsigma$ with $\varsigma=100$,
(4)
fully-connected hidden layer with $128m_c$ output units with softplus activation $\sigma(z)=\ln(1+\exp(\varsigma z))/\varsigma$ with $\varsigma=100$, (5) fully-connected output layer.
All the model parameters  were initialized by the default initialization of PyTorch version 1.4.0 \cite{NEURIPS2019_9015}, which is based on the implementation of \cite{he2015delving}.

\paragraph{Data.}  For the fully-connected   deep neural network, we used the standard two-moons dataset \citep{scikitlearn} and a sine wave dataset. Whereas the two-moons dataset  is known to be easily learnable, a sine wave dataset with a high frequency is known to be challenging for fitting.
For the two-moons dataset, we generated 100 training samples using the scikit-learn command,  sklearn.datasets.make\_moons \citep{scikitlearn}. To generate the  sine wave dataset, we  randomly generated the input $x_i $  from the uniform distribution on  the interval $[-1,1]$ and set $y_i=\one\{\sin(20x_i)<0\}\in \RR$ for all $i \in [n]$ with $n=100$. For the convolutional   deep neural network, we used the Semeion dataset \citep{srl1994semeion} and a random dataset. The Semeion dataset has 1593 data points and we randomly selected  1000 data points  as training data points. The random dataset was created  by randomly generating each pixel of the input image  $x_i \in \RR^{16\times 16 \times 1}$  from the standard normal distribution  and by  sampling $y_i$ uniformly from $\{0,1\}$ for all $i \in [n]$ with $n=1000$.  

\paragraph{Training.} 
For each dataset, we used (mini-batch) stochastic gradient descent (SGD) with mini-batch size of 100. This effectively results in (full-batch) gradient decent for the two-moons dataset and sine wave dataset (because they only have 100 data points), whereas it behaves as SGD  for the Semeion dataset and the random dataset (because they have 1000 data points). We used the cross-entropy loss as the training loss. We fixed the momentum coefficient to be 0.9 for the convolutional network and 0.98 for the fully-connected network. Under this setting, deep neural networks have been empirically observed to have implicit regularization effects (e.g., see  \citep{poggio2017theory}).
We set the learning rate to be 0.1 for all the experiments. We fixed the end epoch to be 1000 for the fully-connected networks and 400 for the convolutional network.

\subsection{Experimental Setup and Additional Results for Section \ref{sec:9}} 
We provide  additional details and  results for the experiments with the sine wave dataset in Appendix \ref{sec:new:1} and image datasets  in Appendix \ref{sec:new:2}. 

\subsubsection{Sine Wave Dataset}  \label{sec:new:1}
 For the sine wave dataset, we also verified that the safe-exploration condition (Assumption \ref{assump:1}) is satisfied in the setting of the experiment in  Section \ref{sec:9}.
We set  $\varepsilon=0.01, \tau=2000$, and $\tGcal=\Gcal$ for the EE wrapper $\Acal$ for the sine wave dataset. We fixed the last epoch to be 10000. All other settings were the same as those for Subsection \ref{sec:3}. That is, to generate the  sine wave dataset, we  randomly generated the input $x_i $  from the uniform distribution on  the interval $[-1,1]$ and set $y_i=\one\{\sin(20x_i)<0\}\in \RR$ for all $i \in [n]$ with $n=100$. We used (mini-batch) stochastic gradient descent (SGD) with a mini-batch size of 100. We fixed the momentum coefficient to be 0.98 and the learning rate to be 0.1.
We used a fully-connected   deep neural network with  four layers and 300 neurons per hidden layer. The input dimension and the output dimension were one.  We set the  activation functions of all layers  to be softplus $\sigma(z)=\ln(1+\exp(\varsigma z))/\varsigma$ with $\varsigma=100$.  
We set $\tsigma(q)= \frac{1}{1+e^{-\varsigma'  q}}$ with $\varsigma' =1000$ for the EE wrapper $\Acal$. Each  entry of the weight matrices  was  initialized independently by the normal distribution with the standard deviation of $1/\sqrt{m}$  with $m=300$ for all layers. For the purpose of the random trials, we repeated this random initialization for three independent random trials with different random seeds in Figure \ref{fig:3}. In Figure \ref{fig:4}, we reported the result of the first trial.  

\subsubsection{Image Datasets} \label{sec:new:2}

In the following, we provide the details of experiments for image datasets. We used  the exactly same setting across the experiments for test performances, training behaviors and computational time. For the experiments on the effect of learning rates and optimizers, we changed the learning rates and optimizers as explained in Section \ref{sec:9}.

\paragraph{Model.} 
We used the standard (convolutional) pre-activation ResNet with $18$ layers \citep{he2016identity}. We set the  activation  to be the softplus function $q\mapsto\ln(1+\exp(\varsigma q))/\varsigma$ with $\varsigma=100$ for all layers of the base ResNet. This  approximates the ReLU activation  well as shown in Appendix \ref{sec:app:3}. We set $\tsigma(q)= \frac{1}{1+e^{-q}}$ for the EE wrapper $\Acal$. All the model parameters  were initialized by the default initialization of PyTorch version 1.4.0 \cite{NEURIPS2019_9015}, which is based on the implementation of \cite{he2015delving}. As $m_{H-1}=513$ (512 + the constant neuron for the bias term) for  the standard ResNet, we set $m_H = \ceil{2( n/m_{H-1})}$ throughout all the experiments with the ResNet.

\paragraph{Data.}

We adopted the standard benchmark datasets  --- MNIST \citep{lecun1998gradient}, CIFAR-10 \allowbreak \citep{krizhevsky2009learning},  CIFAR-100  \citep{krizhevsky2009learning},  Semeion \citep{srl1994semeion}, Kuzushiji-MNIST (KMNIST) \citep{clanuwat2019deep} and SVHN \citep{netzer2011reading}.  We used  all the training and testing data points exactly as provided
by those datasets, except for the Semeion dataset. For the Semeion dataset, the default split of training and testing data points is not provided and  thus we randomly selected  1000 data   points as training data points from  1593 data points.
The remaining 593 points were used as testing data points. In the case of  data-augmentation, we used the standard data-augmentation of images for each datasets by using \texttt{torchvision.transforms}:  random crop (with \texttt{RandomCrop(28, padding=2})) and random affine transformation (with \texttt{RandomAffine(15, scale =(0.85, 1.15})) for MNIST; random crop (with \texttt{RandomCrop(32, padding =4})) and horizontal flip (with \texttt{RandomHorizontalFlip}) for CIFAR-10, CIFAR-100 and SVHN; random crop (with \texttt{RandomCrop(16, padding=1})) and random affine transformation (with \texttt{RandomAffine(4, scale=(1.05, 1.05})) for Semeion;
random crop (with \texttt{RandomCrop(28, padding=2})) for Kuzushiji-MNIST (KMNIST).

\paragraph{Training.}

We employed the cross-entropy loss as the training loss and a standard algorithm, SGD, with its standard hyper-parameter setting for the  algorithm $\Gcal$  with $\tGcal_{}=\Gcal$: i.e.,  we let the mini-batch size be 64, the weight decay rate be $10^{-5}$, the momentum coefficient be $0.9$, the learning rate  be   $\alpha^{t}=0.1$, the last epoch $\hat T$ be 200 (with data augmentation) and 100  (without data augmentation). The hyper-parameters $\varepsilon$ and $\tau=\tau_0 \hat T$ were selected from $\varepsilon\in \{10^{-3},10^{-5}\}$  and $\tau_0 \in \{0.4,0.6,0.8\}$ by only using  training data. That is,  we randomly divided  each training data (100\%)\ into a smaller training data (80\%)\ and a validation data (20\%) for a grid search over the hyper-parameters. \ In other words,  this grid search only used  a smaller training data and a small validation data without using any testing data.
Tables \ref{tbl:7}--\ref{tbl:8} show the results of this grid search for each dataset with data augmentation, whereas Tables \ref{tbl:9}--\ref{tbl:10} summarise the results  for each dataset without  data augmentation. 
Based on these results, we fixed the hyper-parameters to be the values with the  underline in each table: e.g.,  $\varepsilon=10^{-3}$ and $\tau_0=0.8$ for MNIST with data augmentation based on Table \ref{tbl:7}. \begin{table}[!ht]
\begin{minipage}[b]{0.5\linewidth}\centering
\caption{Validation error (\%): MNIST} \label{tbl:7}
\centering 
\renewcommand{\arraystretch}{0.8} \fontsize{11.pt}{11.pt}\selectfont
\begin{tabular}{|l||*{4}{c|}}
\hline
\backslashbox{$\varepsilon$ }{$\tau_0$} 
&\makebox[1em]{0.4}
&\makebox[1em]{0.6}
&\makebox[1em]{0.8}
\\\hline\hline
$10^{-3}$ 
& 
0.65
& 
0.55& \uline{0.50}
\\\hline
$10^{-5}$ 
&
0.55
& 0.65 & 0.59\\\hline
\end{tabular}
\end{minipage}
\begin{minipage}[b]{0.5\linewidth}\centering
\centering 
\renewcommand{\arraystretch}{0.8} \fontsize{11.pt}{11.pt}\selectfont
\caption{Validation error (\%): CIFAR-10 }
\begin{tabular}{|l||*{4}{c|}}
\hline
\backslashbox{$\varepsilon$ }{$\tau_0$} 
&\makebox[1em]{0.4}
&\makebox[1em]{0.6}
&\makebox[1em]{0.8}
\\\hline\hline
$10^{-3}$ 
& 
8.51
& 
8.39& 8.52\\\hline
$10^{-5}$ 
&
8.68
& \uline{8.03} & 8.22 \\\hline
\end{tabular}
\end{minipage}
\end{table}

\begin{table}[!ht]
\begin{minipage}[b]{0.5\linewidth}\centering
\caption{Validation error (\%): CIFAR-100 }
\centering \renewcommand{\arraystretch}{0.8} \fontsize{11.pt}{11.pt}\selectfont
\begin{tabular}{|l||*{4}{c|}}
\hline
\backslashbox{$\varepsilon$ }{$\tau_0$} 
&\makebox[1em]{0.4}
&\makebox[1em]{0.6}
&\makebox[1em]{0.8}
\\\hline\hline
$10^{-3}$ 
& 
31.18
& 
\uline{30.92}& 31.50
\\\hline
$10^{-5}$ 
&
31.18
& 31.43 & 31.21 \\\hline
\end{tabular}
\end{minipage}
\begin{minipage}[b]{0.5\linewidth}\centering
\centering \renewcommand{\arraystretch}{0.8} \fontsize{11.pt}{11.pt}\selectfont
\caption{Validation error (\%): Semeion } 
\begin{tabular}{|l||*{4}{c|}}
\hline
\backslashbox{$\varepsilon$ }{$\tau_0$} 
&\makebox[1em]{0.4}
&\makebox[1em]{0.6}
&\makebox[1em]{0.8}
\\\hline\hline
$10^{-3}$ 
& 
6.00
& 
4.00& 4.00
\\\hline
$10^{-5}$ 
&
\uline{2.50}
& 4.00 & 4.50 \\\hline
\end{tabular}
\end{minipage}
\end{table}

\begin{table}[!ht]
\begin{minipage}[b]{0.5\linewidth}\centering
\caption{Validation error (\%): KMNIST}  
\centering \renewcommand{\arraystretch}{0.8} \fontsize{11.pt}{11.pt}\selectfont
\begin{tabular}{|l||*{4}{c|}}
\hline
\backslashbox{$\varepsilon$ }{$\tau_0$} 
&\makebox[1em]{0.4}
&\makebox[1em]{0.6}
&\makebox[1em]{0.8}
\\\hline\hline
$10^{-3}$ 
& 
0.51
& 
\uline{0.50}& 0.59
\\\hline
$10^{-5}$ 
&
0.60& 0.53 & 0.55 \\\hline
\end{tabular}
\end{minipage}
\begin{minipage}[b]{0.5\linewidth}\centering
\centering \renewcommand{\arraystretch}{0.8} \fontsize{11.pt}{11.pt}\selectfont
\caption{Validation error (\%): SVHN} \label{tbl:8}
\begin{tabular}{|l||*{4}{c|}}
\hline
\backslashbox{$\varepsilon$ }{$\tau_0$} 
&\makebox[1em]{0.4}
&\makebox[1em]{0.6}
&\makebox[1em]{0.8}
\\\hline\hline
$10^{-3}$ 
& 
5.08
& 
\uline{4.91}& 5.26
\\\hline
$10^{-5}$ 
&
5.20
&  5.12 & 5.20 \\\hline
\end{tabular}
\end{minipage}
\end{table}

\begin{table}[!ht]
\begin{minipage}[b]{1.0\linewidth}\centering
\caption{Validation error (\%) for MNIST without data augmentation}  \label{tbl:9}
\centering \renewcommand{\arraystretch}{0.8} \fontsize{11.pt}{11.pt}\selectfont
\begin{tabular}{|l||*{4}{c|}}
\hline
\backslashbox{$\varepsilon$ }{$\tau_0$} 
&\makebox[1em]{0.4}
&\makebox[1em]{0.6}
&\makebox[1em]{0.8}
\\\hline\hline
$10^{-3}$ 
& 
0.55
& 
0.51& \uline{0.50}
\\\hline
$10^{-5}$ 
&
0.54
& 0.62& 0.62\\\hline
\end{tabular}
\end{minipage}
\end{table}
\begin{table}[!ht]
\begin{minipage}[b]{1.0\linewidth}\centering
\centering \renewcommand{\arraystretch}{0.8} \fontsize{11.pt}{11.pt}\selectfont
\caption{Validation error (\%) for CIFAR-10   without data augmentation}
\begin{tabular}{|l||*{4}{c|}}
\hline
\backslashbox{$\varepsilon$ }{$\tau_0$} 
&\makebox[1em]{0.4}
&\makebox[1em]{0.6}
&\makebox[1em]{0.8}
\\\hline\hline
$10^{-3}$ 
& 
16.00
& 
16.02& 16.47\\\hline
$10^{-5}$ 
&
\uline{15.65}
& 16.18 & 15.86 \\\hline
\end{tabular}
\end{minipage}
\end{table}

\begin{table}[H]
\begin{minipage}[b]{1.0\linewidth}\centering
\caption{Validation error (\%) for CIFAR-100  without data augmentation}  \label{tbl:10}
\centering \renewcommand{\arraystretch}{0.8} \fontsize{11.pt}{11.pt}\selectfont
\begin{tabular}{|l||*{4}{c|}}
\hline
\backslashbox{$\varepsilon$ }{$\tau_0$} 
&\makebox[1em]{0.4}
&\makebox[1em]{0.6}
&\makebox[1em]{0.8}
\\\hline\hline
$10^{-3}$ 
& 
\uline{48.65}
& 
51.98 & 51.37
\\\hline
$10^{-5}$ 
&
50.83
& 50.48& 51.74 \\\hline
\end{tabular}
\end{minipage}
\begin{minipage}[b]{0.5\linewidth}\centering
\centering 

\end{minipage}
\end{table}

\section{Additional Discussion} \label{sec:app:3}
\paragraph{On softplus activation.} Throughout the experiments in this paper, we use the softplus function $q\mapsto\ln(1+\exp(\varsigma q))/\varsigma$ with $\varsigma=100$, as a real analytic (and differentiable) function that approximates the ReLU activation. This approximation is of high accuracy  as shown in Figure \ref{fig:softplus_vs_relu}.

\begin{figure}[!ht]
\centering
\begin{subfigure}[b]{0.35\columnwidth}
  \includegraphics[width=1.0\columnwidth,height=0.65\columnwidth]{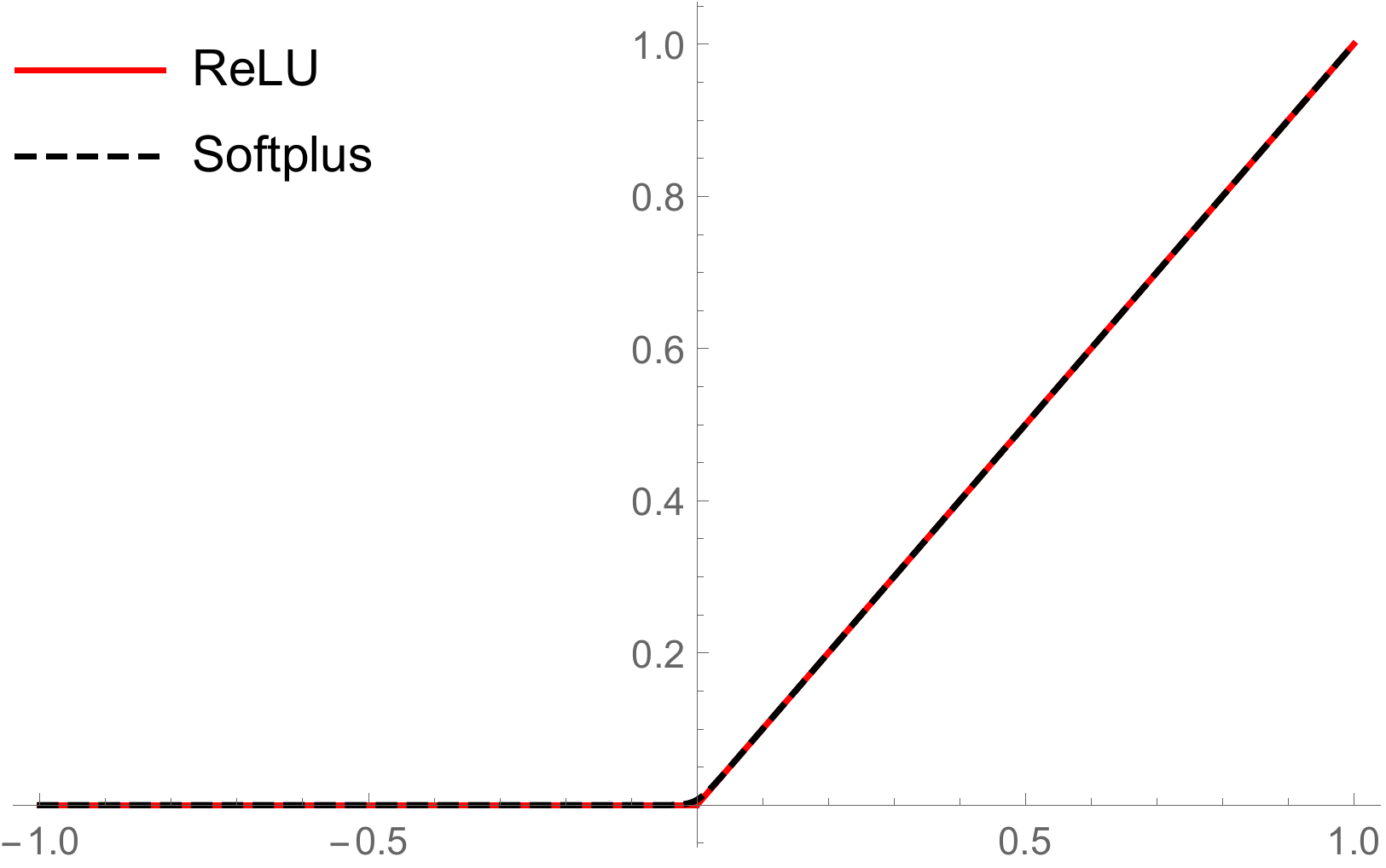}
  \caption{$x$-axis scale $[-1,1]$}
\end{subfigure}
\hspace{20pt}
\begin{subfigure}[b]{0.35\columnwidth}
  \includegraphics[width=1.0\columnwidth,height=0.65\columnwidth]{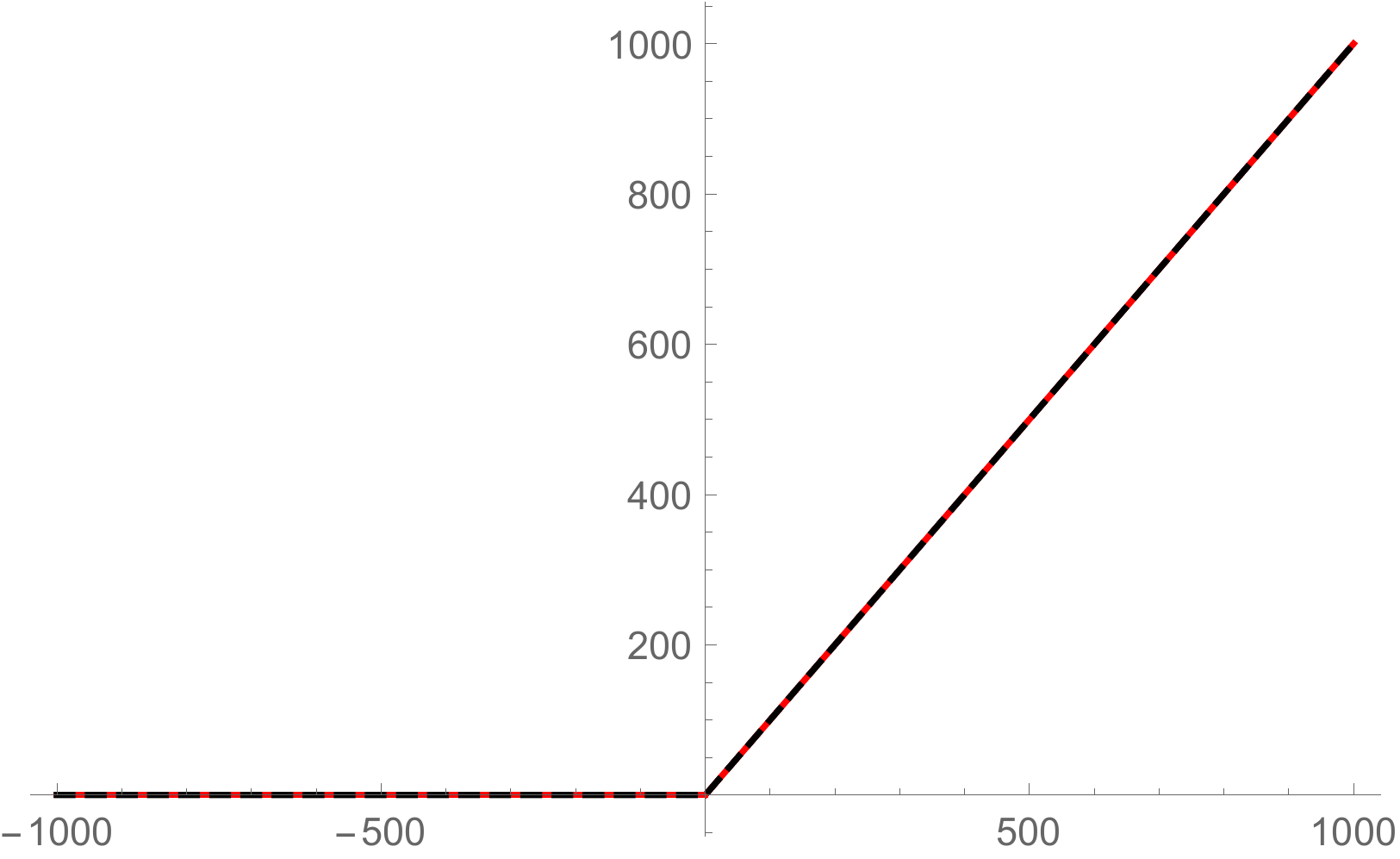}
  \caption{$x$-axis scale $[-1000,1000]$}
\end{subfigure}
\captionof{figure}{ReLU   versus Softplus   $q\mapsto\ln(1+\exp(\varsigma q))/\varsigma$ with $\varsigma=100$. The plot lines  of ReLU  (orange line) and Softplus (black dashed line) coincide in the figure.} 
\label{fig:softplus_vs_relu} 
\end{figure}

\paragraph{On Theorem \ref{thm:p:1}.}
Theorem \ref{thm:p:1} predicts that the upper bound on the global optimality gap ($\min_{t \in [T]}\Lcal(\theta^{t}) -\inf_{\theta \in \RR^d} \Lcal(\theta)$) decreases towards zero whenever $\vect(Y_{\ell}) \allowbreak \in \Col(\frac{\partial \vect(f_{X}(\theta^{t}))}{\partial \theta^{t}})$ at each iteration. This follows from equations \eqref{eq:19} and \eqref{eq:p:5}. Moreover, Theorem \ref{thm:p:1} implies that $\epsilon$-near global minimum  values are achieved for any $\epsilon>0$, if $\vect(Y_{\ell}) \in \Col(\frac{\partial \vect(f_{X}(\theta^{t}))}{\partial \theta^{t}})$ at sufficiently  many iterations $t$ for both regression and classification losses.

\paragraph{On cross-entropy losses with Theorem \ref{thm:p:1}.} Note that even with $\eta=1$, the value of  $\Lcal^*(\eta Y_\ell)$ in Theorem \ref{thm:p:1} is that  of the correct classification with zero classification error. Moreover, given any desired accuracy $\epsilon>0$, there exists  $\eta \in \RR$ to obtain $\epsilon$-near global optimality for the binary and multi-class cross-entropy losses with Theorem \ref{thm:p:1}. However, unlike the squared loss case, we cannot set $\epsilon=0$ as it requires $\eta= \infty$ for which the right-hand-side of equation \eqref{eq:p:5} yields infinity. This is consistent with the properties of the logistic loss and multi-class cross-entropy loss. This  shows that our theory is consistent with true behaviors of the logistic loss and multi-class cross-entropy loss as desired

\paragraph{On the EE wrapper $\Acal$.} The EE wrapper $\Acal$  is designed to provide prior guarantees while \textit{not hurting} practical performances,  instead of improving them.
We can use any  initialization for the parameter vector $\theta^0$  at line 4 of Algorithm \ref{algorithm:EEW}: i.e., a non-standard poor initialization can  degrade the value of the factor $\hL$ in our convergence rate. At   line  2 of Algorithm \ref{algorithm:EEW},  one fully-connected last layer is added if the original network had one fully-connected last layer, and  two fully-connected last layers are added if the original had no fully-connected last layer. The model modifications can also be  automated.


\begin{thebibliography}{}

\bibitem[\protect\astroncite{Bartlett et~al.}{2019}]{bartlett2019gradient}
Bartlett, P.~L., Helmbold, D.~P., and Long, P.~M. (2019).
\newblock Gradient descent with identity initialization efficiently learns
  positive-definite linear transformations by deep residual networks.
\newblock {\em Neural computation}, 31(3):477--502.

\bibitem[\protect\astroncite{Bengio et~al.}{2013}]{bengio2013representation}
Bengio, Y., Courville, A., and Vincent, P. (2013).
\newblock Representation learning: A review and new perspectives.
\newblock {\em IEEE transactions on pattern analysis and machine intelligence},
  35(8):1798--1828.

\bibitem[\protect\astroncite{Bengio et~al.}{2007}]{bengio2007greedy}
Bengio, Y., Lamblin, P., Popovici, D., Larochelle, H., et~al. (2007).
\newblock Greedy layer-wise training of deep networks.
\newblock {\em Advances in neural information processing systems}, 19:153.

\bibitem[\protect\astroncite{Bordes et~al.}{2012}]{bordes2012joint}
Bordes, A., Glorot, X., Weston, J., and Bengio, Y. (2012).
\newblock Joint learning of words and meaning representations for open-text
  semantic parsing.
\newblock In {\em Artificial Intelligence and Statistics}, pages 127--135.
  PMLR.

\bibitem[\protect\astroncite{Ciregan et~al.}{2012}]{ciregan2012multi}
Ciregan, D., Meier, U., and Schmidhuber, J. (2012).
\newblock Multi-column deep neural networks for image classification.
\newblock In {\em 2012 IEEE conference on computer vision and pattern
  recognition}, pages 3642--3649. IEEE.

\bibitem[\protect\astroncite{Clanuwat et~al.}{2019}]{clanuwat2019deep}
Clanuwat, T., Bober-Irizar, M., Kitamoto, A., Lamb, A., Yamamoto, K., and Ha,
  D. (2019).
\newblock Deep learning for classical japanese literature.
\newblock In {\em NeurIPS Creativity Workshop 2019}.

\bibitem[\protect\astroncite{Dahl et~al.}{2010}]{dahl2010phone}
Dahl, G., Ranzato, M., Mohamed, A.-r., and Hinton, G.~E. (2010).
\newblock Phone recognition with the mean-covariance restricted boltzmann
  machine.
\newblock {\em Advances in neural information processing systems}, 23:469--477.

\bibitem[\protect\astroncite{Dahl et~al.}{2011}]{dahl2011context}
Dahl, G.~E., Yu, D., Deng, L., and Acero, A. (2011).
\newblock Context-dependent pre-trained deep neural networks for
  large-vocabulary speech recognition.
\newblock {\em IEEE Transactions on audio, speech, and language processing},
  20(1):30--42.

\bibitem[\protect\astroncite{Deng et~al.}{2010}]{deng2010binary}
Deng, L., Seltzer, M.~L., Yu, D., Acero, A., Mohamed, A.-r., and Hinton, G.
  (2010).
\newblock Binary coding of speech spectrograms using a deep auto-encoder.
\newblock In {\em Eleventh Annual Conference of the International Speech
  Communication Association}.

\bibitem[\protect\astroncite{Dong et~al.}{2014}]{dong2014learning}
Dong, C., Loy, C.~C., He, K., and Tang, X. (2014).
\newblock Learning a deep convolutional network for image super-resolution.
\newblock In {\em European conference on computer vision}, pages 184--199.
  Springer.

\bibitem[\protect\astroncite{Gatys et~al.}{2016}]{gatys2016image}
Gatys, L.~A., Ecker, A.~S., and Bethge, M. (2016).
\newblock Image style transfer using convolutional neural networks.
\newblock In {\em Proceedings of the IEEE conference on computer vision and
  pattern recognition}, pages 2414--2423.

\bibitem[\protect\astroncite{Glorot et~al.}{2011}]{glorot2011domain}
Glorot, X., Bordes, A., and Bengio, Y. (2011).
\newblock Domain adaptation for large-scale sentiment classification: A deep
  learning approach.
\newblock In {\em ICML}.

\bibitem[\protect\astroncite{Golub and Van~Loan}{1996}]{golub1996matrix}
Golub, G.~H. and Van~Loan, C.~F. (1996).
\newblock {\em Matrix computations}.
\newblock Johns Hopkins University Press.

\bibitem[\protect\astroncite{He et~al.}{2015}]{he2015delving}
He, K., Zhang, X., Ren, S., and Sun, J. (2015).
\newblock Delving deep into rectifiers: Surpassing human-level performance on
  imagenet classification.
\newblock In {\em Proceedings of the IEEE international conference on computer
  vision}, pages 1026--1034.

\bibitem[\protect\astroncite{He et~al.}{2016}]{he2016identity}
He, K., Zhang, X., Ren, S., and Sun, J. (2016).
\newblock Identity mappings in deep residual networks.
\newblock In {\em European Conference on Computer Vision}, pages 630--645.
  Springer.

\bibitem[\protect\astroncite{Hinton et~al.}{2012}]{hinton2012deep}
Hinton, G., Deng, L., Yu, D., Dahl, G.~E., Mohamed, A.-r., Jaitly, N., Senior,
  A., Vanhoucke, V., Nguyen, P., Sainath, T.~N., et~al. (2012).
\newblock Deep neural networks for acoustic modeling in speech recognition: The
  shared views of four research groups.
\newblock {\em IEEE Signal processing magazine}, 29(6):82--97.

\bibitem[\protect\astroncite{Hinton et~al.}{2006}]{hinton2006fast}
Hinton, G.~E., Osindero, S., and Teh, Y.-W. (2006).
\newblock A fast learning algorithm for deep belief nets.
\newblock {\em Neural computation}, 18(7):1527--1554.

\bibitem[\protect\astroncite{Kawaguchi}{2016}]{kawaguchi2016deep}
Kawaguchi, K. (2016).
\newblock Deep learning without poor local minima.
\newblock In {\em Advances in Neural Information Processing Systems}, pages
  586--594.

\bibitem[\protect\astroncite{Kawaguchi}{2021}]{kawaguchi2021on}
Kawaguchi, K. (2021).
\newblock On the theory of implicit deep learning: Global convergence with
  implicit layers.
\newblock In {\em International Conference on Learning Representations (ICLR)}.

\bibitem[\protect\astroncite{Kawaguchi et~al.}{2015}]{kawaguchiNIPS2015}
Kawaguchi, K., Kaelbling, L.~P., and Lozano-P\'{e}rez, T. (2015).
\newblock Bayesian optimization with exponential convergence.
\newblock In {\em Advances in Neural Information Processing}, pages 2809--2817.

\bibitem[\protect\astroncite{Kawaguchi et~al.}{2016}]{kawaguchi2016global}
Kawaguchi, K., Maruyama, Y., and Zheng, X. (2016).
\newblock Global continuous optimization with error bound and fast convergence.
\newblock {\em Journal of Artificial Intelligence Research}, 56:153--195.

\bibitem[\protect\astroncite{Kawaguchi and Sun}{2021}]{kawaguchi2021recipe}
Kawaguchi, K. and Sun, Q. (2021).
\newblock A recipe for global convergence guarantee in deep neural networks.
\newblock In {\em Proceedings of the AAAI Conference on Artificial
  Intelligence}, volume~35, pages 8074--8082.

\bibitem[\protect\astroncite{Krizhevsky and
  Hinton}{2009}]{krizhevsky2009learning}
Krizhevsky, A. and Hinton, G. (2009).
\newblock Learning multiple layers of features from tiny images.
\newblock Technical report, Citeseer.

\bibitem[\protect\astroncite{Krizhevsky et~al.}{2012}]{krizhevsky2012imagenet}
Krizhevsky, A., Sutskever, I., and Hinton, G.~E. (2012).
\newblock Imagenet classification with deep convolutional neural networks.
\newblock {\em Advances in neural information processing systems},
  25:1097--1105.

\bibitem[\protect\astroncite{Laurent and Brecht}{2018}]{laurent2018deep}
Laurent, T. and Brecht, J. (2018).
\newblock Deep linear networks with arbitrary loss: All local minima are
  global.
\newblock In {\em International conference on machine learning}, pages
  2902--2907. PMLR.

\bibitem[\protect\astroncite{Le et~al.}{2012}]{le2012structured}
Le, H.-S., Oparin, I., Allauzen, A., Gauvain, J.-L., and Yvon, F. (2012).
\newblock Structured output layer neural network language models for speech
  recognition.
\newblock {\em IEEE Transactions on Audio, Speech, and Language Processing},
  21(1):197--206.

\bibitem[\protect\astroncite{LeCun et~al.}{2015}]{lecun2015deep}
LeCun, Y., Bengio, Y., and Hinton, G. (2015).
\newblock Deep learning.
\newblock {\em Nature}, 521(7553):436--444.

\bibitem[\protect\astroncite{LeCun et~al.}{1998}]{lecun1998gradient}
LeCun, Y., Bottou, L., Bengio, Y., and Haffner, P. (1998).
\newblock Gradient-based learning applied to document recognition.
\newblock {\em Proceedings of the IEEE}, 86(11):2278--2324.

\bibitem[\protect\astroncite{Luan et~al.}{2017}]{luan2017deep}
Luan, F., Paris, S., Shechtman, E., and Bala, K. (2017).
\newblock Deep photo style transfer.
\newblock In {\em Proceedings of the IEEE conference on computer vision and
  pattern recognition}, pages 4990--4998.

\bibitem[\protect\astroncite{Mityagin}{2015}]{mityagin2015zero}
Mityagin, B. (2015).
\newblock The zero set of a real analytic function.
\newblock {\em arXiv preprint arXiv:1512.07276}.

\bibitem[\protect\astroncite{Mohamed et~al.}{2011}]{mohamed2011acoustic}
Mohamed, A.-r., Dahl, G.~E., and Hinton, G. (2011).
\newblock Acoustic modeling using deep belief networks.
\newblock {\em IEEE transactions on audio, speech, and language processing},
  20(1):14--22.

\bibitem[\protect\astroncite{Netzer et~al.}{2011}]{netzer2011reading}
Netzer, Y., Wang, T., Coates, A., Bissacco, A., Wu, B., and Ng, A.~Y. (2011).
\newblock Reading digits in natural images with unsupervised feature learning.
\newblock In {\em NIPS workshop on deep learning and unsupervised feature
  learning}.

\bibitem[\protect\astroncite{Paszke et~al.}{2019a}]{paszke2019pytorch}
Paszke, A., Gross, S., Massa, F., Lerer, A., Bradbury, J., Chanan, G., Killeen,
  T., Lin, Z., Gimelshein, N., Antiga, L., et~al. (2019a).
\newblock Pytorch: An imperative style, high-performance deep learning library.
\newblock In {\em Advances in neural information processing systems}, pages
  8026--8037.

\bibitem[\protect\astroncite{Paszke et~al.}{2019b}]{NEURIPS2019_9015}
Paszke, A., Gross, S., Massa, F., Lerer, A., Bradbury, J., Chanan, G., Killeen,
  T., Lin, Z., Gimelshein, N., Antiga, L., Desmaison, A., Kopf, A., Yang, E.,
  DeVito, Z., Raison, M., Tejani, A., Chilamkurthy, S., Steiner, B., Fang, L.,
  Bai, J., and Chintala, S. (2019b).
\newblock Pytorch: An imperative style, high-performance deep learning library.
\newblock In {\em Advances in Neural Information Processing Systems 32}, pages
  8024--8035.

\bibitem[\protect\astroncite{Pedregosa et~al.}{2011}]{scikitlearn}
Pedregosa, F., Varoquaux, G., Gramfort, A., Michel, V., Thirion, B., Grisel,
  O., Blondel, M., Prettenhofer, P., Weiss, R., Dubourg, V., Vanderplas, J.,
  Passos, A., Cournapeau, D., Brucher, M., Perrot, M., and Duchesnay, E.
  (2011).
\newblock Scikit-learn: Machine learning in {P}ython.
\newblock {\em Journal of Machine Learning Research}, 12:2825--2830.

\bibitem[\protect\astroncite{Poggio et~al.}{2017}]{poggio2017theory}
Poggio, T., Kawaguchi, K., Liao, Q., Miranda, B., Rosasco, L., Boix, X.,
  Hidary, J., and Mhaskar, H. (2017).
\newblock Theory of deep learning iii: explaining the non-overfitting puzzle.
\newblock {\em arXiv preprint arXiv:1801.00173}.

\bibitem[\protect\astroncite{Press et~al.}{2007}]{press2007numerical}
Press, W.~H., Teukolsky, S.~A., Vetterling, W.~T., and Flannery, B.~P. (2007).
\newblock {\em Numerical recipes 3rd edition: The art of scientific computing}.
\newblock Cambridge university press.

\bibitem[\protect\astroncite{Rifai et~al.}{2011}]{rifai2011manifold}
Rifai, S., Dauphin, Y.~N., Vincent, P., Bengio, Y., and Muller, X. (2011).
\newblock The manifold tangent classifier.
\newblock {\em Advances in neural information processing systems},
  24:2294--2302.

\bibitem[\protect\astroncite{Saxe et~al.}{2014}]{saxe2013exact}
Saxe, A.~M., McClelland, J.~L., and Ganguli, S. (2014).
\newblock Exact solutions to the nonlinear dynamics of learning in deep linear
  neural networks.
\newblock In {\em International Conference on Learning Representations}.

\bibitem[\protect\astroncite{Schwenk et~al.}{2012}]{schwenk2012large}
Schwenk, H., Rousseau, A., and Attik, M. (2012).
\newblock Large, pruned or continuous space language models on a gpu for
  statistical machine translation.
\newblock In {\em Proceedings of the NAACL-HLT 2012 Workshop: Will We Ever
  Really Replace the N-gram Model? On the Future of Language Modeling for HLT},
  pages 11--19.

\bibitem[\protect\astroncite{Seide et~al.}{2011}]{seide2011conversational}
Seide, F., Li, G., and Yu, D. (2011).
\newblock Conversational speech transcription using context-dependent deep
  neural networks.
\newblock In {\em Twelfth annual conference of the international speech
  communication association}.

\bibitem[\protect\astroncite{Socher et~al.}{2011a}]{socher2011dynamic}
Socher, R., Huang, E.~H., Pennington, J., Ng, A.~Y., and Manning, C.~D.
  (2011a).
\newblock Dynamic pooling and unfolding recursive autoencoders for paraphrase
  detection.
\newblock In {\em NIPS}, volume~24, pages 801--809.

\bibitem[\protect\astroncite{Socher et~al.}{2011b}]{socher2011semi}
Socher, R., Pennington, J., Huang, E.~H., Ng, A.~Y., and Manning, C.~D.
  (2011b).
\newblock Semi-supervised recursive autoencoders for predicting sentiment
  distributions.
\newblock In {\em Proceedings of the 2011 conference on empirical methods in
  natural language processing}, pages 151--161.

\bibitem[\protect\astroncite{Srl and Brescia}{1994}]{srl1994semeion}
Srl, B.~T. and Brescia, I. (1994).
\newblock Semeion handwritten digit data set.
\newblock {\em Semeion Research Center of Sciences of Communication, Rome,
  Italy}.

\bibitem[\protect\astroncite{Turner et~al.}{1999}]{turner1999conceptual}
Turner, C.~R., Fuggetta, A., Lavazza, L., and Wolf, A.~L. (1999).
\newblock A conceptual basis for feature engineering.
\newblock {\em Journal of Systems and Software}, 49(1):3--15.

\bibitem[\protect\astroncite{Xu et~al.}{2021}]{Xu2021gnnopt}
Xu, K., Zhang, M., Jegelka, S., and Kawaguchi, K. (2021).
\newblock Optimization of graph neural networks: Implicit acceleration by skip
  connections and more depth.
\newblock In {\em International Conference on Machine Learning (ICML)}.

\bibitem[\protect\astroncite{Zheng and Casari}{2018}]{zheng2018feature}
Zheng, A. and Casari, A. (2018).
\newblock {\em Feature engineering for machine learning: principles and
  techniques for data scientists}.
\newblock " O'Reilly Media, Inc.".

\bibitem[\protect\astroncite{Zou et~al.}{2020}]{zou2020global}
Zou, D., Long, P.~M., and Gu, Q. (2020).
\newblock On the global convergence of training deep linear resnets.
\newblock In {\em International Conference on Learning Representations}.

\end{thebibliography}
\end{document}